\pdfoutput=1
\documentclass{article} 
\usepackage{iclr2027_conference,times}

\usepackage{amsmath,amsfonts,bm}

\def\eqref#1{equation~\ref{#1}}

\def\1{\bm{1}}

\DeclareMathAlphabet{\mathsfit}{\encodingdefault}{\sfdefault}{m}{sl}
\SetMathAlphabet{\mathsfit}{bold}{\encodingdefault}{\sfdefault}{bx}{n}

\usepackage{hyperref}
\usepackage{url}
\usepackage{minitoc}
\usepackage[toc,page,header]{appendix}
\usepackage{booktabs}
\usepackage{wrapfig}
\usepackage{graphicx}
\usepackage{colortbl}
\usepackage{xcolor} 
\usepackage{array} 
\usepackage{xcolor} 
\usepackage{nicematrix}
\usepackage{multirow}
\usepackage{placeins}
\usepackage{tikz}
\usetikzlibrary{arrows.meta, positioning, decorations.pathreplacing}
\usepackage{pgfplots}
\usepackage{pgfplotstable}
\pgfplotsset{compat=1.18}
\definecolor{lapmergeblack}{HTML}{111111}
\definecolor{patchblue}{HTML}{1558A6}
\definecolor{tokenred}{HTML}{C62828}
\definecolor{compilegreen}{HTML}{2E7D32}
\definecolor{densepurple}{HTML}{6A1B9A}
\usepackage{amsthm}
\usepackage{algorithm}
\usepackage{algpseudocode}
\newtheorem{theorem}{Theorem}
\newtheorem{proposition}{Proposition}
\newtheorem{definition}{Definition}
\title{LapMerge: Differentiable Exact-Budget Token Merging in Linear Time}

\author{%
\hspace*{-\tabcolsep}\parbox[t]{\textwidth}{\centering
\bfseries
Magdalena Tr{\fontencoding{T1}\selectfont\k{e}}dowicz$^{1}$\quad
Aleksandra Nowak$^{1}$\quad
Tomasz Steifer$^{2,3}$\\
\L{}ukasz Struski$^{1}$\quad
Szymon P\l{}otka$^{1}$\quad
Jacek Tabor$^{1,3}$\\
\normalfont\footnotesize
$^{1}$Jagiellonian University in Krak\'ow\\
$^{2}$Institute of Fundamental Technological Research, Polish Academy of Sciences\\
$^{3}$Centre for Credible AI, Warsaw University of Technology\\
\texttt{magdalena.tredowicz@doctoral.uj.edu.pl}}\hspace*{-\tabcolsep}
}

\newcommand{\our}{LapMerge}

\iclrfinalcopy
\renewcommand{\headrulewidth}{0pt}
\begin{document}
\maketitle

\doparttoc 
\faketableofcontents 

\begin{figure}[h]
\centering
\vspace{-1.05cm}
\resizebox{0.87\linewidth}{!}{%
\begin{tikzpicture}[
  every node/.style={font=\small},
  tok/.style={draw=black!55, fill=blue!8, minimum width=0.28cm, minimum height=0.28cm, inner sep=0},
  mtok/.style={draw, fill=blue!45, minimum width=0.42cm, minimum height=0.42cm, inner sep=0},
  lbl/.style={font=\scriptsize, align=center},
  arr/.style={-{Stealth[length=1.0mm]}}
]

\begin{scope}[shift={(-0.931127,-1.239805)}, scale=0.918, transform shape]
\node[anchor=south west, inner sep=0] (img) at (0,0) {\includegraphics[width=2.0286cm]{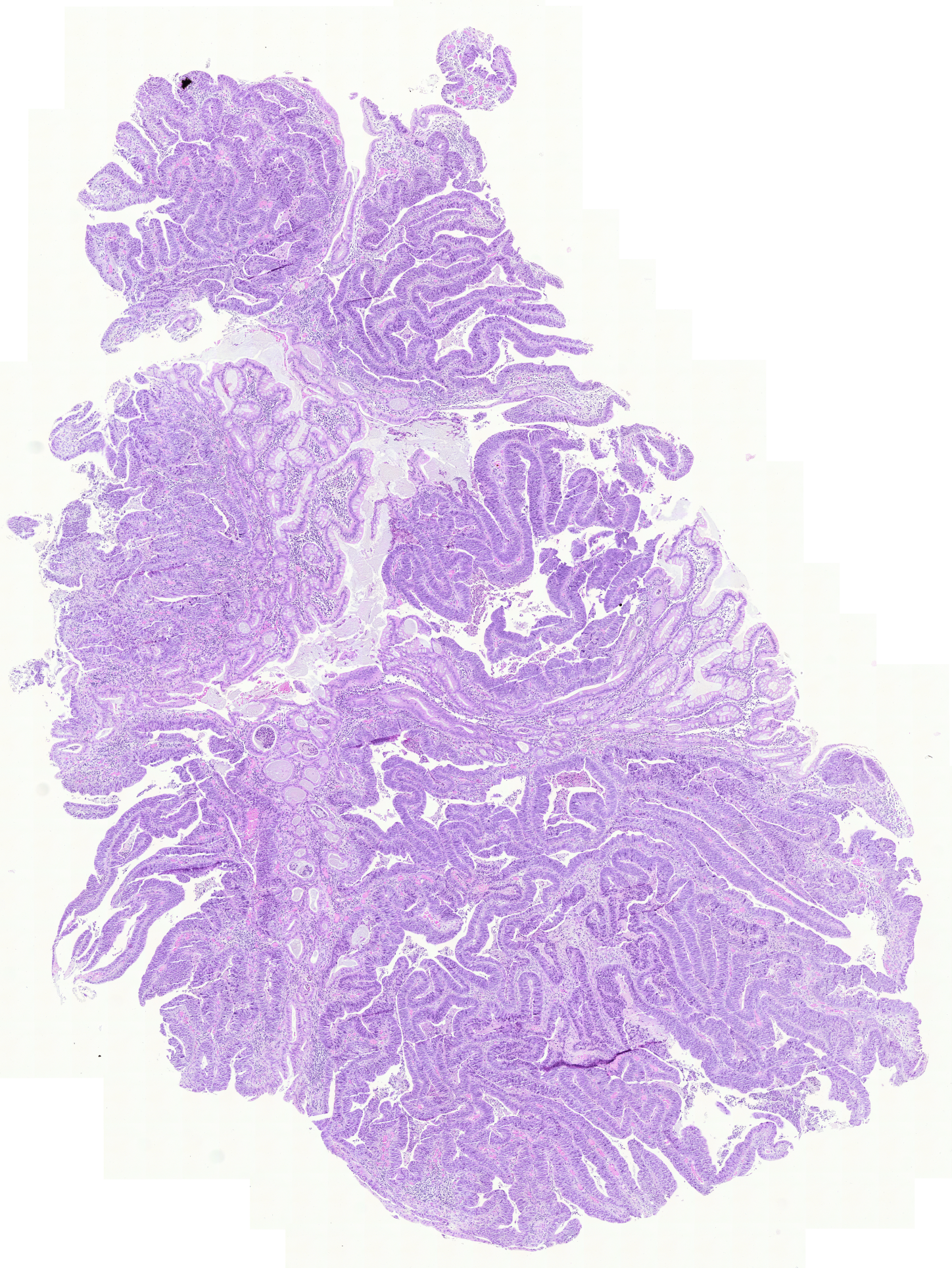}};
\draw[orange!85!black, line width=0.7pt, arr]
  (0.1691,0.1688) -- (0.5071,0.1688) -- (0.5071,0.5065) -- (0.1691,0.5065) -- (0.1691,0.8441) -- (0.1691,1.1817) -- (0.5071,1.1817) -- (0.5071,0.8441) -- (0.8453,0.8441) -- (0.8453,1.1817) -- (1.1834,1.1817) -- (1.1834,0.8441) -- (1.1834,0.5065) -- (0.8453,0.5065) -- (0.8453,0.1688) -- (1.1834,0.1688) -- (1.5214,0.1688) -- (1.8596,0.1688) -- (1.8596,0.5065) -- (1.5214,0.5065) -- (1.5214,0.8441) -- (1.8596,0.8441) -- (1.8596,1.1817) -- (1.5214,1.1817) -- (1.5214,1.5194) -- (1.8596,1.5194) -- (1.8596,1.8570) -- (1.5214,1.8570) -- (1.5214,2.1946) -- (1.8596,2.1946) -- (1.8596,2.5323) -- (1.5214,2.5323) -- (1.1834,2.5323) -- (0.8453,2.5323) -- (0.8453,2.1946) -- (1.1834,2.1946) -- (1.1834,1.8570) -- (1.1834,1.5194) -- (0.8453,1.5194) -- (0.8453,1.8570) -- (0.5071,1.8570) -- (0.5071,1.5194) -- (0.1691,1.5194) -- (0.1691,1.8570) -- (0.1691,2.1946) -- (0.5071,2.1946) -- (0.5071,2.5323) -- (0.1691,2.5323);
\end{scope}

\draw[arr, thick] (1.011127,0) -- (1.411127,0);

\begin{scope}[shift={(1.834607,0)}, rotate=90, scale=0.882, transform shape]
  \foreach \i in {0,...,13} {\node[tok] at (\i*0.26-1.69,0) {};}
\end{scope}

\draw[arr, thick] (2.208087,0) -- (2.608087,0);

\pgfmathdeclarefunction{lapsumV}{1}{%
  \pgfmathparse{%
    0.28*exp(-abs(#1-0*0.26)/0.14)+0.62*exp(-abs(#1-1*0.26)/0.14)+%
    0.34*exp(-abs(#1-2*0.26)/0.14)+0.80*exp(-abs(#1-3*0.26)/0.14)+%
    0.20*exp(-abs(#1-4*0.26)/0.14)+0.24*exp(-abs(#1-5*0.26)/0.14)+%
    0.66*exp(-abs(#1-6*0.26)/0.14)+0.30*exp(-abs(#1-7*0.26)/0.14)+%
    0.74*exp(-abs(#1-8*0.26)/0.14)+0.46*exp(-abs(#1-9*0.26)/0.14)+%
    0.20*exp(-abs(#1-10*0.26)/0.14)+0.70*exp(-abs(#1-11*0.26)/0.14)+%
    0.32*exp(-abs(#1-12*0.26)/0.14)+0.24*exp(-abs(#1-13*0.26)/0.14)%
  }%
}
\begin{scope}[shift={(4.016567,0)}, rotate=90, scale=0.882, transform shape]
  \begin{scope}[shift={(-1.69,0.54)}]
  \foreach \sl/\sr/\gc in {-0.26/0.5937/15,0.5937/1.3358/22,1.3358/2.0210/15,
                           2.0210/2.6933/22,2.6933/3.64/15} {
    \fill[gray!\gc] (\sl,0) -- plot[domain=\sl:\sr, samples=15] (\x,{-0.525*lapsumV(\x)})
      -- (\sr,0) -- cycle;
  }
  \foreach \i/\s in {0/0.28,1/0.62,2/0.34,3/0.80,4/0.20,5/0.24,6/0.66,7/0.30,
                     8/0.74,9/0.46,10/0.20,11/0.70,12/0.32,13/0.24} {
    \draw[orange!40, very thin] plot[domain=-0.38:0.38, samples=15, variable=\t]
      ({\i*0.26+\t},{-0.525*\s*exp(-abs(\t)/0.14)});
    \draw[black!55, semithick] (\i*0.26,0) -- (\i*0.26,-0.525*\s);
    \node[tok] at (\i*0.26,0.35) {};
  }
  \draw[orange!70, thick] plot[domain=-0.26:3.64, samples=86]
    (\x,{-0.525*lapsumV(\x)});
  \foreach \b in {0.5937,1.3358,2.0210,2.6933} {\draw[black, line width=0.5pt] (\b,-0.55) -- (\b,0.5);}
  \node[mtok, minimum width=0.257125cm, minimum height=0.257125cm] (vm0) at (0.13,-1.15) {};
  \node[mtok, minimum width=0.257125cm, minimum height=0.257125cm] (vm1) at (0.91,-1.15) {};
  \node[mtok, minimum width=0.257125cm, minimum height=0.257125cm] (vm2) at (1.69,-1.15) {};
  \node[mtok, minimum width=0.257125cm, minimum height=0.257125cm] (vm3) at (2.47,-1.15) {};
  \node[mtok, minimum width=0.257125cm, minimum height=0.257125cm] (vm4) at (3.25,-1.15) {};
  \draw[arr, gray] (0.1669,{-0.525*lapsumV(0.1669)}) -- (vm0);
  \draw[arr, gray] (0.9647,{-0.525*lapsumV(0.9647)}) -- (vm1);
  \draw[arr, gray] (1.6784,{-0.525*lapsumV(1.6784)}) -- (vm2);
  \draw[arr, gray] (2.3571,{-0.525*lapsumV(2.3571)}) -- (vm3);
  \draw[arr, gray] (3.1666,{-0.525*lapsumV(3.1666)}) -- (vm4);
  \end{scope}
\end{scope}
\node[lbl, align=center] at (4.016567,-2.2) {
    Laplace kernels \quad equal-area segments\\
    soft-merge to $N$
};
\draw[gray!50, thick] (2.7,-1.74) rectangle (4.9,1.71);
\node[font=\scriptsize, inner sep=2pt] at (3.886541,1.98) {\textsc{\our{}}};

\draw[arr, thick] (5.017959,0) -- (5.417959,0);

\begin{scope}[shift={(6.414398,0)}, scale=1.2006225, transform shape]
  \foreach \i in {0,...,4} {\coordinate (vt\i) at (72*\i+90:0.62);}
  \foreach \i in {0,...,4} {
    \foreach \j in {0,...,4} {
      \ifnum\i<\j \draw[gray!65, line width=0.25pt] (vt\i) -- (vt\j); \fi
    }
  }
  \foreach \i in {0,...,4} {\node[draw, fill=blue!45, minimum width=0.17cm, minimum height=0.17cm, inner sep=0] at (vt\i) {};}
\end{scope}

\draw[arr, thick] (7.340837,0) -- (7.740837,0);

\begin{scope}[shift={(7.820837,-1.239805)}, scale=0.918, transform shape]
\node[anchor=south west, inner sep=0] (imgE) at (0,0) {\includegraphics[width=2.0286cm]{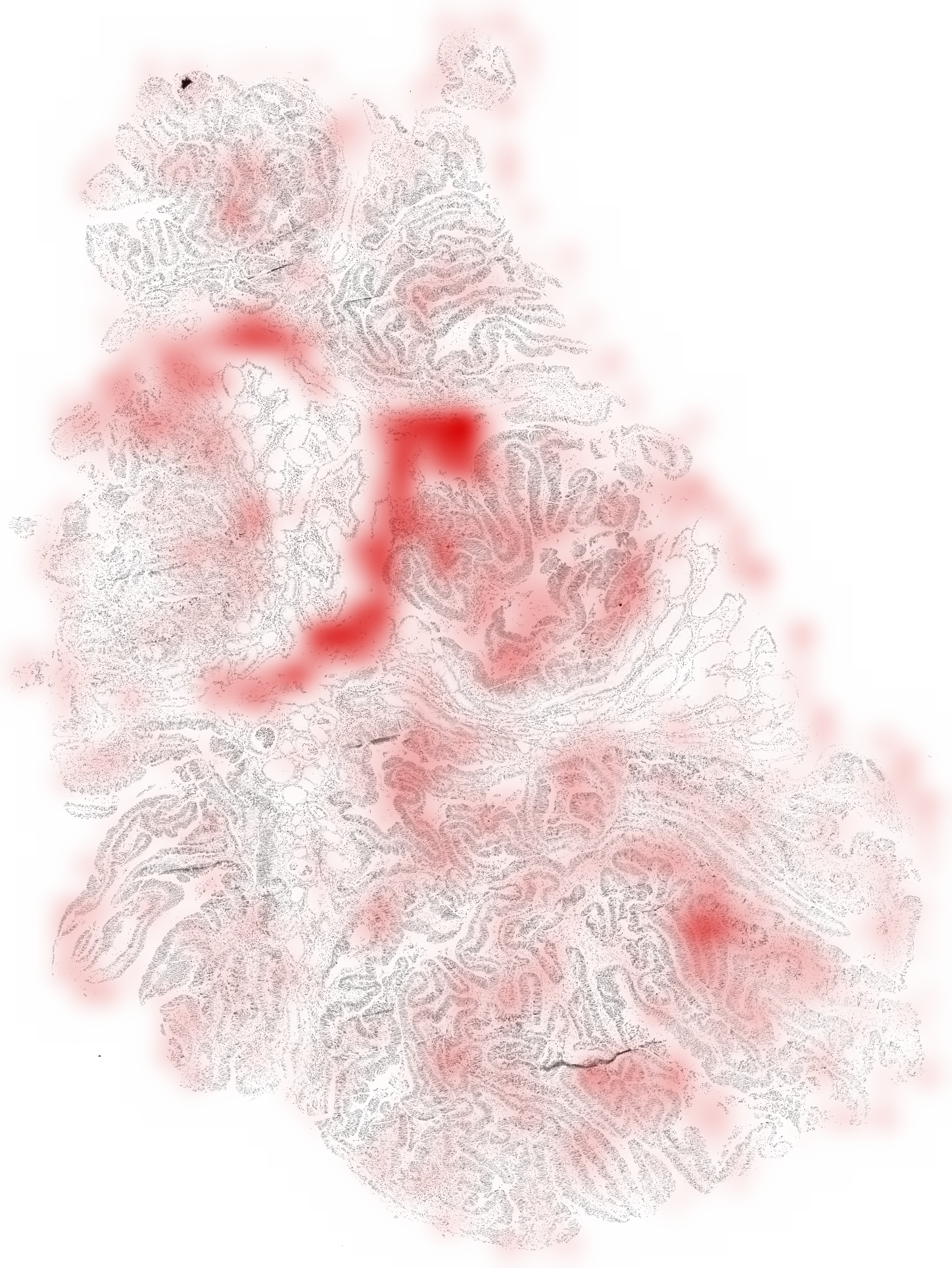}};
\end{scope}

\end{tikzpicture}
}
\vspace{-0.5cm}
\caption{%
  \textbf{Overview of the standard \our{} pipeline.}
Given an ordered sequence of $M$ input tokens, \our{} produces exactly $N$ budget-exact, soft-merged output tokens. A learned scorer assigns each input token a mass, represented as a Laplace kernel. The summed density is scanned cumulatively and partitioned into $N$ equal-mass segments, each of which is collapsed into a single output token. The resulting $N$ tokens are then processed by a Transformer classifier. The rightmost panel shows a real histopathology slide with the learned scorer weights overlaid on the tissue; opacity indicates the importance attributed to each patch. These same scores determine the token merging and can therefore be read as a learned routing-mass map: a direct view of where the model allocates its token budget.
}
\label{fig:pipeline}
\vspace{-0.3cm}
\end{figure}
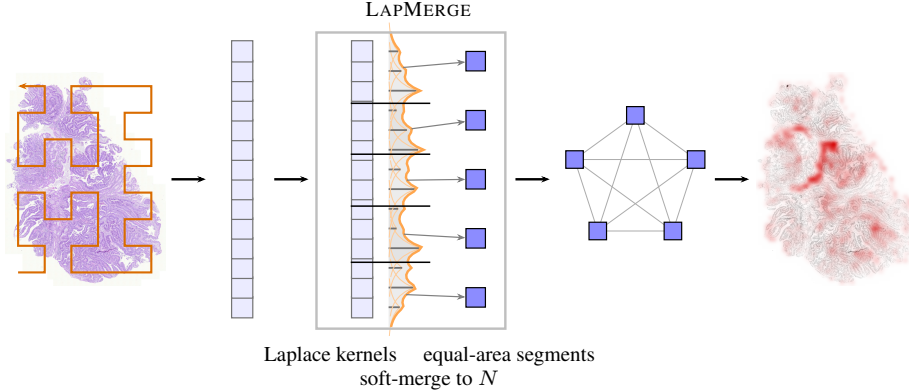

\begin{abstract}

Modern vision and language models process ever longer token sequences, and reducing them is the main lever on downstream cost. Token compression, however, faces a fundamental trade-off: selecting tokens discards information, while existing merging methods preserve it only through pairwise token-to-token matching or $MN$ token-to-output interactions. We introduce \our{}, a differentiable merging operator that resolves this trade-off for ordered sequences: a learned mass distribution partitions the sequence into exactly $N$ equal-mass segments under a Laplace-kernel density, and prefix/suffix scans aggregate each segment, so both the forward and the backward pass run in $\mathcal{O}(M+N)$ time. This combination -- an exact output budget, end-to-end differentiability, and linear cost -- is precisely what previous reducers trade off against one another. The budget can be changed without changing operator parameters, and the learned mass doubles as a token-level map of where the budget is allocated. Across vision and vision-language tasks, \our{} matches or improves on strong token reducers while being up to two orders of magnitude faster at large sequence lengths; it is particularly effective when token order reflects spatial or temporal locality, and works both as a drop-in reduction module and as a trainable architectural component.

\end{abstract}

\section{Introduction}
\label{sec:intro}

Modern vision, language, and multimodal encoders increasingly operate on long token sequences. High-resolution images yield hundreds of patches, documents can contain thousands of tokens, and videos produce long spatiotemporal traces. While efficient attention implementations reduce the memory-access overhead of Transformer computation~\citep{dao2022flashattention}, they do not change the fundamental dependence of computation on sequence length. With a frozen encoder, the key efficiency problem is to reduce the number of tokens processed while preserving the information the downstream task needs.

A large body of work addresses this problem through token reduction. Existing approaches predominantly follow one of two strategies: pruning or merging. Pruning estimates the importance of individual tokens and retains only a subset, discarding the remainder~\citep{rao2021dynamicvit,liang2022evit,fayyaz2022ats}. This makes pruning computationally attractive: token scores are computed in a single pass, and selecting the top-$N$ tokens requires no interaction between input and output tokens. Its limitation is information loss -- pruning suits settings where the relevant evidence is concentrated in a small number of tokens, but can be brittle when the prediction depends on multiple weak cues distributed across the input.

Merging takes the opposite approach, aggregating tokens into a smaller set of representations rather than discarding them~\citep{ryoo2021tokenlearner,renggli2022patchmerger,bolya2023tome,tran2024pitome}. This can preserve information distributed across many tokens. However, existing learned merging mechanisms generally require computing token-to-token or token-to-output interactions. For example, PatchMerger~\citep{renggli2022patchmerger} and TokenLearner~\citep{ryoo2021tokenlearner} construct learned aggregation weights over the input tokens, while Perceiver~\citep{jaegle2021perceiver}-style architectures use cross-attention between a fixed set of latent queries and the input sequence. Such global aggregation explicitly evaluates $M N$ token-to-output interactions. When the output budget is a fixed fraction of the input, $N=\rho M$, this becomes quadratic in $M$; even when $N$ is fixed, the method still carries an $M N$ interaction cost.

This raises a natural question: can token reduction preserve information from all $M$ inputs while producing exactly $N$ outputs without an explicit $M\times N$ interaction? We answer this by exploiting sequence order. \our{} learns a scalar mass at each position and uses its cumulative distribution to partition the sequence into $N$ contiguous, equal-mass regions, each producing one output token. This structured assignment replaces unrestricted token-to-output interactions with prefix/suffix scans, yielding $\mathcal{O}(M+N)$ cost in both forward and backward passes for given masses and ordered positions; producing the masses with the soft top-$k$ router used in our experiments \citep{struski2025lapsum} adds an $O(M\log M)$ sort, which Fast LapSum~\citep{struski2026fastlapsum} makes expected-linear and pointwise sigmoid or softplus routing avoids entirely. The resulting outputs are convex combinations of their inputs, with a log-space implementation ensuring numerical stability at long sequence lengths. Figure~\ref{fig:pipeline} gives an overview of the method.

The key inductive bias is locality: nearby positions are preferentially aggregated together. This makes \our{} particularly suited to spatially or temporally ordered tokens, while offering less advantage when sequence order is unrelated to semantic structure. In practice, the operator is implemented in native PyTorch, optionally accelerated with CUDA, and scales to sequences of more than $10^7$ tokens. Our main contributions:

\begin{itemize}
    \item We introduce \our{}, a token merging operator that combines three properties prior reducers trade off against one another: an exact output budget of $N$ tokens, end-to-end differentiability with closed-form gradients, and $\mathcal{O}(M+N)$ cost in both passes.
    \item We show that \our{} is versatile: a training-free drop-in that improves a released vision-language pipeline, and a trainable component in from-scratch and frozen-backbone models.
    \item Across vision and vision-language benchmarks, \our{} matches or improves on strong token reducers at equal budgets while scaling beyond $10^7$-token inputs, up to two orders of magnitude faster than matched SDPA poolers.
\end{itemize}

\section{Related Work}
\label{sec:related}
Prior work addresses the cost of long token sequences along three main axes: selecting a subset of tokens to retain, aggregating multiple tokens into fewer representations, or modifying the sequence model to process long inputs more efficiently.

\vspace{-0.3cm}\paragraph{Token pruning.}\hspace{-0.3cm} Token pruning reduces computation by identifying and discarding tokens that are unlikely to contribute to the downstream task. DynamicViT~\citep{rao2021dynamicvit}, EViT~\citep{liang2022evit}, ATS~\citep{fayyaz2022ats}, and DiffRate~\citep{chen2023diffrate} learn token importance and retain a subset of the input. More recent methods consider the frozen-backbone setting, including VisionSelector~\citep{visionselector2025}, AutoSelect~\citep{autoselect2026}, and DiffPrune~\citep{diffprune2026}. These approaches are closely related to ours: they learn a lightweight reduction policy under a fixed token budget, but use it to select tokens rather than aggregate them. Differentiable top-$k$ selection itself has been studied extensively through continuous relaxations and differentiable sorting~\citep{grover2019neuralsort,prillo2020softsort,blondel2020fast,cordonnier2021differentiable}. \our{} instead uses the differentiable budget assignment of LapSum~\citep{struski2025lapsum} to determine how the output budget is distributed across an ordered sequence.

\vspace{-0.3cm}\paragraph{Token merging and pooling.}\hspace{-0.3cm} Token merging provides an alternative to pruning by combining multiple input tokens into a smaller set of representations. ToMe~\citep{bolya2023tome} uses bipartite matching to merge similar tokens, while PiToMe~\citep{tran2024pitome} extends this approach with improved matching. Other approaches perform learned or structured pooling. Token Pooling~\citep{marin2023token} uses clustering, TokenLearner~\citep{ryoo2021tokenlearner} and PatchMerger~\citep{renggli2022patchmerger} learn attention-based pooling, and Perceiver~\citep{jaegle2021perceiver} and Flamingo~\citep{alayrac2022flamingo} use cross-attention from a fixed set of latent representations. Set Transformer~\citep{lee2019set} similarly introduces learned inducing points for scalable set aggregation.

\vspace{-0.3cm}\paragraph{Efficient sequence processing.}\hspace{-0.3cm} A complementary line of work reduces the cost of processing long sequences by modifying the sequence model rather than reducing its input. Sparse and linear attention~\citep{child2019sparse,katharopoulos2020linear,choromanski2021performer} restrict or approximate token interactions, while Longformer~\citep{beltagy2020longformer} introduces sparse attention patterns. Funnel Transformer~\citep{dai2020funnel} progressively reduces sequence length within the model, and state-space models such as S4~\citep{gu2022s4} and Mamba~\citep{gu2023mamba} provide alternative mechanisms for processing long sequences. These approaches are orthogonal to token reduction: \our{} can be applied before the sequence model and combined with different architectures.

Existing work primarily explores either which tokens to discard or how to aggregate tokens. \our{} instead asks how the structure already present in an ordered token stream can be used to make aggregation both efficient and adaptive. This places the method between pruning and merging, preserving distributed evidence through aggregation while exploiting order to avoid a global assignment.

\section{\our{}: Laplace Token Merging}
\label{sec:method}

This section gives the main construction and explains why the Laplace kernel
makes exact-budget merging linear in sequence length. We keep the derivations
brief here; Appendix~\ref{app:algorithm} gives the complete construction,
proofs, scan algorithms, and backward formulas.

Throughout the paper, $\our{}$ denotes the standard equal-mass barrier
operator. Its input positions are assumed to arrive in the order in which
they are to be processed; for an ordinary sequence this is simply the original
token order, and no additional sorting is performed by the operator. The
deployment variant used in POPE -- the same operator driven by an external
scorer, with output positions taken from hard selection -- is described
separately as \emph{score-routed \our{}}.

\paragraph{A natural approach to merging, and why it is not differentiable.}\hspace{-0.3cm}
Suppose a lightweight scoring head has already told us how important every token is. Let $(x_i)_{i=1}^M$, $x_i\in\mathbb{R}^d$, denote the input tokens, $p_1<\cdots<p_M$ their positions in the fixed scan order, and $r_i\ge0$ its importance weight. A natural idea suggests itself: accumulate these weights $r_i$ in scan order, cut the sequence wherever the running sum crosses a fixed increment, and average the tokens inside every one of the resulting batches into a single output token -- the pipeline sketched in Figure~\ref{fig:barriers}, top. Do this so that the running sum crosses exactly $N$ increments and we get exactly $N$ output tokens, and we have almost the object we want -- an adaptive partition that spends more of a fixed output budget on high-scoring regions, computed without ever comparing every input token to every output token (the same cumulative-score sweep already selects tokens for ATS~\citep{fayyaz2022ats}, keeping one nearest token per threshold instead of averaging a segment). As stated, though, it is not easy to use and not differentiable: the cut points are level sets of a step function -- the running sum jumps at every token position $p_i$ and is flat in between, so an infinitesimal change to any one score can move a boundary discontinuously past a neighboring token, and the gradient of everything downstream with respect to the scores is zero almost everywhere and undefined exactly at the jumps. The bigger problem hides one step earlier: we wanted to \emph{learn} the scoring head end to end, and a scoring head sitting behind a dead gradient receives no training signal at all -- no amount of averaging inside the batches fixes that, so the scorer itself would never move.

\begin{wrapfigure}{r}{0.4\linewidth}
\centering
\vspace{-0.5cm}
\resizebox{5.55cm}{!}{%
\begin{tikzpicture}
\pgfmathdeclarefunction{gaussBump}{2}{\pgfmathparse{exp(-0.5*((#1-#2)/0.25)^2)}}
\pgfmathdeclarefunction{sumGauss}{1}{\pgfmathparse{
  0.30*gaussBump(#1,0)+0.35*gaussBump(#1,1)+0.30*gaussBump(#1,2)+
  0.98*gaussBump(#1,3)+0.52*gaussBump(#1,4)+0.55*gaussBump(#1,5)
}}
\node[align=center] at (2.5,1.65) {\small\itshape hard partition};
\draw[black!40] (-0.7,0) -- (5.7,0);
\foreach \pos/\wt/\col in {0/0.30/blue!15,1/0.35/blue!15,2/0.30/blue!15,
                           3/0.98/orange!20,4/0.52/green!25,5/0.55/green!25} {
  \filldraw[fill=\col, draw=black!60] ({\pos-0.16},0) rectangle ({\pos+0.16},{1.3*\wt});
  \draw[black!60, thick] (\pos,0) -- (\pos,{1.3*\wt});
  \filldraw[black] (\pos,0) circle (1.6pt);
}
\draw[black, thin] (2.5,-0.4) -- (2.5,1.4);
\draw[black, thin] (3.5,-0.4) -- (3.5,1.4);
\draw[decorate, decoration={brace, amplitude=4pt, mirror}] (-0.7,-0.45) -- (2.4,-0.45)
  node[midway, below=5pt] {\small $\approx\frac1N$};
\draw[decorate, decoration={brace, amplitude=4pt, mirror}] (2.6,-0.45) -- (3.4,-0.45)
  node[midway, below=5pt] {\small $\approx\frac1N$};
\draw[decorate, decoration={brace, amplitude=4pt, mirror}] (3.6,-0.45) -- (5.7,-0.45)
  node[midway, below=5pt] {\small $\approx\frac1N$};
\begin{scope}[yshift=-3.75cm]
\node[align=center] at (2.5,1.65) {\small\itshape soft (continuous) relaxation};
\fill[blue!15] (-0.7,0) -- plot[domain=-0.7:2.5973, samples=80] (\x,{1.3*sumGauss(\x)}) -- (2.5973,0) -- cycle;
\fill[orange!20] (2.5973,0) -- plot[domain=2.5973:3.7277, samples=40] (\x,{1.3*sumGauss(\x)}) -- (3.7277,0) -- cycle;
\fill[green!25] (3.7277,0) -- plot[domain=3.7277:5.7, samples=80] (\x,{1.3*sumGauss(\x)}) -- (5.7,0) -- cycle;
\draw[black!40] (-0.7,0) -- (5.7,0);
\foreach \pos/\wt in {0/0.30,1/0.35,2/0.30,3/0.98,4/0.52,5/0.55} {
  \draw[black!60, thin] plot[domain=-0.7:5.7, samples=60] (\x,{1.3*\wt*gaussBump(\x,\pos)});
  \draw[black!60, thick] (\pos,0) -- (\pos,{1.3*\wt});
  \filldraw[black] (\pos,0) circle (1.6pt);
}
\draw[thick, orange!80!black] plot[domain=-0.7:5.7, samples=120] (\x,{1.3*sumGauss(\x)});
\draw[black, thin] (2.5973,-0.4) -- (2.5973,1.4);
\draw[black, thin] (3.7277,-0.4) -- (3.7277,1.4);
\draw[decorate, decoration={brace, amplitude=4pt, mirror}] (-0.7,-0.45) -- (2.4973,-0.45)
  node[midway, below=5pt] {\small $=\frac1N$};
\draw[decorate, decoration={brace, amplitude=4pt, mirror}] (2.6973,-0.45) -- (3.6277,-0.45)
  node[midway, below=5pt] {\small $=\frac1N$};
\draw[decorate, decoration={brace, amplitude=4pt, mirror}] (3.8277,-0.45) -- (5.7,-0.45)
  node[midway, below=5pt] {\small $=\frac1N$};
\end{scope}
\end{tikzpicture}
}
\vspace{-0.6cm}
\caption{\small The same six tokens cut into $N=3$ pieces two ways -- by token \emph{count} they are 3-1-2, but by \emph{mass} all three pieces are comparable. \emph{Top}: the hard partition from the running sum -- barriers can only fall between tokens, so each piece's mass is merely \emph{close} to $1/N$. \emph{Bottom}: treating positions as noisy turns each token into a smooth bump; the shaded area under the resulting mixture density (bold) is \emph{exactly} $1/N$ per piece (Proposition~\ref{prop:equal-mass-boundaries}). These intervals partition probability mass, not tokens: one token's bump can span several intervals. The kernel shown is Gaussian, purely for illustration.}
\label{fig:barriers}
\vspace{-0.6cm}
\end{wrapfigure}
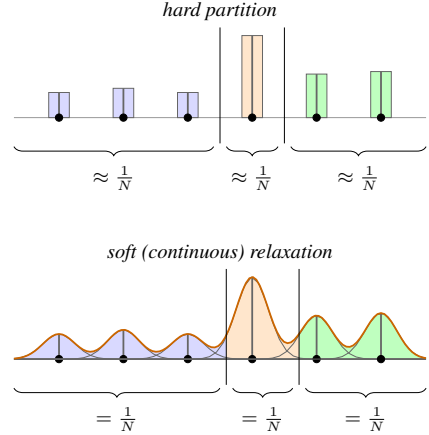

\vspace{-0.3cm}\paragraph{Softening the problem.}\hspace{-0.3cm}
One repairs this by replacing every token's point mass with uncertainty about its position: a fixed location becomes a random position with a continuous distribution around it, as in kernel density estimation -- the soft relaxation sketched in Figure~\ref{fig:barriers}, bottom, where each token's mass spreads into a smooth bump and the resulting mixture density is cut into intervals of exactly equal mass, rather than the top panel's discontinuous running-sum cuts. The same hard-to-soft idea can be found e.g. in speech-boundary segmentation~\citep{hwang2024variable} where it is achieved via $\mathcal{O}(MN)$ attention with no well-posedness guarantee. We want a closed-form, provably well-posed version instead, with $\mathcal{O}(M+N)$ time cost. The rest of this section explains how to achieve that.

\vspace{-0.3cm}\paragraph{A formal construction.}\hspace{-0.3cm}
Let $\kappa$ denote the density of this position noise -- a continuous probability density on $\mathbb{R}$ -- with CDF $\sigma$ and scale $\alpha>0$, and attach token $x_i$ to a random position $P_i=p_i+\alpha E_i$ in place of the fixed $p_i$, with $E_i$ drawn i.i.d.\ from the distribution with CDF $\sigma$. A smaller $\alpha$ concentrates mass near each $p_i$, whereas a larger $\alpha$ spreads each token's mass across more intervals. Draw an index $I$ with $\Pr(I=i)=r_i$, where $r_i\ge0$ and $\sum_{i=1}^Mr_i=1$ are the normalized token weights: the resulting random position $P_I$ is then distributed with CDF
\begin{equation}
\label{eq:mix}
\Phi(x)
=
\sum_{i=1}^{M}
r_i\,\sigma\!\left(\frac{x-p_i}{\alpha}\right),
\end{equation}
a smooth, blurred version of the running sum from the earlier paragraphs. We cut $\mathbb{R}$ into $N$ intervals $[b_{\ell-1},b_\ell)$, $\ell=1,\ldots,N$ (with $b_0=-\infty$ and $b_N=+\infty$), each capturing probability exactly $1/N$ under $\Phi$. The interior barriers satisfy
\begin{equation}
\label{eq:bar}
\Phi(b_\ell)=\ell/N,\qquad 1\leq\ell<N.
\end{equation}
The intervals partition the mass of $P_I$, not the set of input tokens; a token may contribute to more than one interval. The barriers exist uniquely and depend differentiably on the routing weights.

\begin{proposition}
\label{prop:equal-mass-boundaries}
If $\kappa$ is continuous and strictly positive, the boundaries
$\Phi(b_\ell)=\ell/N$, $1\leq\ell<N$, exist uniquely and
the map $r\mapsto b_\ell$ is differentiable on the probability simplex.
For any tangent $\dot r$ with $\sum_j\dot r_j=0$, its derivative is
$\dot b_\ell=\sum_j(\partial b_\ell/\partial r_j)\dot r_j$, where one
representative of the partial derivatives is
\begin{equation}
\label{eq:jvp}
\frac{\partial b_\ell}{\partial r_j}
=
-\,\frac{\sigma\big((b_\ell-p_j)/\alpha\big)}{\Phi'(b_\ell)}.
\end{equation}
\end{proposition}

\begin{proof}[Proof sketch]
Since
\[
    \Phi'(x)
    = \frac{1}{\alpha}\sum_{i=1}^{M}r_i\,
      \kappa\!\left(\frac{x-p_i}{\alpha}\right)
    > 0,
\]
$\Phi$ is continuously differentiable and strictly increasing,
with limits $0$ and $1$ at $-\infty$ and $+\infty$.
Each level $\ell/N$ therefore determines a unique boundary,
and consecutive boundaries enclose mass $1/N$.
The implicit function theorem gives differentiability of the
boundaries. Differentiating
$\Phi(b_\ell(t);r(t))=\ell/N$ gives the displayed formula.
Appendix~\ref{app:algorithm-barriers} gives the complete proof and explains
the simplex constraint.
\end{proof}

With the barriers in hand, define the merged tokens as the conditional expectation of the randomly drawn token given it lands in each interval:
\begin{definition}[Merged tokens]
\label{def:merge}
$$t_\ell \;:=\; \mathbb{E}\!\left[x_I\mid P_I\in[b_{\ell-1},b_\ell)\right],\qquad \ell=1,\ldots,N.$$
\end{definition}
\vspace{-0.2cm}

We can note that each $t_\ell$ is a convex combination of the input tokens -- never an extrapolation outside their span. To be precise, we can obtain
\begin{equation}
\label{eq:membership}
w_i^\ell
=
N r_i
\left[
\sigma\!\left(\frac{b_\ell-p_i}{\alpha}\right)
-
\sigma\!\left(\frac{b_{\ell-1}-p_i}{\alpha}\right)
\right],
\qquad
t_\ell
=
\sum_{i=1}^{M}w_i^\ell x_i,
\end{equation}
Figure~\ref{fig:merging} illustrates a single interval: token $i$ contributes its kernel's probability mass inside $[b_{\ell-1},b_\ell)$, and may also contribute to other intervals. Summing \eqref{eq:membership} over $\ell$ telescopes the CDF differences to one, giving $N^{-1}\sum_\ell t_\ell=\sum_i r_i x_i$: the operator preserves the weighted first moment of the input. This is a conservation property, not information preservation -- at $N<M$ the aggregation is not invertible, and a token with $r_i=0$ contributes nothing.

\paragraph{Standard versus score-routed \our{}.} The construction above is
the standard \our{} operator with a single kernel width; for simplicity we
present this base version. It also admits separate widths for routing
($\alpha_r$, barriers) and read-out ($\alpha_a$, with per-interval
normalization), evaluated by the same scans in $\mathcal{O}((M+N)d)$;
the general form, its properties and cost are in
Appendix~\ref{app:pope-variant}. The POPE experiment uses this form with an
external scorer: per-image standardized VisionSelector scores enter a
LapSum soft top-$k$ at temperature $\alpha_r$, giving budget-exact masses
that sum to $k$ at the raster indices. The $k$ outputs are injected at the
rotary positions of the $k$ highest-scoring tokens, so positions come from
hard selection while values come from the merge; as both widths go to zero,
the pipeline coincides with the released selection pipeline. We refer to
this deployment as \emph{score-routed \our{}} below.

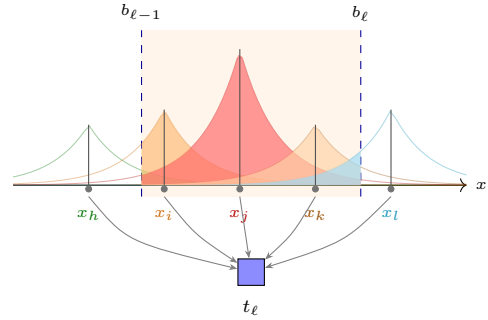
\begin{wrapfigure}{r}{0.48\linewidth}
\centering
\vspace{-0.3cm}
\begin{tikzpicture}[scale=1.0]
\pgfmathdeclarefunction{densH}{1}{\pgfmathparse{(0.4/0.5)*exp(-abs((#1-0)/0.5))}}
\pgfmathdeclarefunction{densI}{1}{\pgfmathparse{(0.5/0.5)*exp(-abs((#1-1)/0.5))}}
\pgfmathdeclarefunction{densJ}{1}{\pgfmathparse{(0.9/0.5)*exp(-abs((#1-2)/0.5))}}
\pgfmathdeclarefunction{densK}{1}{\pgfmathparse{(0.4/0.5)*exp(-abs((#1-3)/0.5))}}
\pgfmathdeclarefunction{densL}{1}{\pgfmathparse{(0.5/0.5)*exp(-abs((#1-4)/0.5))}}
\fill[orange!8] (0.7,-0.15) rectangle (3.6,2.05);
\draw[dashed, blue!60!black] (0.7,-0.15) -- (0.7,2.05);
\draw[dashed, blue!60!black] (3.6,-0.15) -- (3.6,2.05);
\node[above] at (0.7,2.05) {\tiny $b_{\ell-1}$};
\node[above] at (3.6,2.05) {\tiny $b_\ell$};
\draw[->] (-1,0) -- (5,0) node[right] {\tiny $x$};
\draw[green!45!black, thin, opacity=0.35] plot[domain=-1:5,samples=90] (\x,{densH(\x)});
\draw[orange!80!black, thin, opacity=0.35] plot[domain=-1:5,samples=90] (\x,{densI(\x)});
\draw[red!70!black, thin, opacity=0.35] plot[domain=-1:5,samples=90] (\x,{densJ(\x)});
\draw[orange!60!black, thin, opacity=0.35] plot[domain=-1:5,samples=90] (\x,{densK(\x)});
\draw[cyan!70!black, thin, opacity=0.35] plot[domain=-1:5,samples=90] (\x,{densL(\x)});
\fill[green!45, opacity=0.65] (0.7,0) -- plot[domain=0.7:3.6,samples=60] (\x,{densH(\x)}) -- (3.6,0) -- cycle;
\fill[orange!55, opacity=0.65] (0.7,0) -- plot[domain=0.7:3.6,samples=60] (\x,{densI(\x)}) -- (3.6,0) -- cycle;
\fill[red!60, opacity=0.65] (0.7,0) -- plot[domain=0.7:3.6,samples=60] (\x,{densJ(\x)}) -- (3.6,0) -- cycle;
\fill[orange!35, opacity=0.65] (0.7,0) -- plot[domain=0.7:3.6,samples=60] (\x,{densK(\x)}) -- (3.6,0) -- cycle;
\fill[cyan!35, opacity=0.65] (0.7,0) -- plot[domain=0.7:3.6,samples=60] (\x,{densL(\x)}) -- (3.6,0) -- cycle;
\draw[black!70, thin] (0,0) -- (0,{densH(0)});
\draw[black!70, thin] (1,0) -- (1,{densI(1)});
\draw[black!70, thin] (2,0) -- (2,{densJ(2)});
\draw[black!70, thin] (3,0) -- (3,{densK(3)});
\draw[black!70, thin] (4,0) -- (4,{densL(4)});
\foreach \pos in {0,1,2,3,4} {\filldraw[black!55] (\pos,-0.05) circle(1.1pt);}
\node[below, text=green!45!black] at (0,-0.2) {\tiny $x_h$};
\node[below, text=orange!80!black] at (1,-0.2) {\tiny $x_i$};
\node[below, text=red!70!black] at (2,-0.2) {\tiny $x_j$};
\node[below, text=orange!60!black] at (3,-0.2) {\tiny $x_k$};
\node[below, text=cyan!70!black] at (4,-0.2) {\tiny $x_l$};
\node[draw, fill=blue!45, minimum width=0.35cm, minimum height=0.35cm, inner sep=0] (tl) at (2.15,-1.15) {};
\node[below] at (2.15,-1.4) {\scriptsize $t_\ell$};
\draw[-{Stealth[length=1.1mm]}, gray] (0,-0.15) .. controls (0.5,-0.75) .. (tl);
\draw[-{Stealth[length=1.1mm]}, gray] (1,-0.15) .. controls (1.4,-0.65) .. (tl);
\draw[-{Stealth[length=1.1mm]}, gray] (2,-0.15) -- (tl);
\draw[-{Stealth[length=1.1mm]}, gray] (3,-0.15) .. controls (2.7,-0.65) .. (tl);
\draw[-{Stealth[length=1.1mm]}, gray] (4,-0.15) .. controls (3.4,-0.75) .. (tl);
\end{tikzpicture}
\vspace{-0.25cm}
\caption{\small \emph{Why the Laplace kernel is fast}: the tick at each kernel's peak splits it into a left half and a right half, and Equation~\ref{eq:split} aggregates every left half with one prefix sum and every right half with one suffix sum -- never a per-pair computation. Five tokens shown: three peaks fall inside $[b_{\ell-1},b_\ell)$ ($x_i,x_j,x_k$) and two fall outside ($x_h,x_l$); only the shaded, interval-clipped part of each counts toward $t_\ell$ (Definition~\ref{def:merge}).}
\label{fig:merging}
\vspace{-0.9cm}
\end{wrapfigure}

\vspace{-0.3cm}\paragraph{LapMerge: a fast, exact solution to merging.}\hspace{-0.3cm}
Everything above holds for any strictly positive and continuous $\kappa$, and we have not yet said how the weights $r_i$ themselves are obtained from raw scores. Two problems remain, both computational:
\begin{itemize}
\item Under a generic kernel, no closed form exists for the boundaries $b_\ell$: locating each of the $N-1$ barriers by bisecting the length-$M$ mixture costs $\mathcal{O}(M)$ per barrier, $\mathcal{O}(MN)$ in total -- reintroducing exactly the token-to-output interaction we set out to avoid.
\item Even given the barriers, evaluating the conditional expectation $t_\ell$ (Definition~\ref{def:merge}) generically sums over all $M$ tokens for each of the $N$ outputs -- again $\mathcal{O}(MN)$.
\end{itemize}

Both are solved by the same choice of kernel: the Laplace distribution, $\kappa(z)=\tfrac12e^{-|z|}$, with CDF $\sigma(z)=\tfrac12e^{z}$ for $z\le0$ and $\sigma(z)=1-\tfrac12e^{-z}$ for $z\ge0$, which is strictly positive, as we need. Using and generalizing the prefix/suffix tricks that make LapSum~\citep{struski2025lapsum} fast, both problems drop to $\mathcal{O}(M+N)$. A fused prefix/suffix scan provides further GPU acceleration.

The merging operator itself does not require raw scores or a particular
scorer. It takes as input the ordered positions $p_i$ and non-negative routing
masses $r_i$ satisfying $\sum_i r_i=1$:
\begin{equation}
\label{eq:route}
r_i\ge0,\qquad \sum_{i=1}^M r_i = 1.
\end{equation}
These masses may be supplied by a learned scorer, a differentiable
soft-top-$k$ routing module, or a fixed external policy. Such a routing choice
is optional and upstream of LapMerge; the equations below are valid for any
admissible $r$. Substituting the two branches of the Laplace CDF into the
definition of $\Phi$ separates the tokens according to whether
$p_i\leq x$ or $p_i>x$. Factoring out the terms that depend on
$x$ turns Equation \ref{eq:mix} into
\begin{equation}
\label{eq:split}
\Phi(x)
=
\underbrace{\textstyle\sum_{p_i\le x}r_i}_{\text{prefix}}
\;-\;
\tfrac12e^{-x/\alpha}\underbrace{\textstyle\sum_{p_i\le x}r_i e^{p_i/\alpha}}_{\text{prefix}}
\;+\;
\tfrac12e^{x/\alpha}\underbrace{\textstyle\sum_{p_i> x}r_i e^{-p_i/\alpha}}_{\text{suffix}}.
\end{equation}
Because the tokens are already sorted by $p_i$, all three underbraced sums are cumulative sums of the sequence (or its reverse in the suffix case), precomputed with one pass in each direction. The values $\Phi(p_i)$ are strictly increasing, as are the target levels $\ell/N$; merging these two ordered lists locates all barrier cells in $O(M+N)$ time. Inside a cell the sums are constant, so $\Phi(x)=\ell/N$ reduces to one quadratic equation in $u=e^{x/\alpha}$, solved in closed form. We reuse the same prefix/suffix machinery not to define the scorer, but to invert the mixture induced by the supplied $r_i$ and locate the $b_\ell$ themselves.

\vspace{-0.3cm}\paragraph{Aggregation, and the resulting cost.}\hspace{-0.3cm}
The merged tokens can be computed directly from the cumulative
feature sums
\[
\Phi_x(z)=\sum_{i=1}^{M} r_i x_i
\sigma\!\left(\frac{z-p_i}{\alpha}\right),
\qquad
t_\ell=N\bigl[\Phi_x(b_\ell)-\Phi_x(b_{\ell-1})\bigr].
\]
Applying~\eqref{eq:split} coordinatewise with $r_i x_i$ in place
of $r_i$ gives additional prefix/suffix tables of the same
form. Signed feature coordinates are scanned as separate positive and
negative parts (Appendix~\ref{app:algorithm}). Reusing the previously
identified boundary cells, we compute all merged tokens in
$O((M+N)d)$ time without forming individual membership weights.
The gradients needed for backpropagation have the same complexity,
so linear scaling is preserved during training.

\begin{theorem}
\label{thm:laplace-linear}
For the Laplace kernel and fixed ordered positions, the map
$(r,x)\mapsto(b,t)$ and each of its Jacobian-vector and
vector-Jacobian products can be evaluated in $O((M+N)d)$ work for
$d$-dimensional tokens, or $O(M+N)$ when $d$ is fixed.
\end{theorem}
\vspace{-0.2cm}
\begin{proof}[Proof sketch:]Building the scalar prefix/suffix tables
and locating all boundary cells costs $O(M+N)$. Evaluating
the barriers costs $O(N)$; the feature-weighted scans and readout cost
$O((M+N)d)$.
By Equation~\eqref{eq:jvp}, the boundary Jacobian-vector product is a
Laplace-CDF sum weighted by $\dot r_i$. Its adjoint uses
prefix/suffix scans over the ordered boundaries. Both
products therefore have linear cost.
For aggregation, differentiating $r_i x_i$ and the evaluations
$\Phi_x(b_\ell)$ adds only further scans and pointwise
operations. Their adjoints have the same linear cost;
see Appendix~\ref{app:algorithm-sensitivity}--\ref{app:algorithm-aggregate}
for the full proof and the orientation of the reverse CDF scan. Thus all claimed computations take
$O((M+N)d)$ time without forming dense membership matrices
or Jacobians.
\end{proof}

The reverse barrier scan uses the CDF orientation in~\eqref{eq:jvp},
and the aggregation readout is differenced only once. Appendix~\ref{app:algorithm-sensitivity}--\ref{app:algorithm-aggregate}
gives the exact formulas. The formal statement covers the $(r,x)$ map at
fixed ordered positions; the implementations additionally expose derivatives
with respect to the positions $p$ through the same scan structure (the
runtime benchmark's backward pass includes them, so the reported timings
already reflect this cost). The kernel width $\alpha$ is fixed in the LRA
and POPE experiments and learned in histopathology, Cars, and arXiv
(Table~\ref{tab:recipe}), where its gradient flows through the same native
implementation; the formal statement above covers the $(r,x)$ map at fixed
positions and width.

At the lengths we operate at, the exponentials in Equation \ref{eq:split} must be accumulated in log-scale to avoid overflow; Appendix~\ref{app:algorithm} provides the detailed pseudocodes and numerical implementation
discussion. A fused CUDA kernel or, more simply, native PyTorch via \texttt{logcumsumexp} both implement this scan, at a modest constant-factor tradeoff; Appendix~\ref{app:runtime-implementations} compares them.

\section{Experiments and Results}
\label{sec:experiments}

We evaluate \our{} in four settings, chosen to test the operator both where its ordering bias should help and where it should not. \emph{(i) Computational scaling}: a pure runtime study on sequences of up to $2^{24}$ ($>10^7$) tokens, against learned poolers under identical timing conditions. \emph{(ii) Gigapixel histopathology}: predicting microsatellite instability (MSI-H vs.\ MSS), a clinically actionable biomarker, on two TCGA whole-slide cohorts with a frozen pathology encoder, against specialized MIL baselines. \emph{(iii) A frozen vision--language model}: we replace only the token-selection step of the released Qwen2.5-VL-7B VisionSelector pipeline, without any retraining, and measure object hallucination on POPE. \emph{(iv) End-to-end training from scratch} on the five Long Range Arena tasks, against budget-matched reduction baselines. All methods are compared at matched output budgets; per-experiment configurations, training protocols, and dataset details are collected in Table~\ref{tab:recipe} and Appendix~\ref{app:datasets}.

\subsection{Computational Scaling}
\label{sec:exp-scaling}

\begin{wrapfigure}[23]{r}{0.56\linewidth}
\vspace{-1.5cm}
\centering
\begin{tikzpicture}

\begin{axis}[
  width=0.98\linewidth, height=0.76\linewidth,
  xmode=log, log basis x=2,
  ymode=log, log basis y=2,
  xlabel={\small input tokens $M$}, ylabel={\small forward + backward [ms]},
  xmin=2048, xmax=16777216,
  ymin=1, ymax=32768,
  xtick={2048,16384,131072,1048576,8388608,16777216},
  xticklabels={$2^{11}$,$2^{14}$,$2^{17}$,$2^{20}$,$2^{23}$,$2^{24}$},
  ytick={1,2,4,8,16,32,64,128,256,512,1024,2048,4096,8192,16384,32768,65536,131072,262144,524288,1048576,2097152,4194304,8388608,16777216},
  yticklabels={$2^{0}$,$2^{1}$,$2^{2}$,$2^{3}$,$2^{4}$,$2^{5}$,$2^{6}$,$2^{7}$,$2^{8}$,$2^{9}$,$2^{10}$,$2^{11}$,$2^{12}$,$2^{13}$,$2^{14}$,$2^{15}$,$2^{16}$,$2^{17}$,$2^{18}$,$2^{19}$,$2^{20}$,$2^{21}$,$2^{22}$,$2^{23}$,$2^{24}$},
  grid=both, grid style={gray!25},
  tick label style={font=\tiny},
  legend columns=3,
  legend style={
    at={(0.5,1.02)}, anchor=south,
    draw=none, fill=none, font=\tiny,
    column sep=0.6em,
    /tikz/every even column/.append style={column sep=0.6em},
  },
  legend cell align=left,
  legend image post style={line width=1.4pt},
  legend image code/.code={%
    \draw[mark repeat=2,mark phase=2,#1]
      plot coordinates {(0cm,0cm) (0.22cm,0cm)};%
  },
]
\addplot[color=lapmergeblack, dashed, thick, mark=none, forget plot] table {time_comparison/results/pgfplots_data/fixedn_cuda-fused.dat};
\addplot[color=lapmergeblack, solid, thick, mark=none] table {time_comparison/results/pgfplots_data/ratio_cuda-fused.dat};
\addlegendentry{LapMerge}
\addplot[color=tokenred, dashed, thick, mark=none, forget plot] table {time_comparison/results/pgfplots_data/fixedn_token-sdpa.dat};
\addplot[color=tokenred, solid, thick, mark=none] table {time_comparison/results/pgfplots_data/ratio_token-sdpa.dat};
\addlegendentry{TokenLearner (SDPA)}
\addplot[color=patchblue, dashed, thick, mark=none, forget plot] table {time_comparison/results/pgfplots_data/fixedn_patch-sdpa.dat};
\addplot[color=patchblue, solid, thick, mark=none] table {time_comparison/results/pgfplots_data/ratio_patch-sdpa.dat};
\addlegendentry{PatchMerger (SDPA)}
\addplot[only marks, mark=*, mark size=2.2pt, color=lapmergeblack, forget plot] table {time_comparison/results/pgfplots_data/marker_ratio_cuda-fused.dat};
\addplot[only marks, mark=square*, mark size=2.0pt, color=tokenred, forget plot] table {time_comparison/results/pgfplots_data/marker_fixedn_token-sdpa.dat};
\addplot[only marks, mark=o, mark size=2.2pt, color=tokenred, forget plot] table {time_comparison/results/pgfplots_data/marker_ratio_token-sdpa.dat};
\addplot[only marks, mark=o, mark size=2.2pt, color=patchblue, forget plot] table {time_comparison/results/pgfplots_data/marker_fixedn_patch-sdpa.dat};
\addplot[only marks, mark=o, mark size=2.2pt, color=patchblue, forget plot] table {time_comparison/results/pgfplots_data/marker_ratio_patch-sdpa.dat};
\end{axis}
\end{tikzpicture}
\vspace{
-0.3cm
}

\caption{Reducer forward--backward latency on GH200 ($B=1$, $d=512$, FP32; five warm-ups, 30 repeats). Dashed: $N=1024$; solid: $N=M/16$. LapMerge scales similarly in both regimes; the SDPA poolers retain $MN$ interactions. Open circles precede a 600-s timeout, the filled OOM, the filled square an execution failure. Timing covers the reducer core (LapMerge gets precomputed masses; the poolers project inside the timed region); scorer/routing overheads: Appendix~\ref{app:runtime-competitors}.
}
\label{fig:runtime-main-scaling}
\vspace{-0.7cm}
\end{wrapfigure}
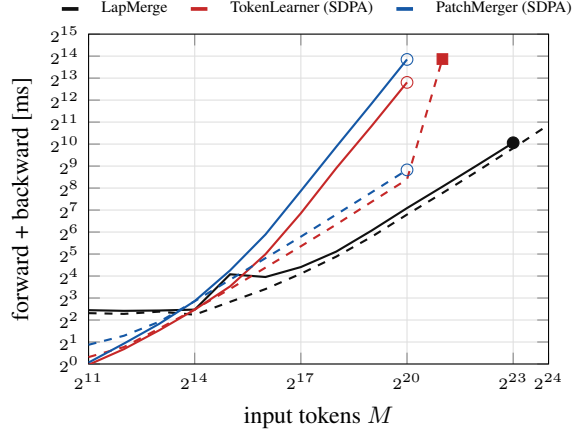

Figure~\ref{fig:runtime-main-scaling} isolates the reducers, excluding upstream networks. We vary $M$ from $2^{11}$ to $2^{24}$ at both a fixed output budget and a fixed 16-fold compression, with identical timing settings for all methods. At $M=2^{20}$ and $N=M/16$, LapMerge takes $0.136$ s versus $7.16$ s for TokenLearner and $14.75$ s for PatchMerger; the pure-PyTorch native path, which is the numerically faithful implementation used in all accuracy experiments, takes $0.290$ s in the same setting -- still $25$--$51\times$ faster than the poolers, so the advantage does not rest on the fused backend. The ordered-coordinate LapMerge run reaches $M=2^{24}$ at $N=1024$ (1.88 s) and $M=2^{23}$ at $N=M/16$ (1.07 s); the next ratio point is out of memory. Similar LapMerge slopes as $N$ grows are consistent with scan-based $O((M+N)d)$ work, without an $M\times N$ assignment. Appendix~\ref{app:runtime-backend-ablation} separates dense, native, compiled, and fused implementations; Appendix~\ref{app:runtime-competitors} gives the matched competitor measurements and failure diagnostics.

\subsection{\our{} on Gigapixel Histopathology}
\label{sec:histopathology}
We first evaluate \our{} on whole-slide histopathology, a natural test bed for long-sequence modeling. A single whole-slide image (WSI) can contain thousands to tens of thousands of tissue patches, while only an unknown subset may be relevant to the diagnostic task. This creates a setting where reducing the sequence before applying expensive sequence modeling is particularly valuable.
Figure~\ref{fig:pipeline} illustrates the complete pipeline on this task. We encode tissue patches with a frozen UNI2-h encoder \citep{chen2024uni,uni2h2025}, order them along a generalized Hilbert curve to preserve spatial locality, and apply \our{} to reduce the resulting sequence from $M$ patches to a fixed budget of $N=1024$ tokens. The compressed sequence is then processed by a standard three-layer Transformer classifier. \our{} is used without task-specific modification or budget-specific tuning. Because the same scores determine the mass distribution used for merging, the reduction mechanism simultaneously provides a patch-level routing-mass map, which passes a deletion-style attribution test (Appendix~\ref{app:data-histo}).

We compare \our{} with established and recent multiple-instance learning (MIL) methods, including ABMIL \citep{ilse2018abmil}, CLAM-SB/MB \citep{lu2021clam}, DS-MIL \citep{li2021dsmil}, DTFD-MIL \citep{zhang2022dtfdmil}, TransMIL \citep{shao2021transmil}, MambaMIL \citep{yang2024mambamil}, and RRT-MIL \citep{tang2024rrtmil}. All methods use the same frozen UNI2-h features and patient-level splits. Results are reported as mean $\pm$ std AUROC over six random seeds and are summarized in Table~\ref{tab:msih-main}.
\begin{wraptable}{r}{0.46\linewidth}
\vspace{-0.6\baselineskip}
\begin{minipage}{\linewidth}
\centering
\caption{Test AUROC (mean $\pm$ std, 6 seeds). External anchors are
literature numbers, not reproduced here: Kather et al.\ 2019 is a CNN
tile-classifier, not MIL; Wagner et al.\ 2023 is AttentionMIL over
CTransPath features.}
\label{tab:msih-main}
\smallskip
\small
\setlength{\tabcolsep}{3pt}%
\begin{tabular}{lcc}
\toprule
method & CRC & STAD \\
\midrule
ABMIL & $88.97{\pm}2.91$ & $92.68{\pm}3.24$ \\
CLAM-SB & $89.92{\pm}4.40$ & $93.75{\pm}2.90$ \\
CLAM-MB & $90.66{\pm}1.43$ & $91.62{\pm}2.74$ \\
DS-MIL & $90.33{\pm}7.60$ & $87.11{\pm}3.35$ \\
DTFD-MIL & $90.86{\pm}1.58$ & $93.22{\pm}4.46$ \\
TransMIL & $81.07{\pm}3.90$ & $89.24{\pm}4.52$ \\
MambaMIL & $87.98{\pm}5.62$ & $93.02{\pm}5.63$ \\
RRT-MIL & $\mathbf{93.54{\pm}3.03}$ & $\mathbf{97.19{\pm}1.61}$ \\
\rowcolor{orange!10} \our{} (ours) & $91.77{\pm}3.14$ & $96.03{\pm}1.99$ \\
\midrule
\emph{Kather '19} & \emph{77.00} & \emph{81.00} \\
\emph{Wagner '23} & \emph{78.00} & --- \\
\bottomrule
\end{tabular}
\end{minipage}
\vspace{-0.9cm}
\end{wraptable}

\our{} achieves $91.77\pm3.14$ AUROC on CRC and $96.03\pm1.99$ on STAD, second-best on both cohorts. Only RRT-MIL is stronger ($93.54\pm3.03$ and $97.19\pm1.61$), while every remaining baseline achieves lower mean AUROC on at least one of the two cohorts. Notably, \our{} is the only method whose sequence model operates on at most $1{,}024$ tokens rather than the full bag (median $5{,}922$/$7{,}329$ patches on CRC/STAD; only $3.6\%$/$8.3\%$ of slides have $M\le1024$, where $N{=}1024$ is not a compression): it stays competitive at a fraction of the sequence length.

The computational advantage is complementary to this predictive performance. \our{} reduces $M$ input tokens to $N$ output tokens in $\mathcal{O}(M+N)$ time, allowing the expensive sequence model to operate only on the compressed representation. In contrast to methods that retain or repeatedly process the full patch sequence, our reduction cost grows linearly with the number of input patches. Thus, this experiment serves as a proof of concept that \our{} can act as a general-purpose front-end for long-sequence modeling.

\subsection{Drop-in Replacement in a Frozen Vision-Language Model}
\label{sec:vlm-pope}

\begin{wraptable}{r}{0.5\linewidth}
\vspace{-0.8cm}
\centering\small
\caption{\textbf{Drop-in token reduction on POPE} (F1, frozen
Qwen2.5-VL-7B, no retraining) at matched $10\%$ and $30\%$ visual-token
budgets. Baseline rows quote the published evaluation of
\citet{visionselector2025} (their Table~1); our rerun of the released
selection reproduces its F1 to $0.06$ (Appendix~\ref{app:pope-sweep}).
\our{} reuses the VisionSelector scorer at default widths
$(\alpha_r,\alpha_a){=}(1,1)$; the POPE-tuned variant is a reference
ceiling (see \emph{Ablation}).}
\label{tab:vlm-field}
\setlength{\tabcolsep}{3.5pt}%
\begin{tabular}{lcc}
\toprule
Method & $10\%$ & $30\%$ \\
\midrule
FastV \citep{chen2024fastv}              & $65.99$ & $80.47$ \\
DART \citep{wen2025stop}                 & $71.91$ & $83.43$ \\
HoloV \citep{zou2026don}                 & $75.84$ & $83.27$ \\
PruMerge+ \citep{shang2025llava}         & $76.52$ & $83.69$ \\
VisionZip \citep{yang2025visionzip}      & $78.46$ & $84.73$ \\
DivPrune \citep{alvar2025divprune}       & $79.27$ & $84.68$ \\
VisionSelector \citep{visionselector2025}& $79.73$ & $85.05$ \\
\midrule
\rowcolor{orange!10}
\textbf{score-routed \our{} (ours)}  & $\mathbf{81.31}$ & $\mathbf{85.50}$ \\
\rowcolor{orange!10}
\emph{score-routed \our{} (best $\alpha_a$)} & \emph{82.01} & \emph{85.87} \\
\bottomrule
\end{tabular}
\vspace{-1.2\baselineskip}
\end{wraptable}

We next evaluate whether score-routed \our{} can improve visual token reduction in a frozen vision-language model (VLM). We take the publicly released VisionSelector pipeline built on Qwen2.5-VL-7B \citep{qwen25vl} as-is, including its vision encoder, its learned scorer, and the language model, and replace only the final selection step. VisionSelector keeps the $k$ highest-scoring tokens and drops the remainder. Score-routed \our{} keeps the same $k$ positions but fills them with merged values: the standardized scores enter a soft top-$k$ at temperature $\alpha_r$, the resulting budget-exact masses define $k$ equal-mass barrier intervals along the raster order, and each output is the mass-weighted average of all tokens under its interval's Laplace window of read-out width $\alpha_a$ (\S\ref{sec:method}). Every output token stays at the position of the corresponding kept token, so the rotary position encoding of the language model is unchanged and no other part of the pipeline has to be adapted.

We evaluate on the full POPE benchmark \citep{li2023pope}, which probes
object hallucination through binary existence questions, at $10\%$ and
$30\%$ visual-token budgets. For score-routed \our{} we use the default widths $(\alpha_r,\alpha_a)=(1,1)$, fixed without reference to POPE. As shown in Table~\ref{tab:vlm-field}, score-routed \our{} reaches $81.31$ and $85.50$
F1, improving on VisionSelector's published F1 by $+1.58$ and $+0.45$ and
outperforming all pruning- and merging-based baselines at both budgets.
The gain is largest at the tighter budget, where hard selection discards
$90\%$ of the visual evidence: a paired bootstrap over the $9{,}000$
questions gives $95\%$ CIs of $[+0.9,+2.4]$ at $10\%$ and $[-0.1,+0.9]$
at $30\%$ (Appendix~\ref{app:pope-sweep}). Since the scorer, the backbone and the positional scheme
are all inherited unchanged, the improvement can be attributed entirely
to replacing hard selection with soft merging -- a strict, training-free
upgrade of an already released pipeline, obtained purely at inference time.

\medskip\noindent\textbf{Ablation.}
The free parameters score-routed \our{} introduces into this pipeline are the
two widths $(\alpha_r,\alpha_a)$. We sweep the read-out width $\alpha_a$ at
fixed $\alpha_r{=}1$, which interpolates between a nearly hard read-out and progressively broader merging. To measure how much the default leaves on the table, we grid-search $\alpha_a$ per budget directly on POPE (no gradient-based tuning or parameter updates). The best widths gain $+2.28$ and $+0.82$. Because these widths are chosen on the evaluation set itself, we report them only as an ablation study and not as a competing result. The benefit of merging over selection is not an artifact of a single lucky width: it persists across the moderate-width range of the sweep in Appendix~\ref{app:pope-sweep}, although an over-wide kernel eventually falls below hard selection. Moreover, a single dataset-agnostic default already captures the bulk of it, so \our{} can be deployed in a frozen VLM without per-task calibration.

\subsection{End-to-End Learning: Long Range Arena}
\label{sec:exp-lra}

\begin{table}[!ht]
\centering\small
\caption{%
  \textbf{Long Range Arena, end-to-end training, budget 10\%.}
  Test accuracy (\%), mean $\pm$ std over 3 seeds.
  \emph{Unreduced} and \emph{Performer} are full-sequence references.
  Best result per task in \textbf{bold}; Pathfinder is at chance level
  for all methods, so no entry is bolded there.
}
\label{tab:lra_10pct}
\setlength{\tabcolsep}{5pt}
\begin{tabular}{lccccc}
\toprule
Method
  & ListOps & Text & Retrieval & Image & Pathfinder\\
\midrule
Unreduced
  & $38.10 \pm 1.21$ & $65.03 \pm 0.66$ & 83.02$\pm$0.20 & $35.94 \pm 0.58$ & $49.9 \pm 0.1$  \\
Performer
  & $38.57\pm 0.37$& $65.93 \pm 0.29$ & $82.2\pm0.15$ & $43.85 \pm 0.13$ & $50.21 \pm 0.13$  \\
\midrule
Uniform pooling
  & $29.65 \pm 0.25$ & $75.05 \pm 0.35$ & $88.12\pm0.33$ & $42.29 \pm 0.65$ & $49.67\pm0.21$  \\
ToMe
  & $35.47\pm2.36$ & $60.97 \pm 0.15$ & $78.85\pm0.34$ & $36.16 \pm 0.19$ & $49.84\pm0.06$\\
Top-$N$ selection
  & $24.92\pm9.24$ & $55.68 \pm 2.55$ & $74.93\pm9.13$ & $36.89 \pm 0.18$ & $50.78\pm1.42$  \\
TokenLearner
  & $38.78\pm0.71$ & $65.47 \pm 0.26$ & $83.51\pm0.34$ & $37.34 \pm 0.28$ &$49.98\pm0.19$  \\
PatchMerger
  & $38.67\pm0.81 $ & $66.26 \pm 0.17$ & $82.22\pm0.2$ & $25.52 \pm 0.18$ & $49.79\pm0.03$   \\
\rowcolor{orange!10}
\our{} (ours)
  &$\mathbf{41.42\pm0.47}$ & $\mathbf{77.05 \pm 0.47}$ & $\mathbf{88.88\pm0.05}$ & $\mathbf{47.03 \pm 0.39}$ & $49.83\pm0.14$ \\
\bottomrule
\end{tabular}
\end{table}

We next evaluate \our{} when the reducer and backbone are trained jointly from scratch, rather than inserted into a pretrained encoder, on the Long Range Arena (LRA) benchmark. Results are presented in Table~\ref{tab:lra_10pct}. \our{} achieves the best accuracy on four of the five tasks -- ListOps, Text, Retrieval and Image -- outperforming the strongest budget-matched baseline by $+2.6$, $+2.0$, $+0.8$ and $+4.7$ points, respectively. Notably, it also surpasses both full-sequence models on all four (e.g., $+12.0$ on Text, $+5.9$ on Retrieval and $+11.1$ on Image over the unreduced Transformer), suggesting that structure-aware merging provides a useful inductive bias rather than acting merely as compression. Learned-selection baselines are considerably less stable: Top-$N$ shows large variance across seeds on ListOps and Retrieval, and PatchMerger collapses on Image. On Pathfinder, all models, including the unreduced ones, remain at chance level ($\approx 50\%$), so differences on this task are not meaningful. A direct order intervention (permuting only the scan coordinate, everything else fixed) confirms that order matters for methods with a local, interval-based read-out (Appendix~\ref{sub:order-intervention}).

\section{Conclusion}
\label{sec:conclusion}

We presented \our{}, a differentiable operator that merges an ordered sequence of $M$ tokens into exactly $N$ outputs in $\mathcal{O}(M+N)$ time. The core idea -- replacing unrestricted token-to-output assignment with a mass-based partition of the input order -- delivers what previous reducers trade off against one another: an exact budget, end-to-end differentiability, and linear cost. The budget can be changed without changing operator parameters, and the learned mass doubles as an interpretable routing map. Empirically, \our{} compresses frozen encoders at almost no accuracy cost, upgrades a released VLM pipeline without retraining -- beating every POPE baseline, decisively at the $10\%$ budget -- achieves the best budget-matched accuracy on four of five LRA tasks, and runs at $10^7$-token scale where existing reducers fail. We believe ordered, mass-based aggregation can become a standard primitive for long-sequence models, offering linear complexity while retaining information from the entire input. However, \our{} is not a universal reducer: it prefers local aggregation, with contributions decaying with distance along the given ordering -- ideal when neighboring tokens are spatially or temporally related, but limiting when related information sits at distant positions, as in some text tasks, where order-agnostic reduction may be preferable.

\newpage
\subsection*{AI use statement}

We used AI assistants for editing and polishing the text, for implementing
and debugging code, and as an aid in designing and analyzing experiments.
The mathematical construction and its proofs were developed by the authors.
All experimental results and final wording were verified by the authors,
who take full responsibility for the content and correctness of the paper.

\subsection*{Reproducibility statement}

Appendix~\ref{app:algorithm} contains the complete construction of the
operator, full proofs of the propositions and the main theorem, the scan
algorithms, and the explicit backward formulas, so the method can be
reimplemented from the paper alone. Appendix~\ref{app:runtime-comparison}
documents the runtime benchmark protocol (hardware, shapes, warm-up and
repeat counts, timeout and failure handling) and the backend ablation.
The experimental appendices describe the frozen-encoder, VLM, and
end-to-end setups, and the datasets appendix lists all data sources; the
histopathology cohorts use published patient-level splits with labels from
the cited sources. The supplementary material includes the code for the
histopathology experiments, with a synthetic-data demo, environment notes,
and instructions for obtaining the gated real data, as well as the scripts
and raw measurements behind the runtime benchmarks. Each reported table
maps to the runner scripts and configurations in the corresponding
supplementary package (histopathology; LRA; POPE/Cars/arXiv; runtime), which
record the exact command, seeds, and hyperparameters per run; the runtime
package additionally contains the raw CSVs from which every plot is
generated, and the POPE evaluation is greedy and therefore deterministic
given the pinned VisionSelector commit it states. All accuracy experiments
use the native PyTorch implementation of the operator; the CUDA-fused
backend appears only in the runtime study.

\bibliography{iclr2027_conference}
\bibliographystyle{iclr2027_conference}

\newpage

\appendix

\renewcommand \thepart{}
\renewcommand \partname{}

\part{Appendix} 
\setcounter{secnumdepth}{2}
\setcounter{tocdepth}{4}
\parttoc 

\clearpage
\section{The \our{} Algorithm}
\label{app:algorithm}

This appendix gives the complete mathematical and scan-level construction of
\our{}. The input is an ordered sequence of $M$ tokens, non-negative routing
weights $r_i$ summing to one, an output budget $N$, and a Laplace scale
$\alpha>0$. The output consists of exactly $N$ convex combinations of the
input tokens. We prove existence and differentiability of the equal-mass
barriers, derive their closed-form Laplace solve and matrix-free derivatives,
and then prove the corresponding aggregation and backward formulas.
Implementation and measurement variants are compared in
Appendix~\ref{app:runtime-implementations}. Throughout, $\Phi$ and $b_\ell$
have the meanings in \eqref{eq:mix}--\eqref{eq:bar}. A reader who mainly
wants to implement the operator can go directly to
Algorithms~\ref{alg:barriers}--\ref{alg:aggregate}, the implementation
checklist closing this section, and the configuration recipe in
Table~\ref{tab:recipe}; the text in between derives and proves each step.

\paragraph{Setup and notation.} Tokens $x_1,\dots,x_M\in\mathbb{R}^d$ arrive sorted by position $p_1<\dots<p_M$. The operator receives non-negative routing masses $r_i$ satisfying $\sum_i r_i=1$; an optional scorer may produce these masses upstream, but is not part of the operator defined below. As discussed in Section~\ref{sec:method}, this scorer can use a sigmoid, softplus, or differentiable soft-top-$k$ normalization, depending on the model variant. Here we only require its final admissible masses $r_i$. We fix the budget $N$ and the Laplace scale $\alpha$. In practice, $\alpha\propto\operatorname{spread}(p)/M$, ensuring the kernel remains local as $M$ grows. If position ties occur, they are broken using the original token indices. The Laplace CDF is defined as $\sigma(z)=\frac{1}{2}e^z$ for $z\le0$ and $\sigma(z)=1-\frac{1}{2}e^{-z}$ for $z\ge0$.

\paragraph{Why no $O(M\log M)$ sorting term appears.} The order of the
tokens is part of the input to this operator. For sequences, the default
position is simply the original index, $p_i=i$, and for spatial or other
structured inputs the caller supplies the already chosen traversal order.
Consequently, all prefix/suffix scans and barrier readouts operate directly
along the existing token axis; the algorithm does not sort the input at every
forward pass. Under this ordered-input convention, the work of the scan and
readout is $O((M+N)d)$ and the auxiliary memory is $O((M+N)d)$ (or
$O(M+N)$ for the scalar barrier summaries). A generic implementation that is
given unsorted coordinates may prepend a stable sort, adding
$O(M\log M)$ work and the corresponding permutation/gather cost, but that is
an input-preparation cost rather than part of the ordered LapMerge operator.
All runtime measurements in this paper use the ordered-input setting; the
original indices provide the deterministic tie-break when coordinates are
equal.

\paragraph{Cost of the optional scorer.} The linear bound above describes the
LapMerge operator after the routing masses $r_i$ have been supplied. If the
upstream scorer uses a differentiable soft-top-$k$ construction implemented by
sorting the scalar logits, the end-to-end routing stage has an additional
$O(M\log M)$ term. This does not change the $O((M+N)d)$ complexity of the
merge itself. In practice, the extra term is often small compared with the
feature-valued scan: only one scalar per token is sorted, and GPU sorting
primitives are highly optimized. For this reason the asymptotic distinction
should be reported explicitly, while measured runtimes should include the
scorer whenever the complete reduction module is being benchmarked.
If the scorer instead maps logits independently through a pointwise sigmoid or
softplus and normalizes the resulting non-negative masses without top-$N$
sorting, this additional sorting cost is absent and the routing stage is
linear in $M$.

\paragraph{From a random token to a merged token.}
Choose $I$ with $\Pr(I=i)=r_i$ and let $P_I=p_I+\alpha E_I$, where each
$E_i$ has CDF $\sigma$. Then $\Pr(P_I\le z)=\Phi(z)$ by conditioning on
$I$. For $B_\ell=[b_{\ell-1},b_\ell)$, the barrier equation gives
$\Pr(P_I\in B_\ell)=\Phi(b_\ell)-\Phi(b_{\ell-1})=1/N$.
Conditioning once more on $I$ yields
\begin{equation}
\label{eq:appendix-membership}
\Pr(I=i\mid P_I\in B_\ell)
=N r_i\!\left[\sigma\!\left(\frac{b_\ell-p_i}{\alpha}\right)
-\sigma\!\left(\frac{b_{\ell-1}-p_i}{\alpha}\right)\right]
=w_i^\ell.
\end{equation}
Every $w_i^\ell$ is non-negative and $\sum_iw_i^\ell=1$, so
$t_\ell=\mathbb E[x_I\mid P_I\in B_\ell]=\sum_iw_i^\ell x_i$ is a
convex combination. Moreover, $\sum_{\ell=1}^Nw_i^\ell=Nr_i$:
all of a token's routing mass is accounted for across the outputs.

\paragraph{Full proof of Proposition~\ref{prop:equal-mass-boundaries}.}
For a continuous, strictly positive density $\kappa$, each summand in
$\Phi'(z)=\alpha^{-1}\sum_i r_i \kappa((z-p_i)/\alpha)$ is non-negative and at
least one is strictly positive; hence $\Phi'(z)>0$ for all finite $z$.
Because $\sum_i r_i=1$, the CDF limits are $\Phi(-\infty)=0$ and
$\Phi(+\infty)=1$. The intermediate value theorem and strict monotonicity
give one and only one solution of $\Phi(b_\ell)=\ell/N$ for each
$1\le\ell<N$, and $b_1<\cdots<b_{N-1}$. For differentiability, consider
a differentiable path $r(t)$ in the probability simplex. Then
$\sum_j\dot r_j=0$ and implicit differentiation gives
\begin{equation}
\label{eq:appendix-barrier-jvp}
0=\Phi'(b_\ell)\dot b_\ell
  +\sum_j\sigma\!\left(\frac{b_\ell-p_j}{\alpha}\right)\dot r_j,
\qquad
\dot b_\ell=-\frac{\sum_j\sigma((b_\ell-p_j)/\alpha)\dot r_j}
{\Phi'(b_\ell)}.
\end{equation}
The denominator is positive, so the implicit function theorem applies.
Equation~\eqref{eq:jvp} is one ambient representative of this derivative:
on the simplex, adding the same constant to all its entries does not
change any directional derivative because $\sum_j\dot r_j=0$. This
observation also explains why equivalent VJPs may differ by a constant
vector before the upstream normalization is differentiated.

\paragraph{Three numerical and algorithmic pitfalls.} First, directly computing $B_i=\sum_{j\le i}r_je^{p_j/\alpha}$ and $C_i=\sum_{j>i}r_je^{-p_j/\alpha}$ in~\eqref{eq:split} can overflow or underflow before the final factors $e^{\mp x/\alpha}$ bring the CDF back into $[0,1]$; recentering each exponential at its query point avoids those large intermediate values. Second, comparing $\ell/N$ with the unsmoothed prefix mass $A_i=\sum_{j\le i}r_j$ can select the wrong barrier cell, because $A_i\ne\Phi(p_i)$; we search the smoothed CDF nodes instead (Safeguard~A). Third, tangents, cotangents, and feature coordinates can be signed, whereas a logarithmic scan requires non-negative summands; separate positive and negative scans handle them. The in-cell solve is further stabilized by a suitable endpoint anchor and a Vieta root (Safeguard~B).

\paragraph{Principle 1: Recenter decayed sums at the query point.}
For a query point $x$ located in cell~$i$ (i.e., exactly $i$ tokens
satisfy $p_j\le x$) and a non-negative sequence $w_1,\dots,w_M$
(with $w=r$ by default), define the log decayed sums
\begin{equation}
\label{eq:recentered}
L_w(x;i)=\log\!\!\sum_{j\le i}w_je^{(p_j-x)/\alpha},
\qquad
R_w(x;i)=\log\!\!\sum_{j>i}w_je^{(x-p_j)/\alpha}.
\end{equation}
Both quantities are computed via global \textsc{LogCumsumExp} scans. For the left term, we scan $\log w_j+p_j/\alpha$ from left to right and subtract $x/\alpha$ upon reading out cell~$i$. For the right term, we execute a reverse scan on $\log w_j-p_j/\alpha$ and add $x/\alpha$. Here $\log0=-\infty$ is an empty summand. Every exponent is recentered to be non-positive at readout: $(p_j-x)/\alpha\le0$ on the left and $(x-p_j)/\alpha\le0$ on the right. The scan stores logarithms, so large global offsets are never exponentiated directly.

\paragraph{Principle 2: Decompose signed sequences prior to scanning.}
The \textsc{LogCumsumExp} operation assumes non-negative summands, a
condition satisfied by $r_i$ but violated by tangents, cotangents, and
token coordinates. For a signed sequence $v_i$, write
$v_i=v_i^+-v_i^-$, with $v_i^+=\max(v_i,0)$ and
$v_i^-=\max(-v_i,0)$. Run a standard \textsc{LogCumsumExp} on each
component independently and combine the results at readout:
\begin{equation}
\label{eq:signed-scan}
\textstyle\sum_{j\le i}v_j
=
\exp\!\big(\textsc{LogCumsumExp}_{j\le i}(\log v_j^+)\big)
-
\exp\!\big(\textsc{LogCumsumExp}_{j\le i}(\log v_j^-)\big).
\end{equation}
Recombination occurs only after the prefix and suffix states are evaluated.
This doubles the scan operations for signed inputs but keeps the logarithms
well defined. We denote this wrapper \textsc{SignedLogCumsumExp}.
For any signed $w$, let
$D^L_w=e^{L_{w^+}}-e^{L_{w^-}}$ and
$D^R_w=e^{R_{w^+}}-e^{R_{w^-}}$, with empty sums equal to zero.
For non-negative $w$, these reduce to $e^{L_w}$ and $e^{R_w}$.

\paragraph{One construction, five uses.} With $A^w_i = \sum_{j\le i}w_j$, we define the recentered scan functional as:
\begin{equation}
\label{eq:phi-w}
\Phi_w(x;i) := A^w_i - \tfrac12D^L_w(x;i) + \tfrac12D^R_w(x;i).
\end{equation}
Setting $w=r$ recovers the mixture CDF $\Phi$ within cell~$i$. The
same prefix/suffix states serve five tasks: computing node values, solving
for barriers, evaluating barrier JVPs, computing barrier VJPs, and
aggregating values. The reverse VJP uses the complementary orientation
$\sigma((b_\ell-p_j)/\alpha)=1-\sigma((p_j-b_\ell)/\alpha)$ after
interchanging barrier and token roles. Table~\ref{tab:phi-w-uses}
specifies the inputs to these scans.

\begin{table}[h]
\centering
\small
\begin{tabular}{lllll}
\toprule
Context & Weight $w$ & Summation domain & Query point & Cell $i$ source \\
\midrule
\S\ref{app:algorithm-barriers}, $\mathrm{node}_i$ & $r$ & tokens $j$ & $p_i$ & index $i$ \\
\S\ref{app:algorithm-barriers}, solve/density & $r$ & tokens $j$ & $b_\ell$ & Safeguard~A search \\
\S\ref{app:algorithm-sensitivity}, \textsc{Forward} & $\dot r$ & tokens $j$ & $b_\ell$ & reused from solve \\
\S\ref{app:algorithm-sensitivity}, \textsc{Reverse} & $c$ (below) & barriers $\ell$ & $p_j$ & complement of swapped scan \\
\S\ref{app:algorithm-aggregate} & $r\odot x$ & tokens $j$ & $b_\ell$ & reused from solve \\
\bottomrule
\end{tabular}
\caption{Five applications of the recentered prefix/suffix scans. The reverse row takes the complement of $\Phi_c(p_j)$ to obtain the required CDF orientation. Cell indices are determined by a monotone merge search.}
\label{tab:phi-w-uses}
\end{table}

\subsection{Locating the barriers}
\label{app:algorithm-barriers}

\begin{algorithm}[h!]
\caption{\textsc{FindBarriers}, $\mathcal{O}(M+N)$}
\label{alg:barriers}
\begin{algorithmic}[1]
\Require sorted positions $p_{1:M}$, weights $r_{1:M}$, budget $N$, scale $\alpha$
\Ensure barriers $b_{0:N}$ (with $b_0=-\infty$, $b_N=+\infty$); cell indices $i_{0:N}$ (with $i_0=0$, $i_N=M$); densities $\Phi'(b_{1:N-1})$
\State $\mathrm{node}_i \gets \mathrm{logit}\,\Phi_r(p_i;i)$ for all $i$
\For{$\ell = 1,\dots,N-1$}
  \State $i_\ell \gets$ merge-search $\mathrm{logit}(\ell/N)$ against $\mathrm{node}_{1:M}$
  \State $q_\ell\gets$ finite endpoint of cell $i_\ell$ with $\Phi_r(q_\ell)$ closest to $\ell/N$
  \State $b_\ell \gets$ cancellation-safe (Vieta) root of $\Phi_r(x;i_\ell)=\ell/N$, recentered at $q_\ell$
  \State $\Phi'(b_\ell) \gets \tfrac1{2\alpha}\big[e^{L_r(b_\ell;i_\ell)}+e^{R_r(b_\ell;i_\ell)}\big]$
\EndFor
\State $b_0\gets-\infty$, $i_0\gets0$;\qquad $b_N\gets+\infty$, $i_N\gets M$
\State \Return $b_{0:N}$, $i_{0:N}$, $\Phi'(b_{1:N-1})$
\end{algorithmic}
\end{algorithm}

Algorithm~\ref{alg:barriers} outlines the procedure for locating the barriers.
Line~1 evaluates the smoothed CDF at every token position and stores
its log-odds, $\mathrm{node}_i=\mathrm{logit}\,\Phi_r(p_i;i)$.
The following non-negative increments make these log-odds accurate even
when a node CDF is close to zero or one. Set
$U_i=\sum_{j\le i}r_j e^{(p_j-p_i)/\alpha}$,
$V_i=\sum_{j\ge i}r_j e^{(p_i-p_j)/\alpha}$, and
$\delta_i=(p_{i+1}-p_i)/\alpha$. Then, with
$d_0=\tfrac12V_1$, $d_M=\tfrac12U_M$, and
\begin{equation}
\label{eq:node-increments}
d_i=\tfrac12\bigl(1-e^{-\delta_i}\bigr)
      (U_i+V_{i+1}),\qquad 1\le i<M,
\end{equation}
we have $d_i>0$, $\sum_{i=0}^{M}d_i=1$, and
$\Phi(p_i)=\sum_{h=0}^{i-1}d_h$,
$1-\Phi(p_i)=\sum_{h=i}^{M}d_h$. The recurrences
$U_i=r_i+e^{-\delta_{i-1}}U_{i-1}$ and
$V_i=r_i+e^{-\delta_i}V_{i+1}$ compute all increments in $O(M)$.
In floating point, run these recurrences in the log domain and compute
$\log(1-e^{-\delta_i})$ via $\log[-\operatorname{expm1}(-\delta_i)]$.
Obtain each node logit by subtracting a suffix
\textsc{LogCumsumExp} of $\log d_h$ from a prefix
\textsc{LogCumsumExp}. This avoids both underflow in small increments
and subtracting a CDF value from one near the right tail. The nodes are strictly increasing, so all barrier
cells can then be found in a single merge pass. Two additional details
make the in-cell solve safe.

\paragraph{Derivation of the Laplace cell formula.}
In cell $p_i\le x<p_{i+1}$, the CDF branch for $j\le i$ is
$1-\tfrac12e^{(p_j-x)/\alpha}$ and the branch for $j>i$ is
$\tfrac12e^{(x-p_j)/\alpha}$. Thus, with
$A_i=\sum_{j\le i}r_j$,
$B_i=\sum_{j\le i}r_je^{p_j/\alpha}$, and
$C_i=\sum_{j>i}r_je^{-p_j/\alpha}$,
\begin{equation}
\label{eq:appendix-cell-split}
\Phi(x)=A_i-\tfrac12 B_i e^{-x/\alpha}
              +\tfrac12 C_i e^{x/\alpha}.
\end{equation}
The three arrays are an ordinary prefix, a weighted prefix, and a
weighted suffix. Their raw form proves the identity; numerical evaluation
uses the recentered forms in~\eqref{eq:recentered}. Evaluating $\Phi$ at
all $M$ token positions and merging those ordered values with the ordered
targets $\ell/N$ identifies every barrier cell in $O(M+N)$ work. Inside
a cell, $u=e^{(x-q)/\alpha}>0$ turns
\eqref{eq:appendix-cell-split} into
\eqref{eq:barrier-quadratic}. Since $\Phi'(x)>0$, the positive root in
the selected cell is the unique barrier. The root and its density require
constant work per barrier, proving the forward bound for barrier location.

\paragraph{Safeguard A: Search against $\Phi$ itself, not the plain mass $A$.}
Because $\mathrm{node}_{1:M}$ is monotone, locating the cell
$[p_i,p_{i+1})$ containing each $b_\ell$ reduces to merging two sorted lists:
the nodes and the targets $\mathrm{logit}(\ell/N)$. A single pair of
pointers advances through the $M$ nodes and $N-1$ targets, costing
$\mathcal{O}(M+N)$ time rather than requiring $N-1$ independent searches. The unsmoothed
prefix mass $A_i$ must not be used here: it ignores the kernel tails and can
choose the wrong cell.

The search uses two sentinel cells. Cell $i=0$ denotes $(-\infty,p_1)$,
cells $1,\ldots,M-1$ denote $[p_i,p_{i+1})$, and cell $i=M$ denotes
$[p_M,+\infty)$. The merge search is therefore performed on the extended
node list $(-\infty,\mathrm{node}_1,\ldots,\mathrm{node}_M,+\infty)$.
At an exact node value, we use the right-continuous convention, so a barrier
equal to $p_i$ is assigned to cell $i$.

\paragraph{Safeguard B: Solve inside the cell anchored at whichever end stays bounded.}
Inside a known cell, set $u=\exp((x-q)/\alpha)$, where $q$ is one of the
finite cell endpoints. Define the recentered masses:
\begin{equation}
L_q=\sum_{j\le i}r_j e^{(p_j-q)/\alpha},\qquad
R_q=\sum_{j>i}r_j e^{(q-p_j)/\alpha},\qquad
A_i=\sum_{j\le i}r_j.
\end{equation}
Holding these quantities fixed turns $\Phi_r(x;i)=\ell/N$ into the explicit
quadratic equation:
\begin{equation}
\label{eq:barrier-quadratic}
R_q u^2 + 2(A_i-\ell/N)u - L_q = 0.
\end{equation}
The valid solution is the positive root whose $x=q+\alpha\log u$ lies in the
chosen cell. Let $B=2(A_i-\ell/N)$ and
$\Delta=B^2+4R_qL_q$. When $R_q>0$, evaluate the positive root in the
cancelation-safe form:
\begin{equation}
u=\begin{cases}
\displaystyle\frac{2L_q}{B+\sqrt{\Delta}},&B\ge0,\\[6pt]
\displaystyle\frac{-B+\sqrt{\Delta}}{2R_q},&B<0.
\end{cases}
\end{equation}
The first branch is derived using Vieta's formulas and avoids
subtracting nearly equal numbers. If a tail is numerically zero, use the
corresponding linear limit and then clamp only at the cell endpoint. For an
interior cell, choose $q=p_i$ or $q=p_{i+1}$ according to which endpoint is
closer in CDF mass to $\ell/N$. At either endpoint, the defining sums
$L_q$ and $R_q$ have non-positive exponents.
For $i=0$, use $q=p_1$ and the one-sided right-tail formula; for $i=M$, use
$q=p_M$ and the one-sided left-tail formula. Finally, verify that the root lies in
the selected cell. The same local scan values give
$\Phi'(b_\ell)=\big(e^{L_r(b_\ell;i_\ell)}+e^{R_r(b_\ell;i_\ell)}\big)/(2\alpha)$,
so no additional scan is required.

\paragraph{Remark (extreme-separation corner).}
When the linear coefficient vanishes ($B{=}0$, exactly balanced masses on
the two sides) and one tail underflows ($e^{-D/\alpha}$ below the
floating-point floor, i.e.\ a barrier more than ${\sim}88$ kernel widths
from all support in FP32), the Vieta branch and its linear fallback both
lose the root; the remedy is the log-space branch
$\log u=\tfrac12(\log L_q-\log R_q)$, with $L_q$ and $R_q$ kept as
logarithms throughout. None of our configurations enter this regime: with
strictly positive routing masses and index-spaced positions, barriers fall
within a few kernel widths of populated support, and the bounded-end
anchoring above covers one-sided underflow. Implementations targeting
extremely separated or near-degenerate mass configurations should enable
this branch. Separately from computational stability, the derivative factor
$1/\Phi'(b_\ell)$ can be genuinely ill-conditioned when the density at a
barrier is very small; no logarithmic representation removes this, as it is
a property of the problem rather than of the implementation.

\subsection{Differentiating the Barriers}
\label{app:algorithm-sensitivity}

\paragraph{Derivation of the gradient formula.}
Each barrier satisfies the scalar equation
$\Phi(b_\ell;r)=\ell/N$. Differentiating it with respect to $r_j$ and using
$\partial_x\Phi=\Phi'(b_\ell)>0$ gives:
\begin{equation}
\frac{\partial b_\ell}{\partial r_j}
=-\frac{\sigma((b_\ell-p_j)/\alpha)}{\Phi'(b_\ell)}.
\end{equation}
This is a dense Jacobian entry, but training requires only its products with a
tangent or a cotangent. The rest of this section demonstrates how to compute both
products using scans.

\paragraph{Motivation against direct materialization.}
The entries form an $(N-1)\times M$ matrix, but materializing it directly would restore
the quadratic cost removed during the forward pass. A forward-mode tangent
requires $\dot b=(\partial b/\partial r)\dot r$, whereas reverse-mode training requires
the transpose product $\bar r=(\partial b/\partial r)^\top\bar b$. Both
operations can be formulated as matrix-free scan problems.

\paragraph{Exploiting shared structure.}
The kernel factor in \eqref{eq:jvp} is identical to the factor used by $\Phi$.
Consequently, the JVP numerator is the generalized scan
$\Phi_{\dot r}(b_\ell;i_\ell)$, evaluated at the previously identified barrier cell.
For the VJP, the roles reverse: we scan over the $N-1$ barriers once and read
the complement of the resulting CDF across all $M$ token positions. Neither direction materializes individual Jacobian
entries.

\paragraph{Analytical backward and saved state.}
The forward barrier solve runs under \texttt{no\_grad}. A custom
\texttt{torch.autograd.Function} supplies the exact JVP/VJP formulas instead
of tracing through the discrete cell search, endpoint branch, and square-root solve.
The benefit is a small, predictable saved state: no computation graph is
stored through the branchy solver, the backward reuses the barriers, cell
indices, and scan values already computed in the forward pass, and the fixed
formulas avoid per-operation tracing overhead. Both products retain the
$\mathcal{O}(M+N)$ cost of the forward pass.

\begin{algorithm}[t!]
\caption{\textsc{BarrierSensitivity} -- the custom backward pass, Theorem~\ref{thm:laplace-linear}, $\mathcal{O}(M+N)$}
\label{alg:sensitivity}
\begin{algorithmic}[1]
\Require $p_{1:M}, r_{1:M}, \alpha$, and $(b_{1:N-1}, i_{1:N-1}, \Phi'(b_{1:N-1}))$ from Algorithm~\ref{alg:barriers}, computed under \texttt{no\_grad}
\Procedure{Forward (JVP)}{tangent $\dot r_{1:M}$}
  \State $\dot b_\ell \gets -\Phi_{\dot r}(b_\ell;i_\ell)/\Phi'(b_\ell)$ for every $\ell$
  \State \Return $\dot b_{1:N-1}$
\EndProcedure
\Procedure{Reverse (VJP)}{cotangent $\bar b_{1:N-1}$}
  \State $c_\ell \gets -\bar b_\ell/\Phi'(b_\ell)$ for every $\ell$
  \State merge-search $p_{1:M}$ against the sorted barriers $b_{1:N-1}$
  \State $C\gets\sum_{\ell=1}^{N-1}c_\ell$; $\bar r_j\gets C-\Psi_c(p_j;k_j)$ for every $j$, with $k_j=\#\{\ell:b_\ell\le p_j\}$
  \State \Return $\bar r_{1:M}$
\EndProcedure
\end{algorithmic}
\end{algorithm}

Theorem~\ref{thm:laplace-linear} now follows operationally: both products
use one scan over the longer sequence and one merge over the shorter one.
For the JVP, we build $L_{\dot r},R_{\dot r}$ with
\textsc{SignedLogCumsumExp}, evaluate them at the stored barrier cells, and divide
by the stored densities. No new cell search is needed. For the VJP, substituting \eqref{eq:jvp} into the adjoint yields:
\begin{equation}
\bar r_j=\sum_{\ell=1}^{N-1}\frac{\partial b_\ell}{\partial r_j}\,\bar b_\ell
=\sum_{\ell=1}^{N-1}\underbrace{-\frac{\bar b_\ell}{\Phi'(b_\ell)}}_{=:c_\ell}\,\sigma\!\left(\frac{b_\ell-p_j}{\alpha}\right),
\end{equation}
where $c_\ell$ already absorbs the minus sign. To make the role swap
explicit, take the already ordered barrier support points $z_\ell=b_\ell$
(the barriers are sorted by construction, so no additional sort is
performed) with weights $c_\ell$ and
define, for a query $y$ in cell $k$,
\begin{equation}
\Psi_c(y;k)=\sum_{\ell\le k}c_\ell
-\tfrac12\sum_{\ell\le k}c_\ell e^{(z_\ell-y)/\alpha}
+\tfrac12\sum_{\ell>k}c_\ell e^{(y-z_\ell)/\alpha}.
\end{equation}
The displayed $\Psi_c(p_j;k_j)$ contains
$\sigma((p_j-b_\ell)/\alpha)$, whereas the barrier derivative contains
$\sigma((b_\ell-p_j)/\alpha)$. By Laplace CDF symmetry,
\begin{equation}
\label{eq:appendix-reverse-cdf}
\boxed{\bar r_j=C-\Psi_c(p_j;k_j)},\qquad
C=\sum_{\ell=1}^{N-1}c_\ell.
\end{equation}
Here $k_j$ is the number of barriers at or below $p_j$;
$k_j=0$ and $k_j=N-1$ cover the sentinel regions. Equivalently, the
right-hand side is
$\tfrac12\sum_{\ell\le k_j}c_\ell e^{(b_\ell-p_j)/\alpha}
+\sum_{\ell>k_j}c_\ell
-\tfrac12\sum_{\ell>k_j}c_\ell e^{(p_j-b_\ell)/\alpha}$.
While a
direct evaluation requires $\mathcal{O}(MN)$ time, performing a single signed prefix/suffix
construction on the barriers followed by a merge search over the token
positions computes all $\bar r_j$ in $\mathcal{O}(M+N)$ time.

\subsection{Aggregating Tokens Given the Barriers}
\label{app:algorithm-aggregate}

\begin{algorithm}[t!]
\caption{\textsc{Aggregate} -- reading off every merged token, $\mathcal{O}((M+N)d)$}
\label{alg:aggregate}
\begin{algorithmic}[1]
\Require tokens $x_{1:M}$, weights $r_{1:M}$, positions $p_{1:M}$, scale $\alpha$, and $(b_{0:N}, i_{0:N})$ from Algorithm~\ref{alg:barriers}
\Ensure merged tokens $t_{1:N}$
\State $\Phi_x(b_0) \gets 0$,\quad $\Phi_x(b_N) \gets \sum_{i=1}^M r_i x_i$ 
\For{$\ell=1,\dots,N-1$}
  \State $\Phi_x(b_\ell) \gets \Phi_{r\odot x}(b_\ell;i_\ell)$ 
\EndFor
\For{$\ell=1,\dots,N$}
  \State $t_\ell \gets N\big[\Phi_x(b_\ell)-\Phi_x(b_{\ell-1})\big]$
\EndFor
\State \Return $t_{1:N}$
\end{algorithmic}
\end{algorithm}

Aggregation repeats the same scan using the vector-valued signed weights
$w_i=r_i x_i$. The stored cell index $i_\ell$ is reused for every barrier, requiring
no second search. For each feature coordinate, we evaluate
$\Phi_{r\odot x}(b_\ell;i_\ell)$ and take the difference of consecutive values. The
boundary limits are explicit: $\Phi_x(-\infty)=0$ and
$\Phi_x(+\infty)=\sum_{i=1}^M r_i x_i$. Multiplication by $N$ yields the
conditional expectation in Definition~\ref{def:merge}. The computational cost of the scan is
$\mathcal{O}((M+N)d)$, including all feature coordinates.

\paragraph{Custom backward pass for aggregation.} This backward pass differs
from the barrier backward pass: it differentiates the readout with respect to both
$x$ and $r$. Signed scans represent $v=v^+-v^-$ using the logarithms of their two
non-negative components. Although a zero component is correctly handled as a dead branch in the
forward pass, differentiating its $\log(0)=-\infty$ node directly can
produce spurious NaNs. The implementation therefore executes the forward scan
under \texttt{no\_grad} and evaluates the exact Laplace-kernel products inside a
custom backward pass. The resulting gradients remain finite and preserve the same
$\mathcal{O}((M+N)d)$ asymptotic complexity.

The readout adjoint can be organized as follows. Let
$\bar t_\ell$ denote the incoming gradient and define
$F_\ell=\Phi_x(b_\ell)$, such that $t_\ell=N(F_\ell-F_{\ell-1})$. The
corresponding cotangents are
$\bar F_0=-N\bar t_1$ (formal only, since $F_0=0$ is constant),
$\bar F_\ell=N(\bar t_\ell-\bar t_{\ell+1})$ for $1\le\ell<N$, and
$\bar F_N=N\bar t_N$. The CDF difference has already been applied
in these cotangents. Define the vector-valued quantity
\begin{equation}
\label{eq:appendix-readout-adjoint}
g_i=\sum_{\ell=1}^{N}\bar F_\ell
\sigma\!\left(\frac{b_\ell-p_i}{\alpha}\right)
=N\bar t_N+\sum_{\ell=1}^{N-1}N(\bar t_\ell-\bar t_{\ell+1})
\sigma\!\left(\frac{b_\ell-p_i}{\alpha}\right).
\end{equation}
Here $\sigma((b_N-p_i)/\alpha)=1$. Equivalently, one could use
$g_i=N\sum_{\ell=1}^{N}\bar t_\ell
[\sigma((b_\ell-p_i)/\alpha)-\sigma((b_{\ell-1}-p_i)/\alpha)]$;
this form uses $\bar t$, not $\bar F$. The direct readout gradients are
$\bar x_i\mathrel{+}=r_i g_i$ and
$\bar r_i^{\mathrm{direct}}\mathrel{+}=x_i^\top g_i$.
All $g_i$ are obtained by one barrier-support CDF scan with the
\emph{unreversed} argument $\sigma((b_\ell-p_i)/\alpha)$, plus the
constant endpoint term. Moving a barrier contributes
\begin{equation}
\label{eq:appendix-readout-barrier}
\bar b_\ell
=\bar F_\ell^\top\Phi_x'(b_\ell),\qquad
\Phi_x'(b_\ell)=\frac1\alpha\sum_i r_i x_i
\kappa\!\left(\frac{b_\ell-p_i}{\alpha}\right),
\quad 1\le\ell<N.
\end{equation}
These cotangents enter Algorithm~\ref{alg:sensitivity}'s reverse procedure,
and its result is added to $\bar r^{\mathrm{direct}}$. Both
the CDF scan and the density scan in~\eqref{eq:appendix-readout-barrier}
take $O((M+N)d)$ work. No $M\times N$ assignment matrix is formed.

\paragraph{Full proof of Theorem~\ref{thm:laplace-linear}.}
The forward construction first computes the scalar prefix/suffix states
in~\eqref{eq:appendix-cell-split}. The monotone merge search and the
closed-form root described in Appendix~\ref{app:algorithm-barriers}
produce all barriers and densities in $O(M+N)$ work. Repeating the
prefix/suffix construction for the $d$ signed sequences $r_i x_{i,c}$
evaluates $F_1,\ldots,F_{N-1}$ in $O((M+N)d)$ work. Setting $F_0=0$
and $F_N=\sum_i r_i x_i$ and differencing adjacent values gives all
outputs in $O(Nd)$. Hence the forward bound holds.

For a tangent $\dot r$ on the simplex, Equation~\eqref{eq:appendix-barrier-jvp}
is exactly a Laplace-CDF scan with signed weights $\dot r_i$, read at
the stored barrier cells. It costs $O(M+N)$. In reverse mode,
Equation~\eqref{eq:appendix-reverse-cdf} scans the signed barrier
cotangents $c_\ell$ and reads them at every $p_j$, also in $O(M+N)$.
For the feature readout, differentiating $r_i x_i$ and the CDF evaluations
only introduces further signed scans and pointwise products. In reverse,
\eqref{eq:appendix-readout-adjoint} computes the direct derivatives
with respect to $r$ and $x$ by one CDF scan, while
\eqref{eq:appendix-readout-barrier} supplies the barrier cotangents
by one density scan. Applying the barrier VJP completes the chain rule.
Each scan, gather, and merge is linear in its input and output lengths,
and there are only a constant number of them per feature coordinate.
Thus forward, JVP, and VJP all cost $O((M+N)d)$, reducing to
$O(M+N)$ for fixed $d$. The case $N=1$ has no interior barriers and
reduces directly to $t_1=\sum_i r_i x_i$, satisfying the same bound.

\paragraph{Algorithmic bound and portable implementation.}
The stated linear-work bound assumes the positions are already ordered,
as in Section~\ref{sec:method}, and that the two ordered searches are
implemented as merge passes. A portable PyTorch implementation may use
\texttt{sort} for unordered coordinates and \texttt{searchsorted} for
each query. Those choices compute the same mathematical operator and
retain matrix-free storage, but their worst-case work includes sorting
and binary-search overhead. The linear bound describes the ordered
scan-level construction, with the merge searches shown in
Algorithms~\ref{alg:barriers} and~\ref{alg:sensitivity}.

\paragraph{Implementation checklist.} A correct implementation should satisfy
six conditions: (i) recenter every decayed prefix/suffix sum as in
\eqref{eq:recentered}; (ii) split signed inputs before scanning as in
\eqref{eq:signed-scan}; (iii) locate cells from $\mathrm{logit}\,\Phi_r(p_i;i)$,
not $A_i$; (iv) solve the in-cell quadratic using the numerically safe endpoint
and root branch; (v) implement separate matrix-free custom backward passes for
barriers and aggregation, running forward scans under \texttt{no\_grad}; and
(vi) maintain $\alpha$ proportional to token spacing rather than to the full
sequence width. Appendix~\ref{app:runtime-implementations} compares the
dense, native-scan, and fused-CUDA implementations of this operator.

\providecommand{\runtimefigure}[2]{%
  \IfFileExists{time_comparison/results/triangular_figures/#1.pdf}
    {\includegraphics[width=#2]{time_comparison/results/triangular_figures/#1.pdf}}
    {\includegraphics[width=#2]{results/triangular_figures/#1.pdf}}}

\subsection{Scorer and routing normalization}
\label{app:scorer}

The scorer is a lightweight module applied independently to the input token
representations. In the default configuration, an MLP produces one scalar
logit $s_i$ for each token. The default LapSum soft-top-$k$ construction is
the differentiable routing mechanism introduced in the original LapSum work
~\citep{struski2025lapsum},
\begin{equation}
 s_i = f_{\mathrm{MLP}}(x_i),
 \qquad i=1,\ldots,M.
\end{equation}
LapMerge does not use the logits directly. They are converted into
non-negative routing masses $r_i$ and normalized so that
$\sum_i r_i=1$. We consider the following interchangeable routing maps:
\begin{align}
 r_i^{\mathrm{sig}} &=
 \frac{\operatorname{sigmoid}(s_i)}
 {\sum_j\operatorname{sigmoid}(s_j)},
 &&\text{sigmoid + normalization},\\
 r_i^{\mathrm{sp}} &=
 \frac{\operatorname{softplus}(s_i)}
 {\sum_j\operatorname{softplus}(s_j)},
 &&\text{softplus + normalization},\\
 r_i^{\mathrm{sm}} &= \operatorname{softmax}(s)_i,
 &&\text{softmax},\\
 r_i^{\mathrm{stk}} &=
 \frac{1}{k}\,\operatorname{soft\text{-}top\text{-}k}(s)_i,
 &&\text{soft-top-$k$ + normalization}.
\end{align}
The first three maps differ in how strongly they suppress low-scoring tokens.
The soft-top-$k$ map additionally encodes the target budget: its unscaled
outputs have total mass $k$, so division by $k$ produces a probability mass
function. Unless stated otherwise, the experiments in this paper use
soft-top-$k$ routing as the default scorer normalization. The other maps are
ablations or alternative deployment choices; once the masses $r_i$ have been
formed, the same LapMerge operator and the same barrier equations apply.

The direct implementation of soft-top-$k$ sorts the $M$ scalar logits and
therefore has theoretical cost $O(M\log M)$, followed by the linear
LapMerge computation. This term is usually not visible in measured runtime:
the sort acts on scalars rather than $d$-dimensional tokens and is implemented
by highly optimized GPU primitives. It can nevertheless matter at extremely
large sequence lengths. Fast LapSum~\citep{struski2026fastlapsum} replaces the
full sort of the scorer logits with probabilistic bracketing that narrows and
sorts only a small band around the threshold, giving an expected-linear
routing step. Note that this concerns the ordering of the \emph{logit values}
inside the routing module; it is unrelated to the ordering of the token
positions $p_i$, which the merge operator assumes as given. We do not require
that more elaborate routing implementation in the present experiments, because
the sorting cost is negligible in our measured regime; the distinction becomes
relevant for sequences substantially above $10^6$ tokens. The sigmoid and softplus alternatives require only pointwise
evaluation and normalization, so they do not introduce this sorting term.

\paragraph{Normalized masses are a specification, not the computational
representation.} We state the routing maps as normalized probability vectors,
$r_i\ge0$ with $\sum_i r_i=1$, because that is the cleanest way to specify
what the barrier construction of \S\ref{app:algorithm} requires
mathematically. The operator's interface accepts either form: the runtime
benchmark, for instance, passes an explicitly materialized (numerically
stabilized) softmax of the logits. The numerically critical part is not
whether $r$ is materialized but how the barrier solver represents its
internal quantities: the search operates on log-scale node values,
$\mathrm{node}_i=\mathrm{logit}\,\Phi_r(p_i;i)$ as constructed in
\S\ref{app:algorithm-barriers}, together with log-scale prefix/suffix
accumulators, so extremely small tail contributions survive at large
$M/\alpha$ where a plain-scale accumulation would underflow. In an optional
fully log-space path the normalized masses need never be formed explicitly:
$\log r_i=s_i-\operatorname{LSE}_j(s_j)$ feeds the same scans directly. Either
way, the normalized-probability view of $r_i$ above is a specification of the
operator's semantics, not a claim about which tensors an implementation
stores.

\paragraph{Configuration recipe per experiment.}
Table~\ref{tab:recipe} collects, in one place, which scorer, routing map,
budget, and widths each experiment uses, and which modules are trained.
Two distinct temperatures appear throughout the paper and should not be
conflated: the \emph{routing} temperature inside soft top-$k$ (fixed at $1$
everywhere) and the \emph{merge} kernel width $\alpha$ of
\eqref{eq:mix}, whose units follow the position coordinate: normalized
positions in $[0,1]$ for histopathology (so the learned width starts at the
mean bin spacing $1/N$), integer rank indices for LRA and POPE, and a
learned scalar coordinate for Cars and arXiv, where the scorer emits a
position alongside the importance and the width is learned per lifted
coordinate. The routing budget $k$ coincides with the output budget $N$ in
every trained model; in POPE, $k=\lceil bM\rceil$ equals both the number
of retained positions and the number of merged outputs.

\begin{table}[h]
\centering
\footnotesize
\setlength{\tabcolsep}{4pt}
\resizebox{\linewidth}{!}{%
\begin{tabular}{llllll}
\toprule
experiment & scorer (input $\to$ arch) & budget $N$ & merge width $\alpha$ & trained modules & backend \\
\midrule
histopathology & feats$+$xy $\to$ MLP(1536{+}2,128,1) & $1024$ &
  learned scalar, init $1/N$ & scorer, width, head & native \\
LRA & token $\to$ linear(dim,1) & $10\%$ of $M$ &
  fixed $1$ (rank units) & scorer, backbone, head & native \\
Cars & token $\to$ MLP(384,128,2) $\to(r,a)$ & $20$ &
  learned per-coord.\ $+$ frame $U$ & scorer, $U$, width, head & native \\
arXiv & token $\to$ MLP(768,$\cdot$,2) $\to(r,a)$ & $8$--$64$ &
  learned per-coord.\ $+$ frame $U$ & scorer, $U$, width, head & native \\
POPE & frozen VisionSelector scorer & $\lceil bM\rceil$ &
  fixed $(\alpha_r,\alpha_a){=}(1,1)$, index units & none (eval-only) & native \\
\bottomrule
\end{tabular}}
\caption{\textbf{Configuration recipe per experiment.} Routing is LapSum
soft top-$k$ at temperature $1$ in every row (POPE standardizes the frozen
scores per image first); the runtime benchmark instead uses softmax routing,
as stated in Appendix~\ref{app:runtime-comparison}. ``Native'' is the pure
PyTorch implementation; the CUDA-fused backend appears only in the runtime
study. Backbones are frozen except on LRA, where the reducer and the
backbone are trained jointly from scratch; the row shows the
configuration shared by all five LRA tasks (only the input
dimensionality varies).}
\label{tab:recipe}
\end{table}

\paragraph{Notation-to-code mapping.} The released code names the
mathematical objects as follows: masses $r_i$ appear as \texttt{p} (or
\texttt{masses}), positions $p_i$ as \texttt{a} (or \texttt{coordinates}),
the kernel width $\alpha$ as \texttt{width} (in the POPE adapter, the pair
$(\alpha_r,\alpha_a)$ as \texttt{route\_alpha}/\texttt{agg\_alpha}), the
barriers $b_\ell$ as \texttt{barriers}, and the output budget $N$ as
\texttt{n\_out} (the routing budget $k$ of soft top-$k$ as \texttt{k}).
Throughout the paper, $\kappa$ denotes the kernel density and $k$ is
reserved for the routing/selection budget.

\section{Runtime and Large-Scale Token Reduction}
\label{app:runtime-comparison}

This appendix gives the implementation and measurement details behind the
runtime comparison in the main paper. We study the standard consecutive-barrier
\our{} operator itself, including barrier construction, aggregation,
normalization, and its backward pass. The encoder and the external scorer are
excluded (their measured share of the end-to-end time is reported in
Table~\ref{tab:e2e-scorer}). Throughout, $M$ and $N$ denote the input and output lengths, while
$B=1$, $d=512$, FP32, and an NVIDIA GH200 120GB are used unless noted
otherwise.

\subsection{What is being accelerated?}
\label{app:runtime-algorithm}

For a barrier interval $i$, the direct formulation assigns every input token
$j$ a Laplace mass and then forms a weighted sum of its value. It is easy to
state, but materializing the resulting $N\times M$ assignment map is the
source of quadratic work. LapMerge uses the identity
\begin{equation}
 \sum_j r_j e^{-|b-p_j|/\alpha}
 =e^{-b/\alpha}\!\sum_{p_j\le b}r_j e^{p_j/\alpha}
  +e^{b/\alpha}\!\sum_{p_j>b}r_j e^{-p_j/\alpha},
 \label{eq:runtime-separation}
\end{equation}
and the same identity with $r_j$ replaced by $r_jx_j$ for value-weighted
quantities. The two sums are prefix and suffix summaries. Interval masses and
numerators are obtained from the corresponding cumulative expressions at the
two bounding barriers. Thus the same interval operator is evaluated without
storing its explicit assignment map. For fixed $B$ and
$d$, the scan and readout therefore have work $O((M+N)d)$ and memory
$O((M+N)d)$, rather than $O(MN d)$. The sequence order is already available,
so this bound contains no additional sorting term.

The backward pass reuses the same factorization. Log-domain accumulation and
local rescaling prevent the exponential factors in Eq.~\ref{eq:runtime-separation}
from overflowing, while the scan structure keeps gradient computation linear
in the two sequence lengths. This is why the method remains differentiable at
large input sizes; the dense computation described below is only a faster
implementation for small shapes, not a different model.

\subsection{Three implementation levels of one operator}
\label{app:runtime-implementations}

There are three useful implementation levels for the same differentiable
operator. They differ in how much work is delegated to matrix kernels and how
many intermediate tensors are materialized, but they use the same barriers,
Laplace weights, normalization, and gradients.

\paragraph{1. Dense reference implementation.}
For small $M$ and $N$, one can explicitly form the assignment weights between
each input coordinate and each output barrier. The resulting dense matrix is
then multiplied by the values and by a vector of ones to obtain the numerator
and denominator of each output token. This is the simplest implementation to
read and an important correctness oracle: it exposes every pairwise weight and
is convenient for FP64 checks and gradient tests. Its cost is $O(MN)$ in work
and memory, so it is unsuitable for long sequences, but it can be faster than
several scan kernels when the matrix is small enough for a highly optimized
GEMM or fused attention primitive.

\paragraph{2. Native PyTorch scan implementation.}
The scalable reference avoids the assignment matrix. It evaluates the two
terms in Eq.~\ref{eq:runtime-separation} with log-domain cumulative scans,
then gathers the scan states at the barriers and performs the final
normalization. Numerators are computed by running the same summaries on
$p_jx_j$, while the denominator uses $p_j$. In eager mode this is assembled
from ordinary PyTorch operations such as \texttt{logcumsumexp}, gathers, and
scatter reductions. \texttt{torch.compile} can fuse parts of this graph and
reduce Python and kernel-launch overhead, although the compiler does not
automatically remove every intermediate tensor. This variant is portable,
easy to inspect, and works without a custom extension; its main limitation is
that the framework may launch separate kernels for logically adjacent scan,
readout, and backward stages.

\paragraph{3. Fused CUDA implementation.}
The performance version moves the same computation into custom CUDA kernels.
It fuses coordinate transforms, prefix/suffix accumulation, barrier lookup,
weighted readout, and normalization where this reduces global-memory traffic.
The backward path follows the same decomposition rather than differentiating
through a large materialized assignment map. The scan uses a
decoupled-lookback design: thread blocks process local tiles in parallel and
consume compact summaries from earlier tiles, preserving linear work while
avoiding a global synchronization after every tile. This is the version used
for the large GH200 measurements and the main-paper curve. It has the highest
implementation complexity and requires compiling a CUDA extension, but it
minimizes intermediate storage and kernel-launch overhead.

In practice, the operator dispatches between the dense path for small
assignment problems and the scan path for large ones. The native and fused
scan paths can use the same shape-aware decision, so this is not a comparison
between different algorithms. The crossover depends on accelerator, dtype,
batch size, and feature dimension; it is therefore calibrated empirically
rather than presented as a universal threshold.

\paragraph{Backend agreement is a measured result.}
Table~\ref{tab:backend-agreement} reports the maximum relative forward
deviation of every backend against a dense-FP64 reference on identical
inputs (\texttt{bench\_agreement.py}: $B{=}2$, $d{=}8$, shapes from
$M{=}7$ to $131{,}072$ including $N{=}M$ and the formally admissible
$N{>}M$, widths $\{0.1,0.7,3\}$, and both moderate and near-hard masses,
the latter from logits scaled $10\times$). Three facts emerge.
\emph{(i)} The dense and scan factorizations are the same operator: in FP64
they agree to $5\times10^{-8}$ or better everywhere, including the
near-hard masses. \emph{(ii)} All FP32 arms (dense, eager scan, compiled,
and the auto dispatcher, which follows the documented rectangle/coupling
rule) deviate together as the problem's conditioning worsens -- up to a few
percent at large $M$ and far more under near-hard masses, where
$1/\Phi'(b_\ell)$ is genuinely ill-conditioned; this is input-precision
sensitivity of the problem, not backend disagreement. \emph{(iii)} The
CUDA-fused path matches to $10^{-12}$ on small shapes but is the one
implementation with genuine limits: at large $M$ it deviates by up to
$\sim0.5$ in max-relative terms even in FP64, and under near-hard masses it
overflows outright (deviations of order $10^{10}$--$10^{14}$, identical in
FP32 and FP64, i.e.\ an algorithmic-path issue rather than precision).
Accuracy experiments therefore use the native implementation throughout
(Table~\ref{tab:recipe}), and the timing study drives all backends with
moderate softmax masses. Two further boundary facts: $N{=}M$ and $N{>}M$
evaluate correctly in every backend, while $N{=}1$ -- the trivial case
$t_1=\sum_i r_ix_i$ with no interior barriers
(Appendix~\ref{app:algorithm}) -- is currently rejected by the
implementations' input check rather than dispatched to this closed form.

\begin{table}[h]
\centering
\small
\setlength{\tabcolsep}{6pt}
\begin{tabular}{lcccc}
\toprule
 & \multicolumn{2}{c}{moderate masses} & \multicolumn{2}{c}{near-hard masses} \\
backend (vs dense-FP64) & small & large & small & large \\
\midrule
scan-FP64 & $4\times10^{-12}$ & $3\times10^{-11}$ & $3\times10^{-10}$ & $5\times10^{-8}$ \\
dense-FP32 & $2\times10^{-3}$ & $2\times10^{-2}$ & $0.16$ & $30$ \\
scan-FP32 (eager) & $3\times10^{-3}$ & $3\times10^{-2}$ & $0.37$ & $56$ \\
auto-FP32 / compile-FP32 & $2\times10^{-3}$ & $3\times10^{-2}$ & $0.16$ & $56$ \\
fused-FP64 & $3\times10^{-13}$ & $0.57$ & $2\times10^{-12}$ & $3\times10^{14}$ \\
fused-FP32 & $2\times10^{-3}$ & $0.57$ & $0.16$ & $3\times10^{14}$ \\
\bottomrule
\end{tabular}
\caption{\textbf{Measured backend agreement}: maximum relative forward
deviation from a dense-FP64 reference on identical inputs. ``Small'' covers
$M\le4096$ with $N\le512$ (plus $N{=}M$, $N{>}M$ cases); ``large'' covers
$M\ge65{,}536$ or $N\ge1024$ at $M\ge4096$. Gradient deviations follow the
same pattern. FP32 rows degrade together with the problem's conditioning
(shared input-precision sensitivity); the fused rows mark that backend's
genuine limits, which the accuracy experiments avoid by using the native
implementation.}
\label{tab:backend-agreement}
\end{table}

\paragraph{Agreement on real trained-scorer masses.}
The grid above uses synthetic masses. As a stronger control, we captured
the operator's \emph{actual} inputs -- masses, positions, and the learned
width -- from the trained histopathology checkpoints on all $98$ test
slides of both cohorts ($M$ from $1{,}068$ to $25{,}389$) by wrapping the
merge call, and repeated the agreement measurement on them
(Table~\ref{tab:realmass-agreement}). On every real configuration the
fused backend matches the scan-FP64 reference to $1.9\times10^{-12}$ in
FP64 and to $4.6\times10^{-4}$ in FP32, indistinguishable from native
FP32 itself ($8.2\times10^{-4}$). The large-$M$ deviations of
Table~\ref{tab:backend-agreement} reappear only in synthetic
million-token probes built by tiling real slide mass profiles (with the
learned width and local mass geometry preserved): up to $0.40$
max-relative at $N{=}M/16$, identical in FP32 and FP64. They are a
property of that backend's algorithmic path at extreme shapes, not of
realistic mass configurations; accuracy experiments use the native
implementation throughout, and the pooler comparison in
\S\ref{app:runtime-competitors} is restated for the native path as well.

\begin{table}[h]
\centering
\small
\setlength{\tabcolsep}{6pt}
\begin{tabular}{lccc}
\toprule
 & real slides ($98$) & tiled $\sim\!10^6$, $N{=}1024$ & tiled $\sim\!10^6$, $N{=}M/16$ \\
\midrule
scan-FP32  & $8.2\times10^{-4}$ & $2.2\times10^{-3}$ & $8.3\times10^{-2}$ \\
fused-FP32 & $4.6\times10^{-4}$ & $5.6\times10^{-2}$ & $4.0\times10^{-1}$ \\
fused-FP64 & $1.9\times10^{-12}$ & $5.6\times10^{-2}$ & $4.0\times10^{-1}$ \\
\bottomrule
\end{tabular}
\caption{\textbf{Backend agreement on real trained-scorer masses}:
maximum relative forward deviation from the scan-FP64 reference (which
agrees with dense-FP64 to $\le5\times10^{-8}$,
Table~\ref{tab:backend-agreement}), over the trained histopathology
models' own (mass, position, width) inputs on all test slides, and over
million-token probes tiled from the same profiles. Gradient deviations
follow the same pattern. Identical fused-FP32/FP64 values on the tiled
probes mark the algorithmic-path limit already documented above.}
\label{tab:realmass-agreement}
\end{table}

\subsection{Backend ablations and dispatch calibration}
\label{app:runtime-backend-ablation}

To separate algorithmic scaling from implementation overhead, we measure
four implementations of the same operator: explicit dense PyTorch, native
PyTorch scan in eager mode, the native path under
\texttt{torch.compile}, and the shape-aware dense/CUDA-fused implementation
used in Figure~\ref{fig:runtime-main-scaling}. Every arm receives the same
ordered coordinates $a_j=j$, normalized token masses, and FP32 features:
none is charged for input sorting. We use one GH200, $B=1$, $d=512$, five
warm-ups, 30 synchronized repeats, and a fresh process per shape. The
measured backward includes gradients through values, coordinates, and
score logits. First-call compilation time is recorded in the released CSVs
but excluded from the steady-state curves.

We sweep powers of two from $M=2^{11}$ through $2^{24}$ in two regimes:
$N=1024$ isolates input-length growth, whereas $N=M/16$ grows both dimensions.
The explicit dense arm was intentionally capped at $M=2^{20}$; the other
implementations were attempted through $2^{24}$. Failed cells are reported
as failures, not assigned an extrapolated latency. Figures
\ref{fig:app-backend-time} and \ref{fig:app-backend-memory} show the full
matched time and incremental-memory measurements.

The dense path forms an explicit $M\times N$ assignment matrix and is a
useful correctness oracle. Its $O(MN)$ work and storage can nevertheless be
competitive at modest sizes. At fixed $N=1024$, both dense and scan paths
are formally linear in $M$; the striking difference is the memory constant.
When $N=M/16$, the dense work instead grows quadratically while the scan
retains $O((M+N)d)$ work and storage. This is the decisive regime for the
scaling claim, rather than an apparent slope difference at fixed $N$.

\begin{figure*}[t]
\centering
\begin{minipage}{0.485\textwidth}
\centering
\begin{tikzpicture}
\begin{axis}[
  width=\linewidth, height=0.68\linewidth,
  xmode=log, log basis x=2, ymode=log, log basis y=2,
  xmin=2048, xmax=16777216, ymin=2, ymax=4096,
  xtick={2048,16384,131072,1048576,8388608,16777216},
  xticklabels={$2^{11}$,$2^{14}$,$2^{17}$,$2^{20}$,$2^{23}$,$2^{24}$},
  xlabel={input tokens $M$},
  ylabel={forward + backward [ms]}, title={$N=1024$},
  grid=both, grid style={gray!20},
  tick label style={font=\scriptsize},
  label style={font=\small}, title style={font=\small},
  legend columns=2,
  legend style={at={(0.5,1.02)},anchor=south,draw=none,font=\scriptsize},
  legend cell align=left,
]
\addplot[very thick,color=densepurple,mark=none] coordinates {(2048,4.322057) (4096,4.957128) (8192,5.532889) (16384,5.045958) (32768,6.820837) (65536,10.856902) (131072,18.399736) (262144,33.806018) (524288,65.042395) (1048576,128.515329)};
\addlegendentry{Dense PyTorch}
\addplot[very thick,color=patchblue,mark=none] coordinates {(2048,5.148069) (4096,5.801223) (8192,5.408961) (16384,5.294276) (32768,6.912650) (65536,10.909921) (131072,32.530262) (262144,61.025876) (524288,135.316077) (1048576,274.298890) (2097152,549.385081)};
\addlegendentry{Native PyTorch}
\addplot[only marks,mark=*,mark size=2.2pt,color=patchblue,forget plot] coordinates {(2097152,549.385081)};
\addplot[very thick,color=compilegreen,mark=none] coordinates {(2048,4.145955) (4096,3.989927) (8192,4.516873) (16384,4.097172) (32768,4.349421) (65536,4.636934) (131072,23.584179) (262144,46.375565) (524288,104.423210) (1048576,214.759136) (2097152,430.816370)};
\addlegendentry{\texttt{torch.compile}}
\addplot[only marks,mark=x,mark size=2.2pt,color=compilegreen,forget plot] coordinates {(2097152,430.816370)};
\addplot[very thick,color=lapmergeblack,mark=none] coordinates {(2048,4.967496) (4096,4.852891) (8192,5.209026) (16384,4.727311) (32768,7.095230) (65536,10.612173) (131072,17.210842) (262144,29.500665) (524288,55.093661) (1048576,111.038612) (2097152,218.465804) (4194304,439.278953) (8388608,899.086194) (16777216,1878.611598)};
\addlegendentry{CUDA fused}
\end{axis}
\end{tikzpicture}
\end{minipage}
\hfill
\begin{minipage}{0.485\textwidth}
\centering
\begin{tikzpicture}
\begin{axis}[
  width=\linewidth, height=0.68\linewidth,
  xmode=log, log basis x=2, ymode=log, log basis y=2,
  xmin=2048, xmax=16777216, ymin=2, ymax=4096,
  xtick={2048,16384,131072,1048576,8388608,16777216},
  xticklabels={$2^{11}$,$2^{14}$,$2^{17}$,$2^{20}$,$2^{23}$,$2^{24}$},
  xlabel={input tokens $M$},
  ylabel={forward + backward [ms]}, title={$N=M/16$},
  grid=both, grid style={gray!20},
  tick label style={font=\scriptsize},
  label style={font=\small}, title style={font=\small},
  legend columns=2,
  legend style={at={(0.5,1.02)},anchor=south,draw=none,font=\scriptsize},
  legend cell align=left,
]
\addplot[very thick,color=patchblue,mark=none] coordinates {(2048,5.491783) (4096,5.261294) (8192,5.267373) (16384,5.212783) (32768,12.453447) (65536,19.219659) (131072,34.936231) (262144,65.320172) (524288,142.582110) (1048576,289.544161) (2097152,579.156068)};
\addlegendentry{Native PyTorch}
\addplot[only marks,mark=*,mark size=2.2pt,color=patchblue,forget plot] coordinates {(2097152,579.156068)};
\addplot[very thick,color=compilegreen,mark=none] coordinates {(2048,3.647407) (4096,4.184560) (8192,3.947351) (16384,4.324253) (32768,8.076712) (65536,13.710494) (131072,26.480903) (262144,50.586750) (524288,113.282731) (1048576,230.125509) (2097152,463.101563)};
\addlegendentry{\texttt{torch.compile}}
\addplot[only marks,mark=x,mark size=2.2pt,color=compilegreen,forget plot] coordinates {(2097152,463.101563)};
\addplot[very thick,color=lapmergeblack,mark=none] coordinates {(2048,5.453631) (4096,5.340295) (8192,5.397173) (16384,5.552017) (32768,16.909069) (65536,15.511568) (131072,21.232045) (262144,34.569398) (524288,67.052067) (1048576,136.090290) (2097152,267.304737) (4194304,531.172556) (8388608,1067.260153)};
\addlegendentry{CUDA fused}
\addplot[only marks,mark=*,mark size=2.2pt,color=lapmergeblack,forget plot] coordinates {(8388608,1067.260153)};
\end{axis}
\end{tikzpicture}
\end{minipage}
\caption{Matched GH200 forward--backward times for the same LapMerge operator at $B=1$, $d=512$, FP32. Coordinates are already ordered. Each point is the median of 30 synchronized repeats after five warm-ups. The dense path was only attempted to $M=2^{20}$. A filled endpoint marks the last success before OOM; a cross marks the last success before a compiler failure. Lines are measured points, not extrapolations.}
\label{fig:app-backend-time}
\end{figure*}
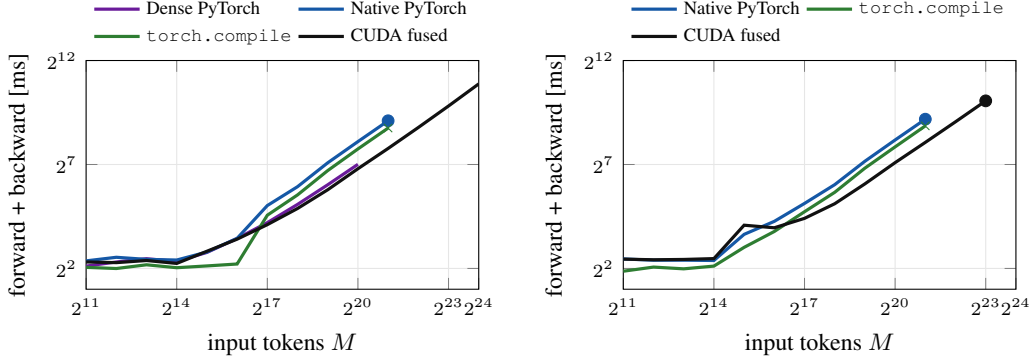

\begin{figure*}[t]
\centering
\begin{minipage}{0.485\textwidth}
\centering
\begin{tikzpicture}
\begin{axis}[
  width=\linewidth, height=0.68\linewidth,
  xmode=log, log basis x=2, ymode=log, log basis y=2,
  xmin=2048, xmax=16777216, ymin=0.008, ymax=128,
  xtick={2048,16384,131072,1048576,8388608,16777216},
  xticklabels={$2^{11}$,$2^{14}$,$2^{17}$,$2^{20}$,$2^{23}$,$2^{24}$},
  xlabel={input tokens $M$},
  ylabel={incremental peak [GiB]}, title={$N=1024$},
  grid=both, grid style={gray!20},
  tick label style={font=\scriptsize},
  label style={font=\small}, title style={font=\small},
  legend columns=2,
  legend style={at={(0.5,1.02)},anchor=south,draw=none,font=\scriptsize},
  legend cell align=left,
]
\addplot[very thick,color=densepurple,mark=none] coordinates {(2048,0.084038) (4096,0.199305) (8192,0.367342) (16384,0.703425) (32768,1.266215) (65536,2.376171) (131072,4.596083) (262144,9.036150) (524288,17.532724) (1048576,35.034189)};
\addlegendentry{Dense PyTorch}
\addplot[very thick,color=patchblue,mark=none] coordinates {(2048,0.084038) (4096,0.199305) (8192,0.367342) (16384,0.703425) (32768,1.266215) (65536,2.376171) (131072,2.772135) (262144,5.526651) (524288,11.035683) (1048576,22.053741) (2097152,44.089878)};
\addlegendentry{Native PyTorch}
\addplot[only marks,mark=*,mark size=2.2pt,color=patchblue,forget plot] coordinates {(2097152,44.089878)};
\addplot[very thick,color=compilegreen,mark=none] coordinates {(2048,0.089884) (4096,0.148508) (8192,0.203257) (16384,0.312754) (32768,0.531748) (65536,0.969736) (131072,3.022127) (262144,6.026645) (524288,12.035677) (1048576,24.053753) (2097152,48.089878)};
\addlegendentry{\texttt{torch.compile}}
\addplot[only marks,mark=x,mark size=2.2pt,color=compilegreen,forget plot] coordinates {(2097152,48.089878)};
\addplot[very thick,color=lapmergeblack,mark=none] coordinates {(2048,0.084038) (4096,0.199305) (8192,0.367342) (16384,0.703425) (32768,1.266215) (65536,2.376171) (131072,0.331838) (262144,0.646054) (524288,1.274463) (1048576,2.531295) (2097152,5.044971) (4194304,10.072316) (8388608,20.127004) (16777216,40.236379)};
\addlegendentry{CUDA fused}
\end{axis}
\end{tikzpicture}
\end{minipage}
\hfill
\begin{minipage}{0.485\textwidth}
\centering
\begin{tikzpicture}
\begin{axis}[
  width=\linewidth, height=0.68\linewidth,
  xmode=log, log basis x=2, ymode=log, log basis y=2,
  xmin=2048, xmax=16777216, ymin=0.008, ymax=128,
  xtick={2048,16384,131072,1048576,8388608,16777216},
  xticklabels={$2^{11}$,$2^{14}$,$2^{17}$,$2^{20}$,$2^{23}$,$2^{24}$},
  xlabel={input tokens $M$},
  ylabel={incremental peak [GiB]}, title={$N=M/16$},
  grid=both, grid style={gray!20},
  tick label style={font=\scriptsize},
  label style={font=\small}, title style={font=\small},
  legend columns=2,
  legend style={at={(0.5,1.02)},anchor=south,draw=none,font=\scriptsize},
  legend cell align=left,
]
\addplot[very thick,color=patchblue,mark=none] coordinates {(2048,0.011971) (4096,0.071396) (8192,0.175932) (16384,0.703425) (32768,0.723854) (65536,1.447691) (131072,2.895384) (262144,5.790764) (524288,11.581541) (1048576,23.163084) (2097152,46.326172)};
\addlegendentry{Native PyTorch}
\addplot[only marks,mark=*,mark size=2.2pt,color=patchblue,forget plot] coordinates {(2097152,46.326172)};
\addplot[very thick,color=compilegreen,mark=none] coordinates {(2048,0.065951) (4096,0.110416) (8192,0.154425) (16384,0.312754) (32768,0.786348) (65536,1.572697) (131072,3.145384) (262144,6.290770) (524288,12.581545) (1048576,25.163088) (2097152,50.326174)};
\addlegendentry{\texttt{torch.compile}}
\addplot[only marks,mark=x,mark size=2.2pt,color=compilegreen,forget plot] coordinates {(2097152,50.326174)};
\addplot[very thick,color=lapmergeblack,mark=none] coordinates {(2048,0.011971) (4096,0.071396) (8192,0.175932) (16384,0.703425) (32768,0.113814) (65536,0.227649) (131072,0.455302) (262144,0.910609) (524288,1.821223) (1048576,3.642451) (2097152,7.284907) (4194304,14.569819) (8388608,29.139649)};
\addlegendentry{CUDA fused}
\addplot[only marks,mark=*,mark size=2.2pt,color=lapmergeblack,forget plot] coordinates {(8388608,29.139649)};
\end{axis}
\end{tikzpicture}
\end{minipage}
\caption{Incremental peak CUDA allocation for the same matched sweeps. Allocation is measured above the already live input tensors, not as total device occupancy. At fixed $N$, the dense path also grows linearly with $M$ but has a much larger constant; $N=M/16$ exposes the different dependence on both sequence dimensions. Endpoint symbols have the meaning of Figure~\ref{fig:app-backend-time}.}
\label{fig:app-backend-memory}
\end{figure*}
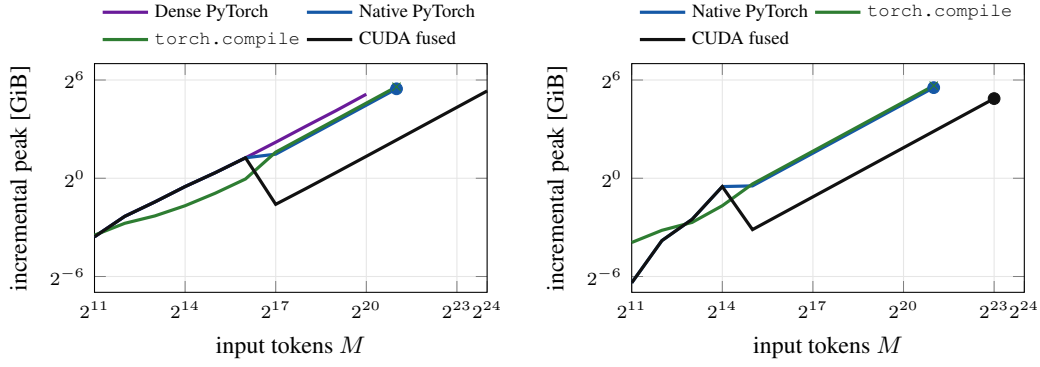

\begin{table*}[t]
\centering
\small
\setlength{\tabcolsep}{7pt}
\begin{tabular}{lrrrr}
\toprule
 & \multicolumn{2}{c}{$N=1024$} & \multicolumn{2}{c}{$N=M/16$} \\
Implementation & time [ms] & peak [GiB] & time [ms] & peak [GiB] \\
\midrule
Explicit dense & 128.52 & 35.03 & --- & --- \\
Native PyTorch & 274.30 & 22.05 & 289.54 & 23.16 \\
\texttt{torch.compile} & 214.76 & 24.05 & 230.13 & 25.16 \\
CUDA fused & \textbf{111.04} & \textbf{2.53} & \textbf{136.09} & \textbf{3.64} \\
\bottomrule
\end{tabular}
\caption{Matched GH200 measurements at $M=2^{20}$, $B=1$, $d=512$, FP32.
Time is steady-state forward plus backward. Peak is incremental CUDA allocation
above live inputs. The explicit dense reference was not swept at $N=M/16$;
its omission is a scope limit, not an OOM result.}
\label{tab:backend-ablation}
\end{table*}

The crossover is visible rather than assumed. At $N=1024$ and $M=2^{16}$,
the fused dispatcher still uses its dense branch; it takes $10.61$ ms versus
$10.86$ ms for forced dense. When $M$ doubles to $2^{17}$, the dispatcher
switches to the CUDA scan: incremental allocation falls from $2.38$ to
$0.33$ GiB even as the input grows, and latency remains close to the dense
reference ($17.21$ versus $18.40$ ms). This discontinuity is an implementation
switch, not a change to the LapMerge operator. For $N=M/16$ the branch changes
earlier, between $M=2^{14}$ and $2^{15}$, because the output length also grows.

At $M=2^{20}$, Table~\ref{tab:backend-ablation} shows that fused CUDA is
$2.47\times$ faster than eager native PyTorch for fixed $N$ and $2.13\times$
faster for $N=M/16$; its peak incremental allocation is $2.53$ and $3.64$
GiB, respectively. Below the crossover, \texttt{torch.compile} can be the
fastest path, so fusion is not presented as a universal small-shape win.
From $M=2^{20}$ to $2^{23}$, an eightfold input increase raises fused latency
by $8.10\times$ at fixed $N$ and $7.84\times$ at fixed ratio. Peak allocation
grows by approximately eightfold in both regimes. This is the empirical
signature expected of the scan/readout factorization, although finite-range
timings alone do not prove an asymptotic bound.

The native eager runs are out of memory from $M=2^{22}$ in both regimes.
\texttt{torch.compile} also reaches $2^{21}$, but its larger cells fail in
Triton code generation because a decoupled-lookback temporary is inferred
as both \texttt{int64} and \texttt{int32}; this is a compiler failure, not a
LapMerge OOM. The fused implementation reaches $M=2^{24}$ at fixed $N$ and
$M=2^{23}$ at fixed ratio, where the next attempted point is OOM.

\subsection{Matched comparison with learned poolers}
\label{app:runtime-competitors}

We compare the same shape-aware LapMerge dispatcher with TokenLearner and
PatchMerger. The latter two use scaled dot-product attention (SDPA) to compute
their published normalized pooling maps without storing an $M\times N$ matrix;
this is a memory-efficient rewrite, not a change to either pooler's output.
Concretely, the implemented variants are: PatchMerger $=$ LayerNorm followed
by a bias-free $d\times N$ projection, softmax over tokens; TokenLearner $=$
LayerNorm, Linear$(d,64)$, GELU, Linear$(64,N)$, softmax over tokens; the
SDPA forms pass the projection weights as queries with an explicit
$\mathrm{scale}{=}1$, so the bias (constant over positions) cancels in the
softmax. Equivalence is verified, not assumed: in FP64 the published and
SDPA forms agree to $\le10^{-14}$ relative error in the forward pass and in
every gradient across shapes up to $M{=}4096$
(\texttt{test\_sdpa\_equivalence.py}); in FP32 the SDPA kernel
introduces a $\sim3\times10^{-4}$ relative deviation, a backend-level
numerical effect, not a mathematical difference. These statements cover this
implementation, not every published TokenLearner variant.

All three reducers use the GH200, FP32, $B=1$, $d=512$, five warm-ups, 30
synchronized repeats, a fresh process per shape, and a 600-second whole-cell
limit. The input shape, mean-square loss, full backward pass, and timing rules
are matched, but the parameterizations are not identical: LapMerge receives
score logits and times their normalization, barriers, and readout, while the
poolers also time their built-in scoring/projection layers. External encoders,
LapMerge's task-specific logit predictor, and downstream networks are excluded.

\begin{table*}[h!]
\centering\small
\setlength{\tabcolsep}{7pt}
\begin{tabular}{lrrrr}
\toprule
 & \multicolumn{2}{c}{$N=1024$} & \multicolumn{2}{c}{$N=M/16$} \\
Reducer & time [ms] & peak [GiB] & time [ms] & peak [GiB] \\
\midrule
LapMerge (fused) & \textbf{111.04} & \textbf{2.53} & \textbf{136.09} & \textbf{3.64} \\
TokenLearner (SDPA) & 336.16 & 4.29 & 7158.16 & 4.29 \\
PatchMerger (SDPA) & 453.95 & 4.02 & 14750.05 & 4.38 \\
\bottomrule
\end{tabular}
\caption{Matched reducer measurements at $M=2^{20}$ on one GH200, $B=1$,
$d=512$, FP32. Times include forward, mean-square loss, and backward; peaks
are incremental CUDA allocations above live inputs. Each cell is a median of
30 synchronized repeats after five warm-ups.}
\label{tab:matched-poolers}
\end{table*}

Table~\ref{tab:matched-poolers} illustrates the difference between the two
regimes. These are reducer-core speedups, not whole-system ones: at fixed
$N$, the LapMerge core is $3.03\times$ faster than TokenLearner-SDPA
and $4.09\times$ faster than PatchMerger-SDPA. When $N=M/16$, the gaps become
$52.6\times$ and $108.4\times$, respectively: SDPA avoids a materialized
assignment map but still performs $MN$ content-to-output interactions.
Because the fused backend has the numerical limits documented in
Table~\ref{tab:backend-agreement}, we also state the same comparison for
the always-faithful native implementation
(Table~\ref{tab:backend-ablation}): at $N=M/16$ it takes $289.54$ ms,
i.e.\ $24.7\times$ and $50.9\times$ faster than the two poolers -- the
order-of-magnitude advantage in the growing-output regime does not depend
on the fused path. At fixed $N=1024$ the native core ($274.30$ ms) is
$1.2\times$/$1.7\times$ ahead, and the further $2.47\times$ there is what
fusion buys. The
memory comparison in Table~\ref{tab:matched-poolers} is more modest,
because the optimized poolers also avoid the quadratic activation map. The
large advantage here is primarily in computation, not a claim that SDPA is
quadratic in stored activations.

Both SDPA poolers time out at the next attempted ratio point, $M=2^{21}$;
LapMerge continues through $M=2^{23}$ in that regime. At fixed $N=1024$,
PatchMerger times out at $M=2^{21}$. TokenLearner completes that point, but
its median jumps from $0.336$ s at $M=2^{20}$ to $14.93$ s at $M=2^{21}$,
then the $M=2^{22}$ run fails with PyTorch's
\texttt{value.stride(2) overflows} error. The jump and error are reported as
observed behavior of this SDPA implementation, not attributed to a particular
kernel switch without a separate diagnostic. A timeout is not a measured
latency at the missing point.

\paragraph{Scorer and routing overhead (end-to-end control).}
The main comparison times the reducer core with precomputed token masses.
As a control, Table~\ref{tab:e2e-scorer} measures the full LapMerge module
end to end -- a lightweight scorer (Linear$(512,256)$+GELU+Linear$(256,1)$),
LapSum soft top-$k$ routing, and the merge -- under the same protocol
(GH200, $B{=}1$, $d{=}512$, FP32, forward+backward, five warm-ups, 30
repeats). The scorer grows linearly with $M$ and the routing cost is
dominated by a nearly constant sort overhead, so together they add about
$11\%$ of the fused end-to-end time at $M{=}2^{20}$ and about $7\%$ at
$M{=}2^{23}$; only at small $M$ ($2^{14}$) is the routing's fixed cost
comparable to the merge itself. The reducer-core comparison above therefore
does not hide a dominant upstream cost.

\begin{table}[h]
\centering
\small
\setlength{\tabcolsep}{5pt}
\begin{tabular}{rrrrrrr}
\toprule
$M$ & $N$ & scorer & routing & merge (eager) & merge (fused) & total (fused) \\
\midrule
$2^{14}$ & $1024$   & $0.8$  & $5.7$ & $6.2$   & $6.1$    & $10.4$ \\
$2^{17}$ & $1024$   & $1.3$  & $5.6$ & $33.3$  & $18.5$   & $24.7$ \\
$2^{17}$ & $8192$   & $1.3$  & $5.6$ & $35.8$  & $22.8$   & $26.0$ \\
$2^{20}$ & $1024$   & $7.8$  & $5.3$ & $274.2$ & $111.0$  & $122.5$ \\
$2^{20}$ & $65536$  & $7.8$  & $5.2$ & $291.3$ & $135.6$  & $147.1$ \\
$2^{23}$ & $1024$   & $60.8$ & $7.2$ & --      & $898.3$  & $979.8$ \\
$2^{23}$ & $524288$ & $61.1$ & $7.3$ & --      & $1059.4$ & $1141.2$ \\
\bottomrule
\end{tabular}
\caption{\textbf{End-to-end decomposition of the full LapMerge module}
(median ms, forward+backward, GH200, $B{=}1$, $d{=}512$, FP32). Components
are timed separately; \emph{total} runs the whole
scorer$\to$routing$\to$merge chain with one backward. Dashes: the eager
merge at $M{=}2^{23}$ exceeds memory in this control harness, which keeps
extra per-component buffers; the main benchmark covers that point.}
\label{tab:e2e-scorer}
\end{table}

\begin{table*}[t]
\centering
\caption{\textbf{Historical published-versus-SDPA control on one GH200 (ms).}
All methods use $B=1$, $d=512$, and FP32.  Published PatchMerger and
TokenLearner materialize their $M\times N$ assignment maps.  Their SDPA variants
are mathematically equivalent memory-efficient implementations that do not
materialize these maps.  LapMerge uses its shape-aware dense/fused-CUDA
dispatcher.  Entries are warmed medians over 20 iterations for the competitors
and 100 iterations for LapMerge.}
\label{tab:gh200-merge-comparison}
\small
\setlength{\tabcolsep}{5pt}
\begin{tabular}{rr|rr|rr|r}
\toprule
& & \multicolumn{2}{c|}{PatchMerger} &
\multicolumn{2}{c|}{TokenLearner} & \textbf{LapMerge} \\
$M$ & $N$ & published & SDPA & published & SDPA & fused hybrid \\
\midrule
$65{,}536$    & $1{,}024$ & 120.22 & 28.01 & 118.52 & 20.64 & \textbf{10.86} \\
$65{,}536$    & $2{,}048$ & 123.81 & 37.45 & 123.38 & 22.55 & \textbf{15.10} \\
$262{,}144$   & $1{,}024$ & 475.24 & 110.60 & 473.75 & 80.46 & \textbf{29.51} \\
$262{,}144$   & $2{,}048$ & 493.98 & 151.90 & 490.19 & 89.33 & \textbf{29.64} \\
$1{,}048{,}576$ & $1{,}024$ & 1932.20 & 456.79 & 1918.33 & 322.97 & \textbf{112.61} \\
$1{,}048{,}576$ & $2{,}048$ & 2019.88 & 604.09 & 1999.63 & 363.95 & \textbf{112.04} \\
\bottomrule
\end{tabular}
\end{table*}

\paragraph{Memory-efficient baseline construction.}
The SDPA implementations compute exactly the same normalized pooling as their
published counterparts.  PatchMerger's learned projection columns are the
queries, normalized input tokens are keys, and the original tokens are values.
For TokenLearner, the final linear layer of its scoring MLP provides the
queries and the preceding hidden features provide the keys; its final bias is
constant over input positions and therefore cancels inside the softmax. SDPA
removes the explicit quadratic activation map but retains the
$\mathcal{O}(MNd)$ arithmetic. Each cell in this older control ran in an
isolated process with a 120-second wall-clock limit. No SDPA cell in that
control reached this limit; both published dense implementations ran out of
memory at $(M,N)=(1{,}048{,}576,4{,}096)$. The memory ordering observed in
this older run does not describe the newer fused LapMerge sweep reported in
Table~\ref{tab:matched-poolers}.

Table~\ref{tab:gh200-merge-comparison} is an older, separate control showing
how much the published dense poolers benefit from mathematically equivalent
SDPA rewrites. It used 20 warmed iterations for the poolers and 100 for
LapMerge; we do not compute cross-method speedup factors from that table.
Figure~\ref{fig:runtime-main-scaling} and Table~\ref{tab:matched-poolers}
instead use the matched 5/30 rerun.

\subsection{Scaling and stress tests}
\label{app:runtime-results}

Table~\ref{tab:experiment-time-h100} is a separate, earlier H100 stress test
of feature dimension at $N=0.1M$. It is useful context for the cost of value
aggregation, but it uses another accelerator and timing protocol and is not
spliced into the GH200 curves of Figure~\ref{fig:runtime-main-scaling} or the
matched backend ablation above.

\begin{table}[h!]
\centering
\caption{
Independent H100 ($80$ GB) dimensional stress test of the complete reducer,
forward plus backward, at $N=0.1M$. Times are warmed medians in milliseconds
from a \texttt{torch.compile} run. This is not the matched GH200 protocol;
the final column shows the step ratio at $d=512$.}
\label{tab:experiment-time-h100}
\small
\setlength{\tabcolsep}{7pt}
\begin{tabular}{rrrrl}
\toprule
$M$ & $d=512$ & $d=1024$ & $d=2048$ & step ratio ($d=512$) \\
\midrule
$10^3$ & $4.8$ & $4.7$ & $7.1$ & --- \\
$10^4$ & $5.9$ & $6.7$ & $10.7$ & $1.2\times$ \\
$10^5$ & $22.0$ & $40.4$ & $75.7$ & $3.7\times$ \\
$10^6$ & $222.5$ & $433.5$ & \textsc{oom} & $\mathbf{10.1\times}$ \\
$3\cdot10^6$ & $688.3$ & $\cdot$ & $\cdot$ & $3.1\times$ \\
$10^7$ & \textsc{oom} & $\cdot$ & $\cdot$ & --- \\
\bottomrule
\end{tabular}
\end{table}

The main-paper curve now uses the ordered-coordinate, shape-aware
dense/CUDA-fused LapMerge path. At fixed $N=1024$, it completes both
directions at $M=2^{24}$. At $N=M/16$, it completes through $M=2^{23}$
and is out of memory at the next attempted point. Table~\ref{tab:runtime-long}
gives the measured forward and forward--backward values at the large end
of both sweeps. In particular, we do not infer a forward-only time for
the failed $M=2^{24},N=M/16$ cell.

\begin{table*}[t]
\centering\small
\setlength{\tabcolsep}{7pt}
\begin{tabular}{llrrrr}
\toprule
Regime & $M$ & $N$ & Forward [ms] & Forward + backward [ms] & Peak [GiB] \\
\midrule
$N=1024$ & $2^{22}$ & $2^{10}$ & 169.09 & 439.28 & 10.07 \\
$N=1024$ & $2^{23}$ & $2^{10}$ & 360.85 & 899.09 & 20.13 \\
$N=1024$ & $2^{24}$ & $2^{10}$ & 807.38 & 1878.61 & 40.24 \\
$N=M/16$ & $2^{22}$ & $2^{18}$ & 150.64 & 531.17 & 14.57 \\
$N=M/16$ & $2^{23}$ & $2^{19}$ & 299.83 & 1067.26 & 29.14 \\
$N=M/16$ & $2^{24}$ & $2^{20}$ & --- & OOM & --- \\
\bottomrule
\end{tabular}
\caption{Large-shape GH200 probes for the CUDA-fused LapMerge dispatcher,
with ordered inputs, $B=1$, $d=512$, and FP32. The failed ratio cell has
no reported latency or allocation; no forward-only measurement is inferred.
Peak allocation is incremental above live inputs.}
\label{tab:runtime-long}
\end{table*}

\subsection{Measurement protocol and interpretation}
\label{app:runtime-protocol}

The new GH200 backend sweeps use five warm-ups and 30 synchronized repeats
per shape, including a mean-square loss and backward pass. Each point runs
in a fresh process with a 600-second whole-cell limit. First-call/JIT time
is logged separately and excluded from steady-state medians; timing IQRs,
incremental peak CUDA allocation, and failure statuses are retained in the
CSV files. The fixed-output and fixed-ratio sweeps use identical inputs
and the same protocol across the four LapMerge implementations. Allocation
is measured relative to already live tensors, so it is not total device
occupancy.

The earlier H100 dimension sweep and the published-versus-SDPA table use
different iteration counts and are intentionally separate from the matched
GH200 figures. The new competitor sweep shares the 5/30 timing settings and
600-second whole-cell limit with the LapMerge sweeps. A timeout includes
warm-up and all repeated measurements; it is not an observed steady-state
latency at the missing point. Open markers in Figure
\ref{fig:runtime-main-scaling} sit on the last successful point before a
timeout; filled markers precede OOM or an execution failure. Neither marker
should be read as a measurement of the failed cell.

\section{When Does Merging Beat Selection?}
\label{sec:vs-select}

This appendix isolates the scientific question that is distinct from runtime:
when is it better to preserve information by merging tokens rather than to
discard all but the highest-scoring ones? We use matched output budgets on
Stanford Cars and arXiv. The two settings deliberately differ in how strongly
task-relevant information follows sequence order: spatially adjacent patches on
Cars versus distributed evidence in long documents. Runtime and implementation
choices are reported separately in Appendix~\ref{app:runtime-comparison}; the
tables below therefore focus on downstream accuracy and on the structure of
the reducer. The section closes with a direct order intervention that
isolates the mechanism (\S\ref{sub:order-intervention}) and with the list of
planned confirmation experiments
(\S\ref{sub:selection-planned-experiments}).

\subsection{Matched evaluation protocol}
\label{sub:selection-merging-protocol}

All rows use the same upstream representation, head, optimizer, and output
budget within a dataset. ``Soft'' means that the differentiable reducer is
present during evaluation. ``Hard'' removes the reducer and retains the
$N$ highest-scoring raw tokens, which is a useful control for distinguishing
the scorer from the aggregation operator. A one-sweep model uses one ordered
LapMerge pass and returns $N$ tokens. The two-sweep model uses two independent
passes of $10$ tokens each, followed by concatenation; its total output budget
is therefore also $20$, but its reducer cost is approximately doubled. This
distinction is made explicit in the Cars table rather than hidden in the
parameter count.

The Cars numbers below are single-run screening results from the established
training protocol; the planned multi-seed confirmation is described after the
tables. The arXiv comparison is budget-matched: every arm is trained at the
budget at which it is evaluated. This is intentionally conservative for
cross-budget generalization, but it means that the arXiv table should be read
as a comparison of operating points, not as evidence that one shared model
dominates all budgets.

\subsection{Spatially ordered tokens: Stanford Cars}
\label{sub:selection-cars}

Features are the $196$ patch tokens of a frozen ViT-S/16 encoder
(\texttt{timm}'s \texttt{vit\_small\_patch16\_224}, ImageNet-pretrained) at
$224^2$ resolution; the reducer and classification head are trained on top.
On Stanford Cars (Table~\ref{tab:cars-reducer-baselines}), \our{} improves over
top-$N$ selection at the same $20$-token budget. A single sweep reaches
$75.46\%$ versus $74.10\%$ for selection, while two complementary sweeps
reach $76.35\%$, within $0.36$ points of the $76.71\%$ dense baseline. The
learned poolers perform worse: TokenLearner reaches $73.60\%$ and PatchMerger
$71.83\%$. Thus, on this spatially structured task, retaining information
through merging is more effective than discarding tokens, while the particular
structure of \our{} provides an additional advantage over generic pooling.
Removing the operator at inference by retaining the highest-scoring raw
tokens reduces accuracy to $67.88\%$. We read this as a \emph{head-dependence
test}: the classifier was trained on merged representations, so the drop
measures how strongly the trained head relies on them -- it is not the
accuracy of an equivalently trained hard-selection model, which reaches
$74.10\%$.

The Cars result supports the ordered-merge hypothesis: neighboring visual
patches often contain complementary evidence, so replacing a group by its
mass-preserving representative can retain information that hard selection
throws away. The two-sweep result should not be interpreted as a free increase
in capacity, however, since it uses two reducer passes. The relevant control
is consequently the one-sweep row at the same output budget.

\paragraph{Budget dose-response.}
Repeating the one-sweep-versus-selection comparison at
$N\in\{10,20,40\}$ with three seeds each (Table~\ref{tab:cars-budgets})
shows that the merge advantage is a tight-budget phenomenon: $+3.4$
points at $N{=}10$, $+0.9$ at the capacity-matched $N{=}20$, and $-0.8$
at $N{=}40$, where selection alone already retains enough evidence (at
$77.3$ it even edges past the $196$-token dense ceiling of $76.71$).
This mirrors the POPE pattern -- decisive at the $10\%$ budget, within
noise at $30\%$ (\S\ref{app:pope-sweep}): merging pays where compression
is aggressive. The operator-free column shows the same trend from the
head's side: its reliance on the merge fades as the budget grows
(soft$-$hard gap $15.4\to7.6\to3.3$).

\begin{table}[h]
\centering
\small
\setlength{\tabcolsep}{6pt}
\begin{tabular}{lccc}
\toprule
 & \our{} one sweep (soft) & operator-free (hard) & Top-$N$ selection \\
\midrule
$N{=}10$ & $\mathbf{74.91{\pm}0.60}$ & $59.51{\pm}0.66$ & $71.47{\pm}0.31$ \\
$N{=}20^{\dagger}$ & $\mathbf{75.10{\pm}0.28}$ & $67.53{\pm}0.26$ & $74.20{\pm}0.44$ \\
$N{=}40$ & $76.52{\pm}0.40$ & $73.25{\pm}0.27$ & $\mathbf{77.28{\pm}0.13}$ \\
\bottomrule
\end{tabular}
\caption{\textbf{Budget dose-response on Stanford Cars} (top-1, mean
$\pm$sd over three shared seeds, original protocol). $^{\dagger}$The
$N{=}20$ row is the capacity-matched frozen-frame configuration of
Table~\ref{tab:cars-reducer-baselines}; the full one-sweep single run
reaches $75.46$ there.}
\label{tab:cars-budgets}
\end{table}

Figure~\ref{fig:cars-mass} visualizes the routing mass learned by the
reducer on two test images: mass concentrates on discriminative car
parts, while background patches receive near-zero mass, mirroring the
attribution behaviour measured quantitatively in histopathology
(Appendix~\ref{app:data-histo}).

\begin{figure}[t]
\centering
\includegraphics[width=0.49\linewidth]{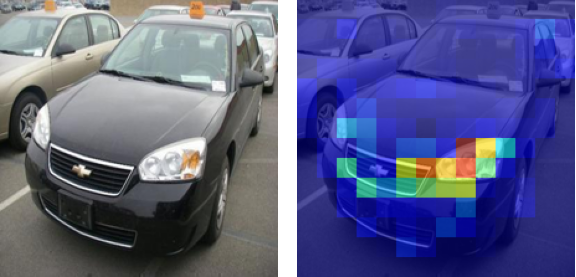}\hfill
\includegraphics[width=0.49\linewidth]{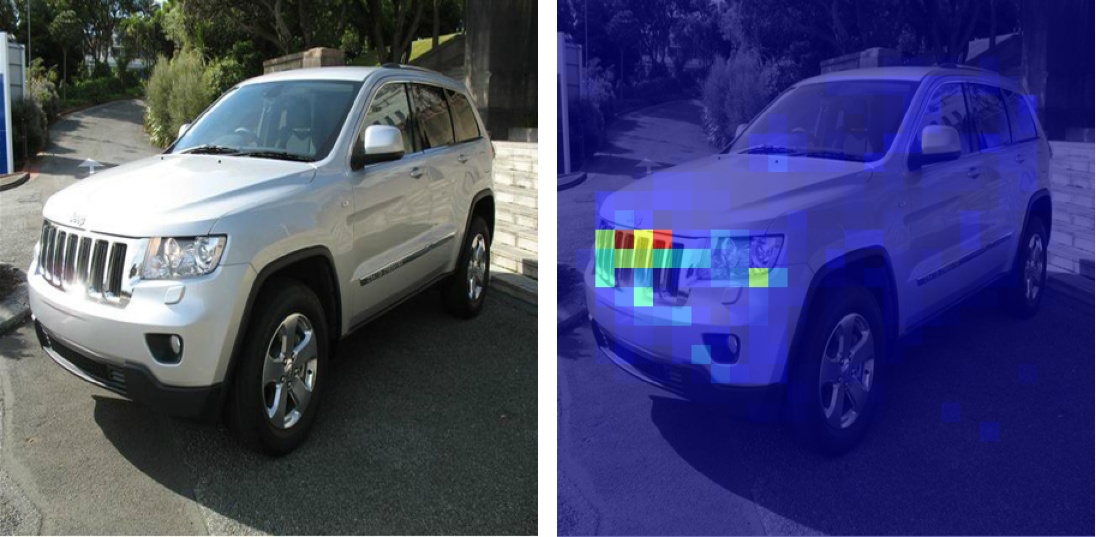}
\caption{\textbf{Learned routing mass on Stanford Cars.} Each pair shows
the input image (left) and the per-token routing mass $r_i$ of the trained
reducer as a heatmap (right; normalized per image). Mass concentrates on
discriminative car parts -- grille, badge, and headlight region -- while
background patches receive near-zero mass. Left pair: the standard
$196$-token single-crop configuration of
Table~\ref{tab:cars-reducer-baselines}; right pair: a $2\times2$
multi-crop variant ($784$ tokens) showing the same behaviour on a finer
grid.}
\label{fig:cars-mass}
\end{figure}

\subsection{Distributed evidence and order ablations}
\label{sub:selection-arxiv}

\paragraph{Protocol.} The arXiv task is 11-way topic classification on the
public \texttt{ccdv/arxiv-classification} dataset with its released
train/test split. Features come from a frozen
\texttt{allenai/longformer-base-4096} checkpoint with its own tokenizer,
truncated to $M{=}1024$ tokens; the classifier is a linear head on the
reduced sequence. Every arm trains for $30$ epochs with AdamW
(lr $3\times10^{-3}$, weight decay $0.2$, batch $64$) at its own output
budget $N\in\{8,16,32,64\}$.

The picture changes on arXiv (Table~\ref{tab:vs-poolers-arxiv}). \our{} exceeds
top-$N$ selection by $5.6$--$6.2$ points across all evaluated budgets, but does
not consistently outperform learned poolers. Across eight budget-by-method
comparisons, the poolers lead in five, \our{} leads in two, and one comparison
is tied; the largest gap in either direction is $1.0$ point. These are
single-seed runs, and the only measured spread ($\sigma\approx0.46$, from three
independent trainings of one arm at one budget) does not estimate the
variance of differences between other arms, so we attach no significance claim.
We state plainly that \our{} shows no consistent accuracy advantage over
learned pooling in this setting; it provides the same accuracy range while
retaining its $\mathcal{O}((M{+}N)d)$ cost.

Together, the results separate two effects of merging. Aggregating tokens
instead of discarding them is particularly valuable when relevant information
is distributed across the sequence, while the additional benefit of
\our{}'s ordered, equal-mass aggregation appears when sequence order itself
carries task structure. On Cars, this yields a $1.86$-point improvement over
the strongest learned pooler; on arXiv, the corresponding advantage
disappears. A train-budget/test-budget robustness matrix is left as a future
control; since such an evaluation runs the model away from its training
budget, we would treat it as a robustness analysis rather than as an
additional matched merging-versus-selection comparison.

\begin{table}[t]
\centering
\small
\setlength{\tabcolsep}{4pt}
\begin{tabular}{lccccc}
\toprule
Reducer  & Tokens & Params & Reducer cost & Soft & Hard \\
                                 & to head &        &              & (\%) & (\%) \\
\midrule
Dense transformer, no reduction  & $196$ & --- & none & $76.71$ & --- \\
\midrule
Top-$N$ selection                & $20$ & $52$\,k & $\mathcal{O}(Md)$ + top-$N$ & $74.10$ & $\mathbf{74.10}^{\dagger}$ \\
TokenLearner                     & $20$ & $53$\,k & $\mathcal{O}(MNd)$ & $73.60$ & n/a \\
PatchMerger                      & $20$ & $8$\,k  & $\mathcal{O}(MNd)$ & $71.83$ & n/a \\
ToMe post-encoder merge          & $20$ & $0$     & $\mathcal{O}(M^2d)$ & $69.87$ & n/a \\
\midrule
\rowcolor{orange!10}
\our{}, two sweeps               & $2\times10$ & $0.20$\,M & $2\mathcal{O}((M{+}N)d)$ & $\mathbf{76.35}$ & $36.02$ \\
\rowcolor{orange!10}
\our{}, one sweep                & $20$ & $0.20$\,M & $\mathcal{O}((M{+}N)d)$ & $75.46$ & $67.88$ \\
\midrule
\rowcolor{orange!10}
\our{}, one sweep, frozen frame  & $20$ & $52$\,k & $\mathcal{O}((M{+}N)d)$ & $75.10{\pm}0.28$ & $67.53{\pm}0.26$ \\
Top-$N$ selection (same seeds)   & $20$ & $52$\,k & $\mathcal{O}(Md)$ + top-$N$ & $74.20{\pm}0.44$ & $74.20{\pm}0.44$ \\
\bottomrule
\end{tabular}
\caption{\textbf{Matched pre-head reducers on Stanford Cars.} ``Reducer
params'' excludes the $2.44$\,M shared by every row; ``reducer cost'' is
analytic. \emph{Soft} runs the reducer as trained; \emph{hard} keeps the $N$
highest-scoring raw tokens and removes it entirely (a head-dependence test,
not a hard-trained model).
$^{\dagger}$Selection coincides on both by construction; ``n/a'' marks
poolers, which expose no operator-free path. The ToMe row reaches $20$ tokens
by iterating bipartite matching roughly four times with no transformer layer
in between, well outside its published schedule; read it as training-free
similarity merging pushed out of its design regime. The last block is the
capacity-matched control (mean$\pm$sd over three shared seeds): the read-out
frame $U$ is frozen at identity, so the reducer matches selection's
$\approx52$k parameters while sharing its scorer -- the merge advantage
survives equalized capacity. Protocol in
Appendix~\ref{sub:selection-merging-protocol}.}
\label{tab:cars-reducer-baselines}
\end{table}

\begin{table}[t]
\centering
\small
\setlength{\tabcolsep}{6pt}
\begin{tabular}{lccccc}
\toprule
Reducer (arXiv, frozen Longformer)         & Reducer & \multicolumn{4}{c}{Output budget $N$ for $M=1024$} \\
\cmidrule(lr){3-6}
\emph{every arm trained at its own budget} & params  & $8$ & $16$ & $32$ & $64$ \\
\midrule
Dense (no reduction)                       & ---       & \multicolumn{4}{c}{$87.6$} \\
\midrule
TokenLearner \citep{ryoo2021tokenlearner}  & $102$\,k  & $\mathbf{88.28}$ & $87.88$ & $\mathbf{88.12}$ & $\mathbf{87.68}$ \\
PatchMerger \citep{renggli2022patchmerger} & $14$\,k   & $88.00$ & $\mathbf{88.28}$ & $87.08$ & $87.48$ \\
Top-$N$ selection                          & $100$\,k  & $81.48$ & $81.44$ & $82.12$ & $81.80$ \\
\midrule
\rowcolor{orange!10}
\our{} (scan coordinate)                   & $0.75$\,M & $87.28$ & $87.68$ & $87.72$ & $\mathbf{87.68}$ \\
\rowcolor{orange!10}
\our{} (published coordinate)              & $0.69$\,M & $86.24$ & $86.80$ & $87.24$ & $86.48$ \\
\bottomrule
\end{tabular}
\caption{\textbf{Budget-matched training on arXiv} (single seed, soft deploy):
every entry is trained at the budget it is evaluated at, which removes the
extrapolation penalty the poolers would otherwise pay and is conservative for
us. The scan coordinate is the parameterization of \S\ref{sec:method} and is
what the text calls \our{}; the published-coordinate row is an ablation.
For scale only: three independent trainings of the
published-coordinate arm give $86.24/85.84/86.76$ at $N{=}8$, i.e.\
$\sigma\approx0.46$ for that one arm; all other entries are single runs.
Against \our{} the poolers lead in five of eight cells, trail in two and tie
in one, with a largest gap of $1.0$ point.}
\label{tab:vs-poolers-arxiv}
\end{table}

\subsection{Order intervention: permuting the scan on LRA-Image}
\label{sub:order-intervention}

\paragraph{The order is the mechanism, and permuting it is the test.}
The comparison so far rests on a contrast \emph{between} datasets -- \our{}
wins on a patch grid (\S\ref{sub:selection-cars}) and does not on a topic
label (\S\ref{sub:selection-arxiv}) -- which changes modality, backbone,
length and task at once and then credits one of them. Here it is an
intervention instead. We permute only the scan coordinate the barriers
walk, applied after the embedding so that every token keeps its content
and its positional encoding, and hold the rest of the recipe fixed; the
sole thing that changes is which tokens share a bin. (This ablation uses
its own shorter training protocol, so absolute numbers are not directly
comparable to Table~\ref{tab:lra_10pct}.) Over three seeds,
three locality-preserving orders stay within $3.5$ points of one another
-- raster $48.0{\pm}1.3$, column-major $45.4{\pm}0.1$, a Hilbert curve
$48.9{\pm}0.3$ -- so nothing here is special about the raster axis
specifically. A random permutation, by contrast -- one draw, fixed for
the whole run, so the Transformer is free to learn the relabelling --
costs $15.5$ points ($\to32.5{\pm}0.9$) and lands the merge \emph{inside}
the order-free band it otherwise beats by ten.

Three controls hold that
reading up, built on different mechanisms so that a leak in the harness
would have to fool all of them: TokenLearner, PatchMerger and learned
top-$N$ selection each return their own number under the permutation,
unchanged to two decimals in both mean and spread across the same three
seeds (Table~\ref{tab:order}). So the permutation reaches only reducers
that walk the scan -- not the data, the positions, or anything else. What
it reaches, it reaches hard, and not only \our{}: uniform pooling, whose
segments are equally positional, loses $15.7$. ToMe does not, at $+0.2$,
because its bipartite matching groups by content and the scan order it is
handed is arbitrary to begin with. The intervention therefore isolates
something narrower, and more useful, than ``the merge needs an order'':
what the order buys is the advantage a positional partition holds over
content-based grouping, and taking it away takes away exactly that.

\begin{table}[h]
\centering
\small
\setlength{\tabcolsep}{5pt}
\begin{tabular}{lrrrrr}
\toprule
reducer & raster & col-major & Hilbert & rnd (fixed) & rnd (per img) \\
\midrule
\our{} & $48.01{\pm}1.34$ & $45.42{\pm}0.14$ & $48.96{\pm}0.26$ & $32.48{\pm}0.89$ & $35.36{\pm}0.80$ \\
uniform pooling & $45.15{\pm}0.27$ & $43.67{\pm}0.45$ & $46.71{\pm}1.14$ & $29.46{\pm}0.82$ & $28.11{\pm}0.28$ \\
ToMe & $37.16{\pm}0.24$ & $37.33{\pm}0.45$ & $37.67{\pm}0.16$ & $37.35{\pm}0.33$ & $42.85{\pm}0.31$ \\
\midrule
TokenLearner & $37.08{\pm}0.34$ & --- & --- & $37.08{\pm}0.34$ & --- \\
PatchMerger & $26.94{\pm}0.34$ & --- & --- & $26.94{\pm}0.34$ & --- \\
learned top-$N$ selection & $37.20{\pm}0.88$ & --- & --- & $37.20{\pm}0.87$ & --- \\
\bottomrule
\end{tabular}
\caption{\textbf{Permuting the scan, with everything else held fixed.}
LRA-Image at a $10\%$ budget, mean${\pm}$sd over three seeds, sharp width.
Down the top block: a locality-preserving order is a locality-preserving
order -- raster, its transpose and a Hilbert curve agree to within $3.5$
points -- while a random permutation costs the merge $15.5$, dropping it
\emph{into} the band of the reducers below it. Across the bottom block:
those reducers cannot see the scan coordinate, and they return their own
numbers to two decimals in both mean and spread, which rules out the
permutation leaking through the data, the positions, or the harness. A
per-image permutation is gentler than a fixed one for the merge ($-12.7$
against $-15.5$) and actually helps ToMe ($+5.7$): resampling the order
every step acts as augmentation, which is why the fixed draw -- one the
Transformer is free to learn -- is the cleaner intervention.}
\label{tab:order}
\end{table}

Two dose-response sweeps (Table~\ref{tab:order-dose}) say what kind of
dependence it is. The gap is flat across a tenfold range of budgets, $14$
to $16$ points from $5\%$ to $50\%$, so a looser budget does not dilute
it. And it shrinks as the kernel widens past the bin spacing, from $15.0$
at $\alpha{=}1$ to $11.2$ at $\alpha{=}M/N$ and $3.9$ at
$\alpha{\approx}M$, where the operator is close to order-free mean
pooling and the raster arm has fallen into the order-free band -- from
the raster side alone, the way a positional mechanism should, since a bin
that blurs across its own boundaries partitions less.

\begin{table}[h]
\centering
\small
\setlength{\tabcolsep}{6pt}
\begin{tabular}{llccc}
\toprule
sweep & value & raster & permuted (fixed) & gap \\
\midrule
width $\alpha$ (budget $10\%$) & $0.5$ & $48.0{\pm}1.2$ & $33.2{\pm}0.7$ & $14.8$ \\
 & $1$ & $48.0{\pm}1.1$ & $33.0{\pm}0.9$ & $15.0$ \\
 & $2$ & $47.2{\pm}0.5$ & $32.3{\pm}0.9$ & $14.9$ \\
 & $5$ & $44.4{\pm}0.1$ & $31.5{\pm}1.0$ & $12.9$ \\
 & $10=M/N$ & $42.6{\pm}0.2$ & $31.4{\pm}1.2$ & $11.2$ \\
 & $30$ & $41.5{\pm}0.2$ & $31.4{\pm}0.8$ & $10.2$ \\
 & $100$ & $40.5{\pm}0.1$ & $31.9{\pm}0.7$ & $8.6$ \\
 & $1000\approx M$ & $36.5{\pm}0.2$ & $32.7{\pm}0.9$ & $3.9$ \\
\midrule
budget ($\alpha=1$) & $5\%$ & $45.2{\pm}0.3$ & $31.0{\pm}0.7$ & $14.2$ \\
 & $10\%$ & $48.8{\pm}0.8$ & $32.8{\pm}0.6$ & $16.0$ \\
 & $20\%$ & $49.1{\pm}1.0$ & $33.2{\pm}1.1$ & $15.9$ \\
 & $50\%$ & $48.4{\pm}0.9$ & $33.3{\pm}0.5$ & $15.1$ \\
\bottomrule
\end{tabular}
\caption{\textbf{How the order effect responds to width and to budget.}
LRA-Image, \our{}, mean${\pm}$sd over three seeds; every row is the raster
scan against one random permutation of the scan coordinate, fixed for the
run, with everything else as in Table~\ref{tab:order}; the gap is the
difference of means. Across the width it is not flat: sharp kernels
($\alpha\le2$) hold the full gap, which shrinks as $\alpha$ grows toward
$M$, where the raster arm itself falls to the order-free level (the
permuted arm stays between $31$ and $33$ throughout) -- exactly what a
positional mechanism should do. The $10\%,\alpha{=}1$ configuration
appears in both blocks with the same three seeds and returns
$48.0{\pm}1.1$ and $48.8{\pm}0.8$ (permuted $33.0$ and $32.8$): the
difference is run-to-run nondeterminism of the GPU scatter, about the
size of the seed spread, so differences under a point between rows are
not to be read.}
\label{tab:order-dose}
\end{table}

\paragraph{One knob matters, and its natural default is the wrong one.}
Everything else we varied moved the result by less than seed noise; the
bin width did not. Over three seeds, $\alpha{=}0.5$ and $\alpha{=}1$ tie
at $48.0$, $\alpha{=}2$ gives $47.2$, $\alpha{=}5$ gives $44.4$, and
setting $\alpha$ to the mean bin spacing $M/N\approx10$ -- the value the
construction seems to invite, since it makes the kernel as wide as a bin
-- costs five to six points at $42.6{\pm}0.2$, and a kernel as wide as the
sequence ($\alpha\approx M$) another six, to $36.5$
(Table~\ref{tab:order-dose}). Wide bins blur across their own boundaries
and the partition stops being one; sharper than $1$ buys nothing, so the
sharp setting is a plateau, not an edge.

\paragraph{The speed comes from the scan, not from computing something
weaker.} A \emph{dense} implementation of the same operator -- same
barriers, same read-out, coupling materialized -- reaches $48.85\%$ at
$90.1$\,s/epoch against $48.90\%$ at $12.9$\,s/epoch for the matrix-free
form (both single seed, same job): identical accuracy, seven times the
speed. So the linear-time machinery of \S\ref{sec:exp-scaling} is
what makes the reduced model cheaper than the unreduced one, and it shows
up here on a downstream task, not only in a microbenchmark.

\subsection{Planned confirmation experiments}
\label{sub:selection-planned-experiments}

One control remains planned, and two have been carried out. First, the
one-sweep-versus-selection comparison has been repeated at three output
budgets with three seeds each (Table~\ref{tab:cars-budgets}); extending
the multi-seed protocol to the two-sweep and learned-pooling rows remains
future work. Second --- still planned ---
the order-permutation control of \S\ref{sub:order-intervention} -- carried
out on LRA-Image -- will be extended to the Cars and arXiv studies at their
own budgets, with the permutation applied before both scoring and merging
while the backbone, optimization schedule, and output budget remain
unchanged. A third control, equalizing reducer capacity, has been carried
out (Table~\ref{tab:cars-reducer-baselines}, last block): with the read-out
frame frozen at identity the \our{} reducer matches selection's
$\approx52$k parameters and still leads, $75.10\pm0.28$ against
$74.20\pm0.44$ over three shared seeds.

For completeness, the implementation ablation is reported separately in
Appendix~\ref{app:runtime-backend-ablation}: it compares eager PyTorch,
\texttt{torch.compile}, and the CUDA-fused hybrid, and includes forced-dense,
forced-scan, and automatic dispatch. Thus the accuracy tables here answer
whether merging helps, while the runtime appendix answers how the same
operator is evaluated efficiently.

\section{Drop-in Token Reduction in a Frozen VLM: Details}
\label{app:pope}
 
This appendix complements Section~\ref{sec:vlm-pope}. We describe the
pipeline into which score-routed \our{} is inserted, state its general two-width form with properties and cost (Appendix~\ref{app:pope-variant}), and report the sweep over both widths $(\alpha_a,\alpha_r)$, with the released selection pipeline rerun under identical conditions. The POPE benchmark itself is described in Appendix~\ref{app:data-pope}.
 
\subsection{Pipeline}
\label{app:pope-pipeline}
 
\paragraph{Backbone.}
We use Qwen2.5-VL-7B-Instruct \citep{qwen25vl} without any fine-tuning. Qwen2.5-VL encodes an image at its native resolution with a ViT, and a $2\times2$ patch merger then groups neighbouring patches, so the number of visual tokens $M$ passed to the language model depends
on the input resolution. Visual tokens carry multimodal rotary position embeddings (M-RoPE) indexed by their (temporal, height, width) coordinates.
 
\paragraph{VisionSelector.}
We take the publicly released VisionSelector pipeline
\citep{visionselector2025} as-is. It attaches a lightweight learned scorer to the frozen VLM, which assigns an importance score $s_i$ to every visual token $x_i$, $i=1,\dots,M$. The scorer is trained once by the VisionSelector authors and is not retrained by us.
At a budget $b\in\{10\%,30\%\}$, VisionSelector keeps the
$k=\lceil bM\rceil$ highest-scoring tokens, discards the rest, and feeds the kept tokens to the language model at their original positions.

\subsection{Score-routed \our{}: general form, properties, and cost}
\label{app:pope-variant}

Section~\ref{sec:method} presents \our{} with a single kernel width. The
POPE deployment uses the natural two-width generalization, which we state
here explicitly. Let $r_i\ge0$ be the routing masses (normalized to sum to
one; the soft top-$k$ output sums to $k$, which only rescales them), placed
at the raster indices $p_i=i$, and let $\Phi^{\alpha}$ denote the mixture
CDF~\eqref{eq:mix} at kernel width $\alpha$. With $N{=}k$ outputs, the
barriers solve the equal-mass equations at the \emph{routing} width,
$\Phi^{\alpha_r}(b_\ell)=\ell/N$ ($1\le\ell<N$, $b_0=-\infty$,
$b_N=+\infty$), and each interval is read out at the \emph{read-out}
width $\alpha_a$ with its own normalization:
\begin{align}
w_i^\ell&=\frac{r_i\big[\sigma\big(\tfrac{b_\ell-p_i}{\alpha_a}\big)
-\sigma\big(\tfrac{b_{\ell-1}-p_i}{\alpha_a}\big)\big]}
{\Phi^{\alpha_a}(b_\ell)-\Phi^{\alpha_a}(b_{\ell-1})},
\label{eq:score-routed}\\
t_\ell&=\sum_{i=1}^{M}w_i^\ell x_i
=\frac{\Phi^{\alpha_a}_{r\odot x}(b_\ell)-\Phi^{\alpha_a}_{r\odot x}(b_{\ell-1})}
{\Phi^{\alpha_a}(b_\ell)-\Phi^{\alpha_a}(b_{\ell-1})}.\nonumber
\end{align}
The explicit denominator is needed because, for $\alpha_a\neq\alpha_r$,
the read-out kernel does not assign mass $1/N$ to the intervals cut at the
routing width, so the standard formula $N[\Phi_x(b_\ell)-\Phi_x(b_{\ell-1})]$
would no longer be a weighted average.

\paragraph{Properties.} \emph{(i) Exact budget}: exactly $N$ outputs for
any widths. \emph{(ii) Convexity}: $w_i^\ell\ge0$ and $\sum_iw_i^\ell=1$
by construction, so every $t_\ell$ is a convex combination of the inputs.
\emph{(iii) Reduction to the standard operator}: for $\alpha_a=\alpha_r$
the denominator equals $1/N$ for every interval and \eqref{eq:score-routed}
coincides with~\eqref{eq:membership}; in particular the default
$(\alpha_r,\alpha_a)=(1,1)$ used for Table~\ref{tab:vlm-field}
\emph{is} the standard operator. \emph{(iv) Selection limit}: as both
widths vanish, the soft top-$k$ masses harden to the top-$k$ indicator,
each interval contains exactly one kept token, and $t_\ell$ equals that
token. \emph{(v)} The first-moment identity
$N^{-1}\sum_\ell t_\ell=\sum_ir_ix_i$ holds only for $\alpha_a=\alpha_r$;
for other widths the outputs are normalized per interval rather than
mass-conserving.

\paragraph{Cost.} The barriers are the standard equal-mass solve at width
$\alpha_r$ (Algorithm~\ref{alg:barriers}, $\mathcal{O}(M+N)$). Numerator
and denominator of~\eqref{eq:score-routed} are Laplace-CDF scans at width
$\alpha_a$ with weights $r_ix_i$ and $r_i$, queried at the already ordered
barriers -- the construction of Algorithm~\ref{alg:aggregate} with a second
set of prefix/suffix states. The variant therefore has the same
$\mathcal{O}((M+N)d)$ forward and backward cost as the standard operator,
up to a constant factor of about two. The POPE adapter used for our
experiments evaluates~\eqref{eq:score-routed} with a dense reference
implementation (a per-quantile barrier solve and an interval-by-token CDF
tensor, chunked over feature coordinates); with a few hundred to about two
thousand visual tokens per image this cost is negligible next to the 7B
language model, and the scan path is a drop-in replacement. The runtime
study of Section~\ref{sec:exp-scaling} concerns the standard operator.

\subsection{Sensitivity to the routing and read-out widths}
\label{app:pope-sweep}

Table~\ref{tab:vlm-width} sweeps both widths of score-routed \our{} in
the same shared harness and the same F1 metric as
Table~\ref{tab:vlm-field}: the read-out width $\alpha_a$ at fixed
$\alpha_r{=}1$, and the routing width $\alpha_r$ at fixed
$\alpha_a{=}1$, together with the released selection pipeline rerun
under identical conditions (its published F1 is reproduced to within
$0.06$ at both budgets). The POPE-tuned row of
Table~\ref{tab:vlm-field} is the per-budget maximum of this grid and
should be read as a reference ceiling rather than a fair comparison.

Three observations follow. \emph{(i)} Merging beats hard selection
across the moderate-width range: at the $10\%$ budget every
$\alpha_a\in\{0.3,0.7,1.0,1.5\}$ and both routing widths improve over
selection, and only the over-wide $\alpha_a{=}3.0$ falls marginally
below it ($79.56$ vs $79.67$); at $30\%$ the improvement persists for
$\alpha_a\le1.0$ and disappears beyond. The gain therefore does not
hinge on a finely tuned value, but the kernel must not be made
arbitrarily wide. \emph{(ii)} The optimal width shrinks as the budget
grows: at $10\%$ each output stands in for about nine discarded tokens,
so a wider read-out recovers more of the lost evidence; at $30\%$ the
kept positions are dense and a narrow window suffices. \emph{(iii)} The
routing width is a flat knob: halving or doubling it changes F1 by at
most $0.6$.

\paragraph{Mechanism control: per-anchor smoothing (Recipe-B).}
As a control that separates equal-mass routing from mere local smoothing,
we also rerun a per-anchor alternative in the same harness and metric:
each kept position aggregates its neighbourhood under a softmax of width
$\tau$ over tokens (Table~\ref{tab:vlm-recipeb}). At $\tau{=}0.05$ it
reproduces the selection pipeline exactly ($79.67$/$85.11$, the designed
$\tau\!\to\!0$ limit), and its best setting reaches $80.70$/$85.32$ --
below score-routed \our{} at the \emph{fixed} default
($81.31$/$85.50$), let alone its ceiling. Equal-mass barrier routing
therefore contributes beyond what local smoothing around the kept
positions can provide.

\begin{table}[h]
\centering
\small
\setlength{\tabcolsep}{6pt}
\begin{tabular}{lccccc}
\toprule
 & \multicolumn{5}{c}{per-anchor softmax width $\tau$} \\
\cmidrule(lr){2-6}
budget & $0.05$ & $0.3$ & $0.7$ & $1.5$ & $3.0$ \\
\midrule
$10\%$ & $79.67$ & $79.75$ & $80.70$ & $80.58$ & $78.61$ \\
$30\%$ & $85.11$ & $85.32$ & $85.26$ & $84.41$ & $80.92$ \\
\bottomrule
\end{tabular}
\caption{\textbf{Per-anchor smoothing control (Recipe-B)}: POPE F1 in the
shared harness. $\tau{=}0.05$ recovers hard selection exactly; the best
column trails the default score-routed \our{} of
Table~\ref{tab:vlm-width} at both budgets.}
\label{tab:vlm-recipeb}
\end{table}

\paragraph{Uncertainty of the headline differences.}
Decoding is greedy and both arms answer the same $9{,}000$ questions, so
the difference to selection is paired. A nonparametric bootstrap over
questions ($10{,}000$ resamples, percentile intervals) against our rerun
of the released selection gives $\Delta\mathrm{F1}=+1.64$
($95\%$ CI $[+0.88,+2.40]$, $P(\Delta>0)>0.999$) at the $10\%$ budget and
$\Delta\mathrm{F1}=+0.39$ ($[-0.08,+0.87]$, $P(\Delta>0)=0.94$) at
$30\%$: decisive at the aggressive budget, positive but not separated
from zero at the mild one.

\begin{table}[t]
\centering
\small
\caption{\textbf{Width sweep of score-routed \our{}} (POPE F1, frozen
Qwen2.5-VL-7B; the same shared harness and metric as
Table~\ref{tab:vlm-field}). ``select'' reruns the released selection
pipeline under identical conditions and reproduces its published F1
($79.73$/$85.05$) to within $0.06$. The starred column is the
dataset-agnostic default used in Table~\ref{tab:vlm-field}. Decoding is
greedy, so each cell is deterministic.}
\label{tab:vlm-width}
\setlength{\tabcolsep}{5pt}
\begin{tabular}{lcccccccc}
\toprule
 & & \multicolumn{5}{c}{read-out width $\alpha_a$ ($\alpha_r{=}1$)} & \multicolumn{2}{c}{routing width $\alpha_r$ ($\alpha_a{=}1$)} \\
\cmidrule(lr){3-7}\cmidrule(lr){8-9}
budget & select & $0.3$ & $0.7$ & $1.0^{\star}$ & $1.5$ & $3.0$ & $0.5$ & $2.0$ \\
\midrule
\rowcolor{orange!10}
$10\%$ & $79.67$ & $\mathbf{82.01}$ & $81.67$ & $81.31$ & $80.69$ & $79.56$ & $80.97$ & $80.70$ \\
\rowcolor{orange!10}
$30\%$ & $85.11$ & $\mathbf{85.87}$ & $\mathbf{85.87}$ & $85.50$ & $84.73$ & $82.92$ & $85.19$ & $84.97$ \\
\bottomrule
\end{tabular}
\end{table}

\section{Datasets}
\label{app:datasets}
 
We evaluate \our{} on five groups of benchmarks that stress different properties of token merging: (i) gigapixel whole-slide images, where the input is a bag of thousands of tiles and most of the tissue is uninformative for the label; (ii) the Long Range Arena, which probes long-context reasoning on synthetic, textual and image-as-sequence
inputs; (iii) POPE, which measures whether a vision-language model still ``sees'' the objects in an image after its visual tokens are compressed; (iv) fine-grained image classification, where the discriminative evidence is concentrated in small object parts; and (v) long-document topic classification on arXiv, where relevant evidence is spread through the whole text. Table~\ref{tab:datasets}
summarises the statistics. Each dataset is described in detail below.

\begin{table}[h]
  \centering
  \small
  \caption{Summary of the datasets used in our experiments.}
  \label{tab:datasets}
  \begin{tabular}{llrrr}
    \toprule
    Domain & Dataset & \#Classes & Train & Test  \\
    \midrule
    \multirow{2}{*}{Histopathology}
      & TCGA-CRC (COAD+READ) & 2 & \multicolumn{2}{c}{365 slides$^{\dagger}$}  \\
      & TCGA-STAD            & 2 & \multicolumn{2}{c}{303 slides$^{\dagger}$}  \\
    \midrule
    \multirow{5}{*}{LRA}
      & ListOps    & 10 & 96{,}000  & 2{,}000  \\
      & Text       & 2  & 25{,}000  & 25{,}000 \\
      & Retrieval  & 2  & 147{,}086 & 17{,}437 \\
      & Image      & 10 & 45{,}000  & 10{,}000  \\
      & Pathfinder & 2  & 160{,}000 & 20{,}000 \\
    \midrule
    Long documents
      & arXiv topic classification & 11 & 28{,}388 & 2{,}500 \\
    \midrule
    VLM hallucination
      & POPE & 2  & -- & 9{,}000 Q / 500 img   \\
    \midrule
    Fine-grained
      & Stanford Cars   & 196 & 8{,}144 & 8{,}041 \\
    \bottomrule
  \end{tabular}\\[2pt]
  {\footnotesize $^{\dagger}$Patient-level splits; see
   Appendix~\ref{app:data-histo}.}
\end{table}

\subsection{Histopathology}
\label{app:data-histo}
 
\paragraph{Task.}
We consider slide-level prediction of microsatellite instability
(MSI-high, MSI-H, vs.\ microsatellite stable, MSS) from routine
haematoxylin-and-eosin (H\&E) stained tissue. MSI-H is caused by a
deficient DNA mismatch-repair system. It is a clinically actionable
biomarker, because MSI-H tumours respond markedly better to immune
checkpoint inhibitors. It is currently determined with dedicated
immunohistochemistry or PCR-based genetic assays. Predicting it directly
from H\&E morphology has become a standard benchmark for computational
pathology \citep{kather2019msi, wagner2023attentionmil}. A single binary label is available per slide (inherited from the patient), while the morphological evidence, e.g. tumour-infiltrating lymphocytes, mucinous or medullary growth
patterns, poor differentiation occupies only a fraction of the
tissue.

\paragraph{Cohorts.}
We use two cohorts from The Cancer Genome Atlas (TCGA):
\begin{itemize}
  \item \textbf{Colorectal cancer (TCGA-CRC):} colon adenocarcinoma
        (TCGA-COAD) and rectum adenocarcinoma (TCGA-READ), pooled into a
        single cohort of 365 diagnostic slides.
  \item \textbf{Stomach adenocarcinoma (TCGA-STAD):} 303 diagnostic
        slides.
\end{itemize}
All slides are formalin-fixed, paraffin-embedded (FFPE) diagnostic
sections, as opposed to the frozen sections that TCGA also releases. FFPE
slides preserve morphology substantially better and are the ones used in
clinical practice. The slides were collected at many different
contributing institutions and therefore exhibit strong variation in
staining, section thickness and scanner, which makes the benchmark a
realistic test of robustness.

\paragraph{Labels and splits.}
MSI status labels are taken from the TCGA pan-gastrointestinal molecular
analysis of \citet{liu2018gi}, which assigns MSI status from
sequencing-based assays. Tumours classified as MSI-low are grouped with
MSS, following common practice. We follow the published patient-level
splits of \citet{kather2019msi} and \citet{wagner2023attentionmil}, so
that all slides of one patient fall into the same partition and no
patient information leaks between training and evaluation.

\paragraph{Features and protocol.}
Patch features are the released, gated
\texttt{MahmoodLab/UNI2-h-features} (HuggingFace): 1536-dimensional
UNI2-h embeddings in CLAM-style HDF5 files that also carry the patch-center
coordinates; we download these precomputed features and do not re-extract
them, so patching geometry and background filtering are inherited from that
release. The patient-level splits contain $259/52/54$ slides for CRC
(MSI-H fraction $17.4/15.4/16.7\%$) and $205/47/51$ for STAD
($18.5/23.4/15.7\%$); no patient spans two partitions (verified on both
manifests; a few patients contribute more than one slide: $365$ slides from
$359$ CRC patients, $303$ from $285$ STAD patients).

Slide lengths vary
widely: CRC has median $M{=}5{,}922$ patches (range $328$--$31{,}521$),
STAD median $M{=}7{,}329$ (range $180$--$31{,}458$); only $3.6\%$ of CRC
and $8.3\%$ of STAD slides have $M\le1024$, for which the fixed budget
$N{=}1024$ is not a compression; such slides go through the same code
path, with the operator emitting $N$ convex combinations of the $M\le N$
inputs (the $N\ge M$ case is exercised in
Table~\ref{tab:backend-agreement}). Model selection takes the checkpoint with
the best validation AUROC; the test AUROC is reported from that checkpoint.
The random seed controls only the initialization and batch order -- the
splits are fixed files. The MIL baselines process the full bag of $M$
patches, whereas \our{} first reduces it to $N{=}1024$ tokens: the
comparison matches inputs and splits, not downstream sequence length.

\paragraph{Perturbation test of the routing-mass map.}
To check that the learned mass is an attribution and not only a budget
map, we delete $q\in\{5,10,20,50\}\%$ of each test slide's patches --
chosen by highest mass, lowest mass, or uniformly at random (masses always
computed on the unperturbed slide) -- and re-run the full pipeline on the
survivors, for five of the six seed checkpoints per cohort (those with
saved weights). Deleting the highest-mass patches degrades AUROC with
growing $q$ (CRC: $0.912\to0.894\to0.895\to0.868\to0.800$ at
$q{=}0/5/10/20/50\%$; STAD: $0.959\to0.956\to0.941\to0.920\to0.849$),
while deleting the lowest-mass or random patches
leaves it essentially unchanged even at $q{=}50\%$ (CRC $0.902$/$0.895$,
STAD $0.962$/$0.962$). The mass map therefore passes a deletion-style
attribution test: the prediction rests on the patches the map points to.

\paragraph{What the scorer attends to.}
Figure~\ref{fig:scorer-heatmap} complements the quantitative test above
with a qualitative view: for six CRC test slides (three MSI-H, three MSS,
all correctly and confidently classified) it pairs the true-color
thumbnail with the same thumbnail desaturated under the scorer-weight
overlay, where opacity is tied to the weight so that low-weight patches
fade into the background. The weight is visibly non-uniform and
slide-specific, concentrated on compact sub-regions that differ from
slide to slide -- consistent with the scorer doing real,
content-dependent work rather than acting as a near-uniform
pass-through.

\begin{figure}[htbp]
  \centering
  \includegraphics[width=\linewidth]{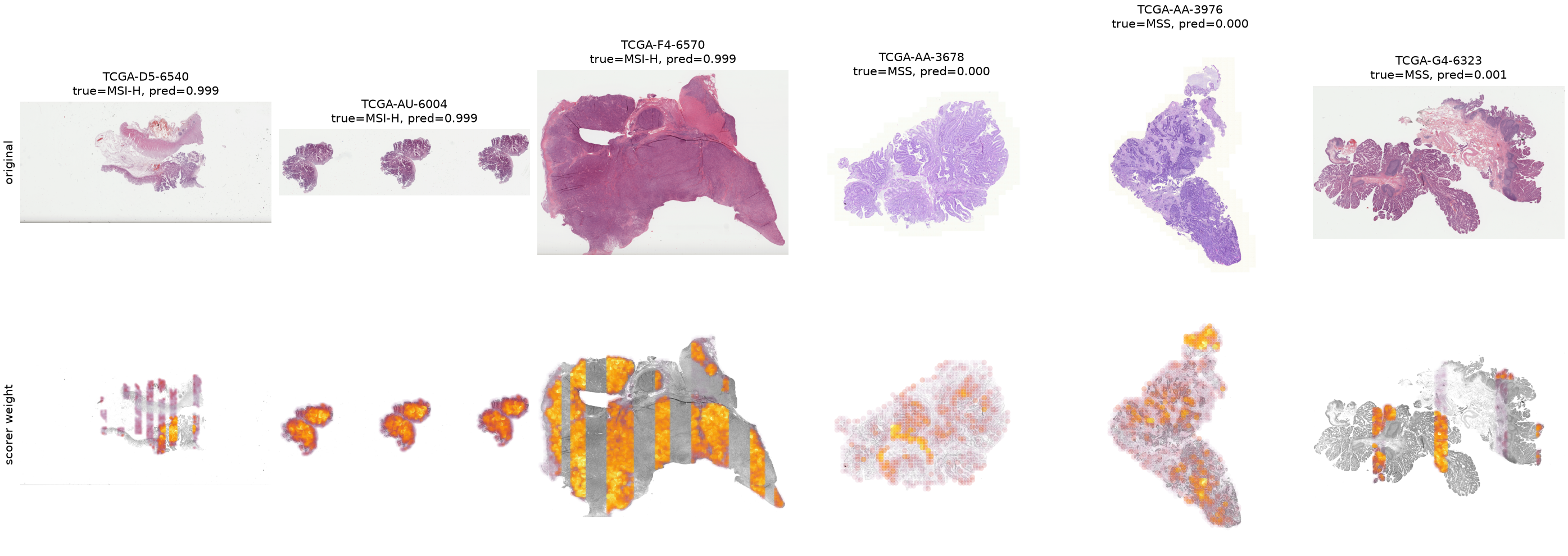}
  \caption{Per-patch scorer weight of \our{} for six TCGA-CRC test slides
  (three MSI-H, three MSS), one column per slide: true-color thumbnail
  with true label and predicted $p(\text{MSI-H})$ (top) vs.\ the
  desaturated scorer-weight overlay (bottom; opacity tied to weight, so
  only the regions the scorer relies on stand out). \emph{The banded
  columns (TCGA-D5-6540, TCGA-F4-6570) are gaps in patch coverage between
  several narrow tissue fragments on those slides -- confirmed from the
  raw patch coordinates -- not a scorer artefact.}}
  \label{fig:scorer-heatmap}
\end{figure}

\subsection{LRA}
\label{app:data-lra}
 
The Long Range Arena (LRA) \citep{tay2021lra} is a suite of
sequence-classification tasks designed to evaluate efficient
Transformers under long contexts (1K--16K tokens). It covers
hierarchical/structured reasoning, byte-level text, document matching,
and image-as-sequence inputs. In all tasks the model sees a flat
sequence of discrete tokens with no built-in inductive bias for the
underlying structure. We use the standard data, pre-processing and splits
released with the benchmark. In our experiments the reducer and a
four-layer Transformer backbone are trained jointly from scratch on every
task with Adam (initial learning rate $10^{-3}$, cosine decay,
$80{,}000$ steps, batch size $32$); all budget-matched baselines share
the backbone, optimizer, and training budget. The shared $10\%$ budget
sets $N$ per task from its input length $M$: ListOps
$2{,}048\!\to\!204$, Text $4{,}096\!\to\!410$, and Image/Pathfinder
$1{,}024\!\to\!102$ tokens, i.e.\ roughly $100\times$ fewer attention
pairs in the reduced backbone.

\paragraph{ListOps.}
Each input is a nested prefix-notation expression \citep{nangia2018listops} over single digits (0--9) with the operators \texttt{MAX}, \texttt{MIN}, \texttt{MED} (median) and \texttt{SM} (sum modulo 10), e.g.\ \texttt{[MAX 4 3 [MIN 2 3 ] 1 0 [MED 1 5 8 9 2 ]]}. The target is the value of the expression, i.e.\ a 10-way classification. Sequences are up to 2{,}048 tokens long and the vocabulary consists of the digits, operators and brackets. Solving the task requires modelling the hierarchical structure over the whole sequence. A single token at the very beginning can change the result. Split: 96{,}000 / 2{,}000 /
2{,}000 (train / val / test).

\paragraph{Text.}
Binary sentiment classification on the IMDb movie-review dataset
\citep{maas2011learning}, operating at the \emph{byte} (character) level
rather than on word tokens, which turns a moderate-length document task
into a long-sequence one and forces the model to compose characters into
words and phrases. Reviews are truncated or padded to 4{,}096 bytes.
Split: 25{,}000 training and 25{,}000 test reviews, balanced between
positive and negative.

\paragraph{Retrieval.}
Byte-level document matching on the ACL Anthology Network (AAN) corpus
\citep{radev2013acl}. Given two scientific papers, the model predicts
whether one cites the other (binary). Each document is encoded
separately at byte level with up to 4{,}000 tokens (8{,}000 tokens per
pair); following the LRA protocol the two document representations are
combined only at the classifier (a two-tower setting), so the task tests
how well a fixed-size representation compresses a long document.
Split: 147{,}086 / 18{,}090 / 17{,}437.
 
\paragraph{Image.}
CIFAR-10 \citep{krizhevsky2009learning} images converted to 8-bit
grayscale and flattened in raster-scan order into sequences of
$32\times32=1{,}024$ pixel intensities (vocabulary of 256), with 10
object classes. Because no 2D structure is given explicitly, the model
must recover spatial relations between pixels that are 32 positions
apart in the sequence. Split: 45{,}000 / 5{,}000 / 10{,}000.

\paragraph{Pathfinder.}
A synthetic visual reasoning task \citep{linsley2018learning}. Each
$32\times32$ grayscale image contains two highlighted endpoints and a
number of dashed curves, and the model must decide whether the two
endpoints are connected by a single path (binary). Distractor paths make
local cues uninformative, so the task requires integrating evidence
along long, curved paths. Images are flattened to 1{,}024 tokens.
Split: 160{,}000 / 20{,}000 / 20{,}000. Path-X, the $128\times128$ (16{,}384-token) version, is not included in our experiments.

\subsection{arXiv Topic Classification}
\label{app:data-arxiv}

The \texttt{ccdv/arxiv-classification} dataset assigns each of $33{,}388$
arXiv papers one of $11$ subject classes. We use the released public split
($28{,}388$ / $2{,}500$ / $2{,}500$ train/validation/test) and truncate
each document to $M{=}1024$ tokens of the frozen
\texttt{allenai/longformer-base-4096} tokenizer. Relevant evidence is
spread through the whole text, which makes the task a deliberate stress
test for order-based merging; the full training protocol and results are
in Appendix~\ref{sub:selection-arxiv}.

\subsection{POPE}
\label{app:data-pope}
 
\paragraph{Task.}
The Polling-based Object Probing Evaluation (POPE)
\citep{li2023pope} measures object hallucination in large
vision-language models (LVLMs), i.e. the tendency to claim that objects
are present in an image when they are not. Instead of parsing free-form
captions, POPE casts evaluation as a set of binary questions of the form
\emph{``Is there a \textless object\textgreater{} in the image?''}, which
the model must answer with ``yes'' or ``no''. This makes it a direct
probe of whether fine-grained visual evidence survives token
compression. An aggressive merging scheme that discards small or
peripheral objects will increase false negatives, whereas one that
blends object and background tokens can increase false positives.

\paragraph{Data.}
We use the standard MSCOCO version of POPE, built on 500 images, each containing at least three annotated object categories. For every image, three positive questions are formed from ground-truth objects and three negative questions from objects absent from the image, giving a balanced set of 3{,}000 questions (50\% ``yes'') per setting. Three settings differ only in how the negative objects are sampled:
\begin{itemize}
  \item \textbf{Random:} negatives are drawn uniformly from the COCO
        object vocabulary among categories not present in the image.
  \item \textbf{Popular:} negatives are the most frequent COCO
        categories not present in the image, testing whether the model
        is biased towards generally common objects.
  \item \textbf{Adversarial:} negatives are the categories that most
        often \emph{co-occur} with the image's ground-truth objects
        (e.g.\ ``fork'' for an image with a ``dining table''), the
        hardest setting, since language priors strongly suggest their
        presence.
\end{itemize}
In total, the benchmark comprises 9{,}000 image--question pairs.

\subsection{Fine-grained Vision: Stanford Cars}
\label{app:data-fgvc}
 
\paragraph{Stanford Cars.}
Stanford Cars \citep{krause2013stanfordcars} contains 16{,}185 images of 196
car classes, where a class is defined at the level of
\emph{make, model and year} (e.g. \emph{2012 Tesla Model S} vs.
\emph{2012 BMW M3 coupe}). The official split has 8{,}144 training and
8{,}041 test images, i.e. roughly a 50--50 split with about 41 training
images per class. Images come from the web and vary strongly in
resolution, viewpoint (front, rear, side, three-quarter views) and
background. Several classes differ only by model year
and are distinguishable only through subtle details, making the dataset
particularly demanding for methods that discard high-frequency
information.
\end{document}